%% file: prep.tex
\documentclass{MITcsail}
\usepackage[T1]{fontenc}
\usepackage{algpseudocode} % defines \begin{algorithmic} ... \State ...

\usepackage{array} % Enhances table and array formatting
\usepackage{supertabular} % Supports tables spanning multiple pages

\usepackage{enumitem} % Add different enumerate bullets
\usepackage{hyperref} % Adds hyperlinks
\usepackage{url} % Simplifies typesetting and breaking of URLs
\usepackage{graphicx} % Required for inserting images
\usepackage{float} % To use the H option

\usepackage{amsmath} % Math typesetting
\usepackage{amssymb} % Additional math symbols
\usepackage{amsthm} % Theorem-like environments
\usepackage{algorithm}

\usepackage{url}

\newtheorem{requirement}{Requirement}[section]
\usepackage{xcolor}
\hypersetup{
    colorlinks=true,
    linkcolor=purple,
    citecolor=purple,
    filecolor=darkred,
    urlcolor=purple,
}
\usepackage[numbers, sort&compress]{natbib}

\title{OML:  A Primitive for Reconciling Open Access with Owner Control in AI Model Distribution
}

\author{%
  Zerui Cheng$^{1}$,
  Edoardo Contente$^{2}$,
  Ben Finch$^{2}$,
  Oleg Golev$^{2}$,
  Jonathan Hayase$^3$,\\
  Andrew Miller$^4$,
  Niusha Moshrefi$^1$,
  Anshul Nasery$^3$,
  Sandeep Nailwal$^2$,\\
  Sewoong Oh$^3$,
  Himanshu Tyagi$^2$,
  Pramod Viswanath$^{1,2}$\\
\vspace{1em}
\normalfont{\small $^{1}$ Princeton University}\\
\normalfont{\small $^{2}$ Sentient Foundation}\\
\normalfont{\small $^{3}$ University of Washington }\\
\normalfont{\small $^{4}$ University of Illinois Urbana-Champaign}\\
  \vspace{1em}
  \texttt{\{zerui.cheng, pramodv\}@princeton.edu}\\
}

\begin{document}
\maketitle

\begin{abstract}
The current paradigm of AI model distribution presents a fundamental dichotomy: models are either closed and API-gated, sacrificing transparency and local execution, or openly distributed, sacrificing monetization and control. We introduce \textbf{OML} (Open-access, Monetizable, and Loyal AI Model Serving), a primitive that enables a new distribution paradigm where models can be freely distributed for local execution while maintaining cryptographically enforced usage authorization. We are the first to introduce and formalize this problem, introducing rigorous security definitions tailored to the unique challenge of white-box model protection: \emph{model extraction resistance} and \emph{permission forgery resistance}. We prove fundamental bounds on the achievability of OML properties and characterize the complete design space of potential constructions, from obfuscation-based approaches to cryptographic solutions. To demonstrate practical feasibility, we present OML 1.0, a novel OML construction leveraging AI-native model fingerprinting coupled with crypto-economic enforcement mechanisms. Through extensive theoretical analysis and empirical evaluation, we establish OML as a foundational primitive necessary for sustainable AI ecosystems. This work opens a new research direction at the intersection of cryptography, machine learning, and mechanism design, with critical implications for the future of AI distribution and governance.

\end{abstract}

\section{Introduction}
\label{intro}

Artificial Intelligence (AI) is advancing at an incredible pace, reshaping diverse fields from household robotics \cite{roomba, atlas} and superhuman game-playing \cite{silver2017mastering, silver2017masteringgo, silver2018general} to intricate formal mathematical reasoning \cite{GoogleDeepMindAlphaproof}, protein structure elucidation \cite{jumper2021highly, evans2021protein}, accelerated drug discovery \cite{bostrom2018expanding, strokach2020fast, schneider2020rethinking}, and novel mathematical exploration \cite{romera2024mathematical}. The emergence of powerful generative models like GPT series \cite{openai2023gpt4, bubeck2023sparks}, OpenAI \texttt{o3} \cite{openaio1, jaech2024openai}, and DeepSeek \texttt{R1} \cite{guo2025deepseek} marks a watershed moment, heralding their potential to fundamentally transform human endeavors.

Yet, the rapid advancement of artificial intelligence has created an unprecedented challenge in model distribution. Current approaches force an unnecessary tradeoff between accessibility and control, limiting both innovation and sustainable development of AI systems. This paper introduces and formalizes OML, a primitive that reconciles these conflicting requirements.

\subsection{The Fundamental Distribution Problem}

Modern AI development, responsible for the ubiquitous systems we see today, has been significantly shaped by open-source contributions. Until recently, core libraries and powerful models such as BERT~\cite{devlin2018bert} and early iterations of GPT~\cite{openai2023gpt4, bubeck2023sparks} were openly available. However, as AI matured and its profound economic potential became evident, many large companies that initially embraced open development have transitioned to closed strategies. These entities have often geared their efforts towards establishing significant positions in the AI economy, with a model where others act as high-level users of the AI they build~\cite{openai, forefront_ai, ai21}. Consequently, AI is currently delivered to users predominantly via the following two dominant paradigms, each with critical limitations.

\textbf{Closed API Services}: In this paradigm, AI models are primarily accessed through public APIs~\cite{openai, forefront_ai, ai21}. Platforms like OpenAI's GPT and Anthropic's Claude maintain complete control over model execution, enabling monetization and usage governance. Such centralized, closed services offer benefits like scalability and the implementation of certain safety measures, including content moderation and misuse prevention. Conversely, this approach can lead to monopolization, rent-seeking behaviors, and significant privacy concerns. Users also lack ultimate control over the paid service, as model owners can arbitrarily filter inputs, modify outputs, or change the underlying model without direct user consent. While options for fine-tuning closed models may exist, such customization is typically constrained by the limitations of the provided API.  Users cannot verify model behavior, ensure data privacy, or maintain operational independence.

\textbf{Open-weight distribution}: In this paradigm, creators release model weights, often with their architectures, allowing users to download and run inference locally. Platforms like Hugging Face enable unrestricted model distribution, providing transparency, local execution, and modification capabilities. This grants users full control over model selection, inference processes, and the ability to build upon these models (e.g., through fine-tuning) and compose them with other AI systems. Meta's Llama model series~\cite{touvron2023llama}, DeepSeek~\cite{guo2025deepseek}, and the wide array of models available on platforms like Hugging Face exemplify this approach. However, once models are released in this manner, creators lose direct control over their subsequent use, cannot enforce safety constraints and prevent unethical applications, and lack mechanisms for sustainable monetization. This creates a tragedy of the commons where innovation is undersupported.

This dichotomy is not merely a business model choice but represents a \textbf{fundamental technical limitation} in our current infrastructure. We lack the primitives necessary to enable models that are simultaneously open-access for local execution and controlled for authorized and ethical usage.

\subsection{OML: Technical Requirements for Ideal Resolution}
\label{sec:intro:4OML}

While both closed and open models offer distinct advantages, we pursue maximal openness comparable to current open-weight distributions, augmented with mechanisms for owner control. An ideal solution must therefore reconcile three seemingly contradictory properties, which we term ``OML":
\begin{itemize}
\item \textbf{[O] Open-access}. Models must be freely distributable for local execution, analogous to compiled binaries where functionality is accessible while implementation details remain protected. Once distributed, models become immutable artifacts independent of their creators, enabling user data privacy, consistent service quality, and on-premise deployment.

\item \textbf{[M] Monetizable}. The framework enables model owners to capture economic value through granular, per-inference authorization mechanisms. Each model invocation requires individual permission from model owners, and an economic Nash Equilibrium, where rational users obtain proper authorization (such as purchasing access tokens) rather than attempting circumvention, guarantees that the all agreements and licenses are enforced.

\item \textbf{[L] Loyal}. Models must technically enforce owner-defined policies through pre-hoc authorization verification. The model produces high-utility outputs exclusively when presented with valid, cryptographically-bound permissions, ensuring compliance with safety and ethical constraints before computation rather than through legal remedies after violation.
\end{itemize}

Note that, while monetizability and loyalty both enable governance, they address fundamentally different requirements. Monetizability concerns economic sustainability and can leverage post-hoc mechanisms such as collateral-backed compliance. Loyalty demands pre-hoc technical enforcement to prevent generation of harmful or policy-violating outputs before they occur.

This challenge becomes particularly acute under white-box access conditions where users possess complete visibility into model weights, architecture, and computation flow. Unlike traditional software where obfuscation can leverage discrete execution paths and control flow complexity, neural networks present continuous, differentiable computations that resist conventional protection mechanisms. Any solution must therefore leverage properties unique to machine learning systems rather than adapting existing software protection paradigms.

\subsection{Our Contributions}

\noindent\textbf{Core contribution.} We are the first to identify and formalize the OML challenge, a primitive that enables AI models to be distributed openly while maintaining cryptographically enforced usage control under white-box access.
This paper establishes the foundations for OML through:

\begin{enumerate}
\item \textbf{Problem Formalization}: We are the first to provide rigorous definition of requirements for reconciling open access with owner control, with three properties (Open-access, Monetizable, Loyal) and quantifiable metrics$(\epsilon_{utility}, \epsilon_{robust}, \epsilon_{overhead})$ for evaluating solutions.

\item \textbf{Security Framework}: We establish novel security games for white-box model protection (model extraction resistance, permission forgery resistance) where adversaries have complete visibility into weights and computation, establishing security standards for OML realization.

\item \textbf{Design Space Characterization}: We analyze a wide array of potential approaches (obfuscation, TEEs, cryptography) with their fundamental tradeoffs and theoretical bounds.

\item \textbf{Feasibility Demonstration}: We introduce and instantiate OML 1.0, a practical construction using AI-native fingerprinting that achieves ``next-day security" with negligible $\epsilon_{overhead}$.
Experiments and empirical validation further demonstrates feasibility of our approach.
\end{enumerate}

\textbf{Our primary contribution is establishing OML as a well-defined primitive and demonstrating its feasibility, thereby opening a new research direction at the intersection of cryptography, machine learning, and mechanism design.} The challenge of protecting neural networks under white-box access while preserving utility represents one of the most difficult problems in cryptography and AI, requiring fundamentally new techniques beyond classical approaches. We identify critical open problems spanning theoretical questions (tight complexity bounds, connections to program obfuscation, optimal constructions) and practical challenges (robustness under fine-tuning, compositional security, efficient real-time enforcement).

The full realization of OML will likely require years of sustained research effort from the community, comparable to the decades-long development of practical homomorphic encryption or secure multiparty computation. By providing the first formal framework and proving initial feasibility through OML 1.0, we aim to catalyze the long-term research program necessary to advance OML from primitive concept to robust, deployable infrastructure for AI distribution.

\subsection{Significance and Impact}
Through the OML primitive, we can address critical safety and privacy challenges faced in AI development today, while enabling a more collaborative ecosystem for AI development tomorrow:

\textbf{Privacy-Preserving Local Execution}: OML enables organizations handling sensitive data, e.g. healthcare providers, financial institutions, government agencies, to deploy state-of-the-art models without exposing confidential information to external APIs. Medical institutions can run diagnostic models on patient data locally while still compensating model creators. Financial firms can apply fraud detection models without sharing transaction patterns. This solves the current impossibility where organizations must choose between using inferior models or violating privacy requirements.

\textbf{Technical Enforcement of Ethical and Legal Usage}: By requiring pre-hoc cryptographic authorization for each inference, OML provides robust mechanisms for enforcing safety policies and any signed user agreements on open-weight models that currently rely solely on terms of service today.

\textbf{Enabling a More Colloborative Ecosystem for AI}: Most fundamentally, OML creates the technical infrastructure for distributed AI development where every contribution, from initial training to incremental improvements, can be tracked and rewarded. Figure \ref{fig:problem} illustrates how OML transforms the current one-way distribution into a sustainable ecosystem where usage generates returns for all contributors, creating economic incentives for continuous improvement rather than one-time model releases. Collaboration with positive feedback loop and proper incentives, rather than monopolization, guarantees that AI development is on the right track that benefits the humanity.

\begin{figure}[htbp]
\centering
\includegraphics[width=.9\linewidth]{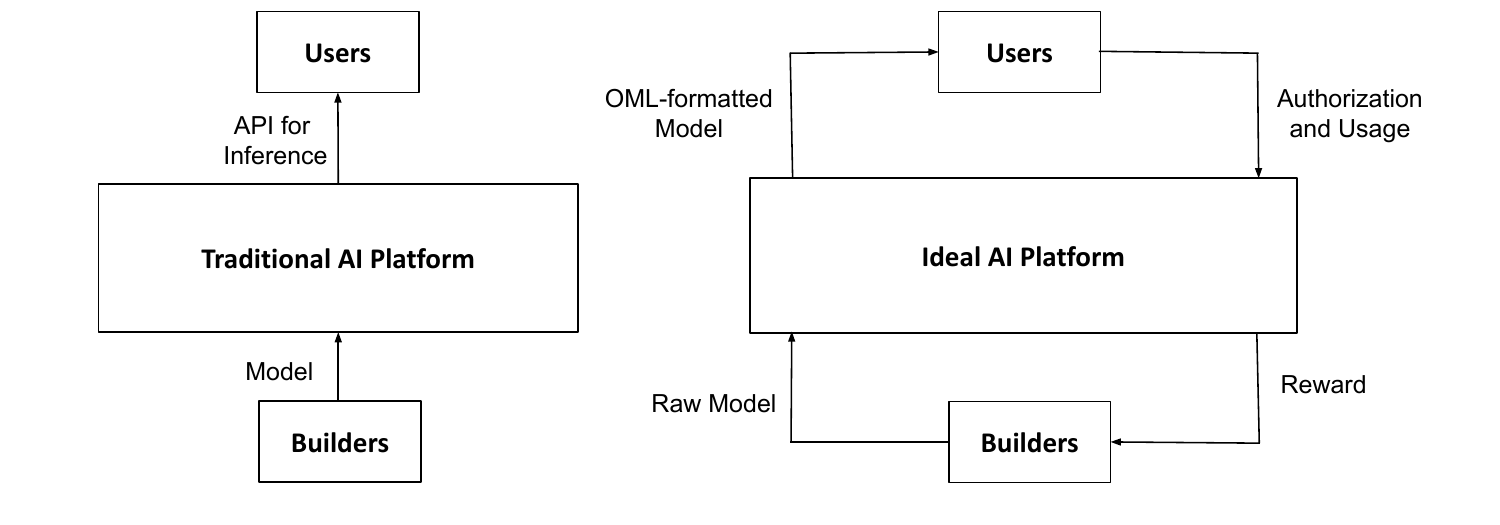}
\caption{OML enables transition from one-way model distribution to bidirectional value flow. \\Left: Current paradigm where models are distributed without feedback or compensation mechanisms. \\Right: OML-enabled ecosystem where usage generates returns for all contributors, incentivizing continuous collaborative improvement.}
\label{fig:problem}
\end{figure}
These three impacts are mutually reinforcing: privacy-preserving local execution expands the market for AI models; license enforcement makes broader deployment responsible; and collaborative development efforts ensure sustainable innovation. Together, they address the fundamental limitations preventing AI from achieving its potential as a broadly beneficial technology.

\section{The OML Primitive: Formal Definition}\label{oml}

The transformation of AI models from static artifacts to dynamic, controllable assets requires a fundamental reconceptualization of how we distribute and govern computational intelligence. In this section, we formalize the \emph{Open-access, Monetizable, and Loyal (OML)} primitive, a framework that enables white-box model distribution while preserving ownership rights and enforcing usage policies.

\subsection{Properties and Design Space}\label{sec:protocol_design_space}

Consider an AI model $M$ as an economic asset: valuable, replicable, and vulnerable to unauthorized extraction once distributed. The OML primitive transforms this vulnerable asset into a controlled artifact that retains its utility while coupling high-quality outputs to owner authorization. 

We begin by establishing our notation framework, which we will use throughout this section.

\begin{table}[H]\centering
\caption{Notation and Core Components of the OML Framework}\label{tab:notation}
\small
\begin{tabular}{@{}ll@{}}\toprule
\textbf{Symbol} & \textbf{Description} \\
\midrule
$M:\mathcal{X}\!\to\!\mathcal{Y}$ & Original model mapping inputs to outputs \\
$M_{\text{oml}}$ & OML-formatted model with embedded authorization \\
$h:\mathcal{X}\!\to\!\mathcal{H}$ & Input-binding transform (e.g., cryptographic commitment) \\
$\sigma:\mathcal{H}\!\times\!\mathcal{K}_{\text{own}}\!\to\!\mathcal{P}$ & Permission token generator \\
$k_{\text{own}}$ & Owner's secret key; $vk_{\text{own}}$ denotes optional public verifier \\
$p_x=\sigma(h(x),k_{\text{own}})$ & Permission token cryptographically bound to input $x$ \\
$d(\cdot,\cdot)$ & Task-appropriate distance or divergence metric \\
$\epsilon_{\text{utility}}$ & Maximum fidelity loss on authorized queries \\
$\epsilon_{\text{robust}}$ & Minimum degradation on unauthorized queries \\
$\epsilon_{\text{overhead}}$ & Relative computational overhead bound \\
\bottomrule
\end{tabular}
\end{table}

With this notation established, we can now formally define the OML transformation, a process that embeds authorization logic so deeply within the model's computational graph that removing it becomes computationally equivalent to retraining the model from scratch.

\begin{definition}[\textbf{OMLized Model}] \em
Given an original model $M:\mathcal{X}\!\to\!\mathcal{Y}$, an OMLization process
\[
\mathrm{OMLize}(M;h,\sigma,\mathrm{params}) \ \longrightarrow \ M_{oml},
\]
produces a locally executable artifact that operates on input-token pairs $(x,p)$.
For each input $x\in\mathcal{X}$, authorization requires a valid token $p_x=\sigma(h(x),k_{own})$ computed with owner's secret key $k_{own}\in\mathcal{K}_{own}$. Informally, $M_{oml}$ behaves as $M$ on authorized inputs and degrades otherwise.

\end{definition}

This definition captures the essence of controlled distribution: the model remains functionally accessible but computationally gated. To understand how this works in practice, we present the idealized OML workflow, which demonstrates how authorization, utility preservation, and security enforcement interact.

\paragraph{Promises of an Ideal OML.}
Let $d:\mathcal{Y}\!\times\!\mathcal{Y}\!\to\!\mathbb{R}_{\ge0}$ denote a distance metric and $T(F,z)$ the computational cost of evaluating function $F$ at input $z$. A correct OMLization satisfies:
\begin{enumerate}[leftmargin=*,itemsep=2pt]
\item \textbf{Authorization:} Users submit $h(x)$ to owner $\Pi_{\mathcal O}$; if approved, they receive $p_x$ and query $(x,p_x)$.
\item \textbf{Fidelity:} $d(M_{oml}(x,p_x),M(x))\le \epsilon_{\text{utility}}$, ensuring preservance of the model's core capabilities.
\item \textbf{Protection:} For invalid $p$, $d(M_{oml}(x,p),M(x))>\epsilon_{\text{robust}}$ with $\epsilon_{\text{robust}}>\epsilon_{\text{utility}}$.
\item \textbf{Overhead:} $T(M_{oml},(x,p_x))\le(1+\epsilon_{\text{overhead}})\,T(M,x)$, preserving practical deployability.
\end{enumerate}

These four promises collectively define what we term the \emph{quality profile} $(\epsilon_{\text{utility}},\epsilon_{\text{robust}},\epsilon_{\text{overhead}})$, a quantitative characterization of an OMLization's effectiveness. And the instantiation requires careful design of three interconnected components that form the technical foundation of any OML construction listed below. We also depict a high-level OMLization process in Algorithm \ref{alg:omlize}.
\begin{itemize}
\item an ownership key $k_{\text{own}}$ (either cryptographic or AI-native); 

\item a binding/permission token issuance mechanism $(h,\sigma)$, optionally exposing $vk_{\text{own}}$;

\item \emph{verifier entanglement}, which couples authorization to critical computations so that the high-utility pathway is reachable only under valid authorization.
\end{itemize}

\begin{algorithm}[H]
\caption{\textsc{OMLize}: Transforming Models into Controlled Artifacts}\label{alg:omlize}
\begin{algorithmic}[1]
\State \textbf{Input:} Original model $M$, binding function $h$, token scheme $\sigma$, public parameters
\State \textbf{Output:} Controlled artifact $M_{\text{oml}}$
\State \textbf{Step 1:} Embed verifier $\alpha:\mathcal{X}\!\times\!\mathcal{P}\!\to\!\{0,1\}$ that validates tokens against input commitments
\State \textbf{Step 2:} Entangle $\alpha$ within $M$'s critical paths to construct $F$ such that:
\Statex \quad\quad (i) Valid authorization: $\alpha(x,p_x)=1 \Rightarrow F(x,p_x)\approx M(x)$
\Statex \quad\quad (ii) Invalid tokens: $\alpha(x,p)\neq 1 \Rightarrow F(x,p)$ yields degraded/noisy output
\State \textbf{Step 3:} Optionally expose $vk_{\text{own}}$ for public verification capability
\State \Return $M_{\text{oml}}(x,p)=F(x,p)$
\end{algorithmic}
\end{algorithm}
\subsection{Security Guarantees and Adversarial Model}\label{sec:security_adversarial}

The security of OML must be analyzed under the assumption of white-box access where adversaries can inspect, modify, and experiment with $M_{oml}$ arbitrarily. This threat model reflects the reality where OMLized models may be controlled by potentially adversarial users.

\textbf{Adversary Model.}
We model adversaries as probabilistic polynomial-time (PPT) algorithms $\mathcal{A}$ with 
\begin{itemize}[leftmargin=*]
    \item Complete white-box access to $M_{oml}$, including all parameters and computation graphs
    \item Oracle access to an authorization service $\Pi_{\mathcal O}$ for up to $N$ queries
    \item The resulting knowledge base $\mathcal{D}_{known}=\{(x_i,p_{x_i},y_i)\}_{i=1}^{N}$ where $y_i=M_{oml}(x_i,p_{x_i})$
\end{itemize}

\textbf{Security Goal.} Against such adversaries, two fundamental hardness properties should hold:

\begin{requirement}[\textbf{Model Extraction Resistance}]\label{req:extraction_resistance}\em

In experiment $\mathrm{Expt}^{\mathrm{ME}}_{\mathcal A}$: 

(1) $\mathcal A$ receives $M_{oml}$ and oracle access to $\mathcal{P}_{\mathcal O}$ for $N$ queries; (2) $\mathcal A$ outputs a stand-alone model $M'$; 

(3) a fresh $x^*\sim\mathcal D_{\mathcal X}$ is drawn with $x^*\notin\{x_i\}$;
(4) $\mathcal A$ \emph{wins} if $d(M'(x^*),M(x^*))\le \epsilon_{utility}$. 

The scheme is $(t,N,\epsilon_{ME})$-extraction-resistant if every PPT $\mathcal A$ running in time $t$ wins with probability at most $\epsilon_{ME}(t,N)$. Informally, any adversary cannot replicate a functionally equivalent model that bypasses authorization within reasonable cost. 
\end{requirement}

\begin{requirement}[\textbf{Permission Forgery Resistance}]\label{req:forgery_resistance}\em

In experiment $\mathrm{Expt}^{\mathrm{PF}}_{\mathcal A}$: 

(1) $\mathcal A$ receives $M_{oml}$ and oracle access to $\mathcal{P}_{\mathcal O}$ for $N$ queries; 

(2) a fresh $x^*\sim\mathcal D_{\mathcal X}$ is revealed with $x^*\notin\{x_i\}$; 

(3) $\mathcal A$ outputs $p^*$; 
(4) $\mathcal A$ \emph{wins} if $d(M_{oml}(x^*,p^*),M(x^*))\le \epsilon_{utility}$. 

The scheme is $(t,N,\epsilon_{PF})$-forgery-resistant if every PPT $\mathcal A$ running in time $t$ wins with probability at most $\epsilon_{PF}(t,N)$. Informally, adversaries cannot generate valid tokens for unauthorized inputs.
\end{requirement}

These requirements must withstand a diverse threat landscape:

\begin{enumerate}[leftmargin=*,itemsep=3pt]
    \item \textbf{Surgical Extraction:} Adversaries employ network surgery techniques~\cite{raiman2019neural,NIPS2016_2823f479} to identify and excise authorization logic while preserving model functionality.
    
    \item \textbf{Runtime Manipulation:} Fault injection or state tampering forces the internal verifier to accept invalid tokens, bypassing authorization checks without modifying the model itself.
    
    \item \textbf{Cryptanalytic Forgery:} Adversaries attempt to forge valid tokens through: (a) exploiting implementation vulnerabilities in $\sigma$, (b) recovering the secret key $k_{own}$ from side channels, or (c) training surrogate functions $\hat\sigma:x\mapsto p_x$ using $\mathcal{D}_{known}$.
\end{enumerate}

\paragraph{The Failure of Naive Approaches.}
To illustrate why sophisticated entanglement is necessary, consider a naive wrapper design with a cryptographic digital signature scheme:
\[
M_{oml}(x,p)\ :=\ \begin{cases} M(x) & \text{if } \mathrm{Verify}_{vk_{own}}(h(x),p) = \text{true} \\ \perp & \text{otherwise} \end{cases}
\]

With white-box access, an attacker can trivially locate the conditional branch, remove the verification check, and recover the original model $M$. This vulnerability motivates our requirement for deep computational entanglement, i.e. the verifier must be so thoroughly integrated that removing it is tantamount to destroying the model's learned representations.

\subsection{Theoretical Foundations and Limits}\label{sec:theoretical}

In this subsection, we state three results that define the feasible region for OML: an upper bound that prevents over-claiming, a sufficient condition that anchors OML in standard hardness, and an operational constraint that links security to issuance policy. Proofs are deferred to App.~\ref{app:proofs}.

First, if an adversary controls the artifact and can issue unbounded authorized queries, information alone suffices to reconstruct the task mapping, and perfect protection is therefore unattainable.

\begin{theorem}[Information-theoretic impossibility]\label{thm:impossibility} \em
No OML scheme achieves perfect security against unbounded adversaries with unlimited oracle access.
\end{theorem}

Second, under strong program hiding, authorization can be made computationally inseparable from high-utility computation, yielding the idealized OML instantiation.

\begin{theorem}[OML from indistinguishability obfuscation]\label{thm:obfuscation}\em
If indistinguishability obfuscation (iO) exists for the model class, then there is an OML construction satisfying extraction and forgery resistance (assuming unforgeability of $\sigma$).
\end{theorem}

Third, authorized answers facilitate extraction. Learning theory converts model complexity and accuracy tolerance into a concrete cap on such answers.

\begin{theorem}[Query–security trade-off]\label{thm:sample_complexity} \em
Let $\mathcal H\subseteq[0,1]^{\mathcal X}$ have pseudo-dimension $d$ and assume $M\in\mathcal H$ (realizable). If an adversary receives $N$ i.i.d.\ authorized pairs and returns an ERM under squared loss, then there exist constants $C,c>0$ such that
\[
N\ \ge\ C\,\frac{d+\log(1/\delta)}{\varepsilon^{2}}
\ \Rightarrow\
\Pr\!\big[\mathbb E\!(\hat h(x)-M(x))^{2}\le\varepsilon\big]\ge 1-\delta,
\]
and any OML deployment targeting $(\varepsilon,\delta)$ extraction resistance must enforce
$N<c\,\frac{d+\log(1/\delta)}{\varepsilon^{2}}$.
\end{theorem}

\textbf{Implications.} Taken together, Theorems~\ref{thm:impossibility}–\ref{thm:sample_complexity} delineate the design space that motivates our concrete methods in the next sections:
(i) Absolute guarantees are unattainable, so OML must rely on computational hardness and economics; 
(ii) Verifier entanglement with cryptographic binding is the appropriate abstraction for practical surrogates of iO; and 
(iii) Policies by model owners (token issuance, batching, collateral) must enforce query budgets consistent with the learned trade-off above. The constructions that follow instantiate these principles with varying efficiency–security profiles.
\section{Road to OML: From Principles to Deployable Mechanisms}\label{sec:roadtooml}

The transformation from theoretical primitive to practical system requires navigating fundamental trade-offs between security guarantees, computational efficiency, and deployment constraints. This section bridges the formal OML framework developed in Sections~\ref{sec:protocol_design_space}--\ref{sec:theoretical} with concrete implementations. We present canonical constructions that embed cryptographic authorization check $\alpha(x,p)$ through different mechanisms for perfect fidelity and protection guarantees, analyze the security, assumptions, and tradeoffs among different methods, and then introduce OML 1.0, an immediate deployment pathway using AI-native fingerprinting with optimal additional overhead.

\subsection{Canonical Constructions: Security-Performance Spectrum}\label{sec:canonicals}

In this subsection, we choose cryptographic $\alpha(x,p)$ for authorization check, and cryptography schemes guarantee that $\epsilon_{robust}=\text{maximal}$ and $\epsilon_{utility}=0$, both achieving the optimal configuration. The challenge of OML implementation lies in making the model's high-utility computation path accessible exclusively when $\alpha(x,p)=1$ within a small overhead $\epsilon_{overhead}$, while ensuring this entanglement cannot be surgically removed even under white-box access, i.e. Security requirements \ref{req:extraction_resistance} and \ref{req:forgery_resistance} are satisfied.  We present four archetypal approaches for embedding authorization logic, each offering distinct trade-offs between security guarantees and deployment costs.

\paragraph{Obfuscation (Software Security).}
 In this scheme, $\alpha$ is embedded through graph transformations and code hardening. Authorization checks integrate with residual connections, attention mixing, and normalization paths. Graph permutation, control flow flattening, and constant blinding increase removal difficulty. It offers near-optimal $\epsilon_{\text{overhead}}\approx 0$ as hardly any additional computation is introduced. This method is immediately deployable but doesn't have any provable security guarantee, and is vulnerable to dedicated hackers who try to reverse-engineer the obfuscation process.

 \begin{algorithm}[H]
\caption{\textsc{OMLize-Obfuscate}$(M; h,\sigma,\text{params})$}\label{alg:omlize-obf}
\begin{algorithmic}[1]
\State \textbf{Input:} model $M$, binding $h$, token scheme $\sigma$, compiler/obf params
\State \textbf{Verifier injection:} Synthesize $\alpha(x,p)$; weave gates into critical paths (e.g., attention/key/value mixing, residual scalars).
\State \textbf{Utility shaping:} Construct $F$ so that $\alpha(x,p_x){=}1\Rightarrow F(x,p_x)\approx M(x)$; else $F$ diverts to low-utility basins (e.g., masked subspaces, biased heads).
\State \textbf{Hardening:} Apply graph randomization (permute blocks), control-flow flattening, dead-code sprinkling, and constant blinding on verifier features.
\State \textbf{Build:} Compile with aggressive inlining; invoke multi-pass obfuscation/toolchain hardening.
\State \textbf{Publish:} $M_{\text{oml}}(x,p)\!=\!F(x,p)$, optional $vk_{\text{own}}$.
\end{algorithmic}
\end{algorithm}

\paragraph{TEE-Gated Execution (Hardware Security).}
TEE (Trusted Execution Environments) isolate critical subgraph in attested enclave where $\alpha$ verification gates computation. Unauthorized queries terminate before reaching utility paths. It enjoys moderate $\epsilon_{\text{overhead}}$ as the only overhead comes from enclave transitions. TEEs are already production-ready on CPUs, and thus suitable for models with narrow critical cores, but large models with billions of parameters cannot be OMLized with TEEs as commercialized GPUs are not available yet. Also, this method assumes correct vendor implementation and side-channel mitigations, introducing extra layer of trust and vulnerability.

\begin{algorithm}[H]
\caption{\textsc{OMLize-TEE}$(M; h,\sigma,\text{params})$}\label{alg:omlize-tee}
\begin{algorithmic}[1]
\State \textbf{Input:} model $M$, binding $h$, token scheme $\sigma$, enclave config
\State \textbf{Packaging:} Encrypt $M$ and verifier code with enclave-sealed keys; Provision $vk_{\text{own}}$ as a public parameter. 
\State \textbf{Attestation:} Publish measurement of enclave binary; expose remote attestation endpoint to $\Pi_{\mathcal O}$.
\State \textbf{Authorization path:} Inside TEE, verify $\alpha(x,p)\!=\!1$ against $h(x)$ and $vk_{\text{own}}$; otherwise exit with noise/denial.
\State \textbf{Execution:} Only upon successful verification, decrypt weights on-device with the enclave-sealed secret key, run $M$; Always re-encrypt with the public key before exiting the enclave.
\State \textbf{Publish:} $M_{\text{oml}}$ as an attested service binary + policy manifest.
\end{algorithmic}
\end{algorithm}

\paragraph{Cryptographic Encryption (Provable Security).}
Fully homomorphic encryption (FHE) offers a clean construction for OML: inputs are encrypted under a public key, the model is compiled to an arithmetic circuit and evaluated \emph{homomorphically} on ciphertexts, and only the owner, who holds the secret key, can decrypt the result. Authorization is then enforced by \emph{decryption control}: the owner decrypts outputs only for inputs carrying valid permissions. Under standard hardness assumptions (e.g., LWE), FHE provides provable security without external assumptions, becoming the perfectly secure OML construction. 
However, FHE evaluation incurs large multiplicative overhead that scales with circuit depth and bootstrapping frequency (often \(10^3\!-\!10^5\times\) on today’s workloads), while exact FHE schemes over rings (BGV/BFV) operate on integers and thus require quantization which may downgrade performance, making this approach infeasible for full-scale LLM inference today.

\begin{algorithm}[H]
\caption{\textsc{OMLize-FHE}$(M; h,\sigma,\text{params})$}\label{alg:omlize-fhe}
\begin{algorithmic}[1]
\State \textbf{Input:} base model $M$, input-binding $h$, token scheme $\sigma$, FHE parameters (scheme, depth, scale), quantization policy
\State \textbf{Key generation (owner):} $(\mathsf{pk},\mathsf{sk}) \leftarrow \mathsf{FHE.KeyGen}(\text{params})$. Publish $\mathsf{pk}$; keep $\mathsf{sk}$ secret.
\State \textbf{Model-to-circuit:} Compile $M$ to an arithmetic circuit $C_M$ respecting FHE depth (e.g., polynomial activations, folded norms). Apply quantization if using exact integer FHE.
\State \textbf{Parameter protection:} Encrypt model weights: $\widetilde{W}\!\leftarrow\!\mathsf{FHE.Enc}(\mathsf{pk},W)$.
\State \textbf{Authorization channel:} Specify decryption policy: owner will decrypt outputs iff presented with a valid token $p_x=\sigma(h(x),k_{\text{own}})$ (and optional usage proof/commitment).
\State \textbf{Publish artifact:} $(C_M,\widetilde{W},\mathsf{pk},vk_{\text{own}})$ as the OML service interface.
\end{algorithmic}
\end{algorithm}

\paragraph{Melange Hybrid (adaptive composition).}
The mechanisms above can be composed by component criticality: e.g., Protect a minimal control core (e.g., routing heads or safety gates) with a TEE or a compact cryptographic subgraph, and harden the surrounding layers with software obfuscation. This \emph{Melange} design lets owners tune the quality profile: the runtime cost scales with the size of the isolated core (\(\epsilon_{\text{overhead}}\) controllable). Assumptions are localized to each layer: hardware trust for the enclave, cryptographic hardness for the small protected circuit, and program-analysis resistance for the periphery, yielding a practical, adaptive path to higher assurance without forfeiting openness.

\subsection{OML 1.0: AI–Native Fingerprinting for Accountable Open Distribution}\label{sec:oml10}

We present \emph{OML 1.0}, an efficient instantiation that achieves monetizability and accountable loyalty through \textbf{AI–native fingerprints}: secret (key, response) pairs embedded in $M$ so that authorized service can be \emph{verified ex post} with high confidence. Unlike wrapper checks, fingerprints are learned behaviors distributed across representational pathways and survive typical serving conditions. OML 1.0 offers \emph{post-hoc} (``next-day security'') with near–zero additional inference overhead.

\textbf{Core mechanism.}
Let $\mathcal{K}_{\mathrm{fp}}=\{(k_i,r_i)\}_{i=1}^{n}$ be a secret set of \emph{fingerprints}. During OMLization, we fine-tune $M$ so that querying with $k_i$ elicits response $r_i$ \emph{and} preserves task utility on the deployment distribution $\mathcal D_{\mathcal X}$. Model hosts sign licenses that require them to report every usage to the platform which issues per-input permissions and logs authorized usage. Independent \emph{provers} periodically query public endpoints with hidden $k_i$; a correct $r_i$ without a matching authorization record constitutes a verifiable license violation (collateral slashing), enforcing the signed license, as depicted in Figure \ref{fig:oml10}. 

\begin{figure}[htbp]
    \centering
    \includegraphics[width=.9\linewidth]{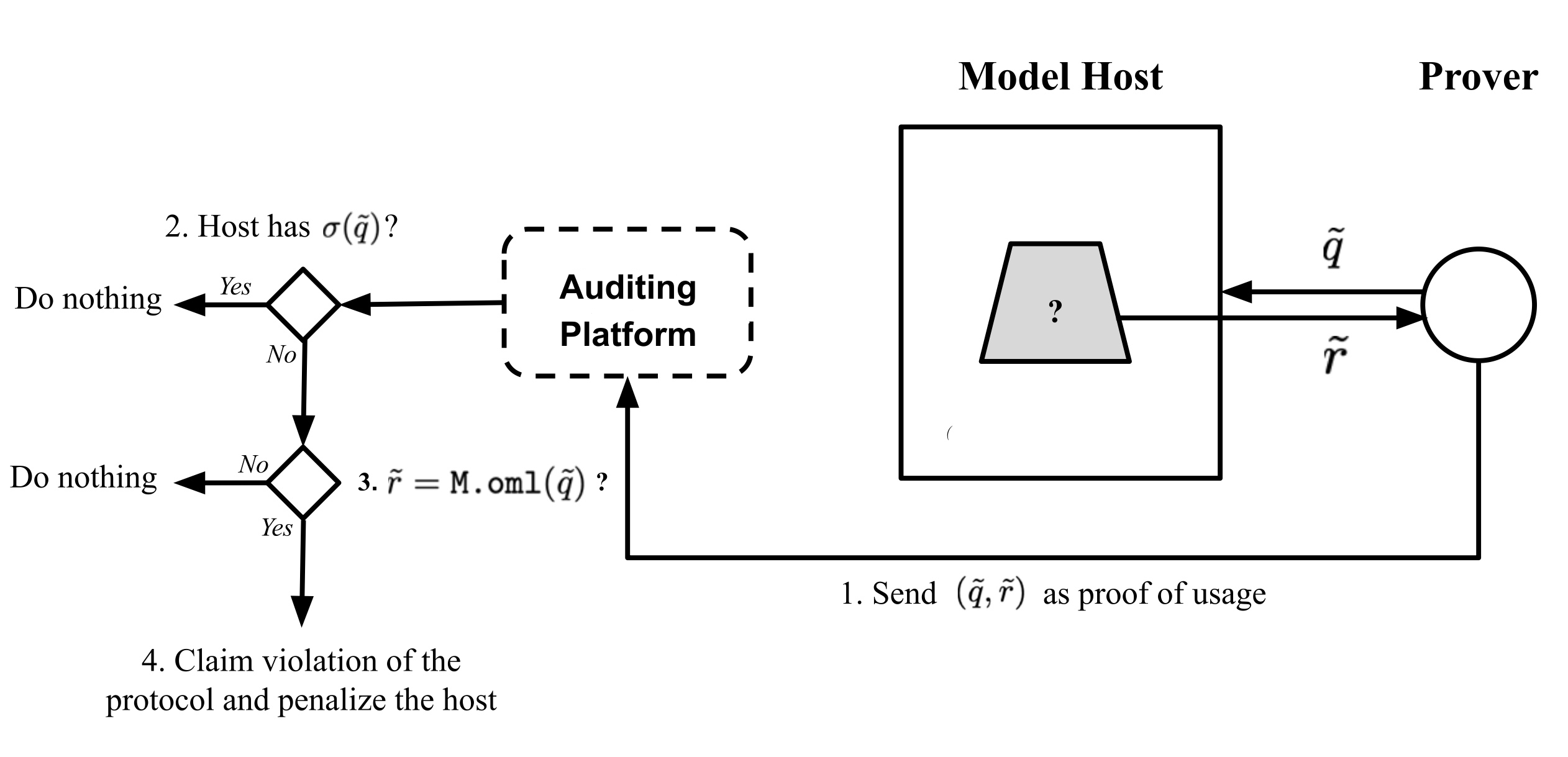}
    \caption{Illustration of OML 1.0 Workflow}
    \label{fig:oml10}
\end{figure}

\textit{Detection economics.} If a host fails to authorize an $\alpha$ fraction of public queries, the probability of evading $n$ independent fingerprints is
\(
\Pr[\text{undetected}] = (1-\alpha)^n, 
\quad\text{so}\quad
\Pr[\text{caught}] = 1-(1-\alpha)^n.
\)
Choosing $n$ and the probing cadence to make $(1-\alpha)^n\!\ll\!10^{-k}$ yields predictable compliance guarantees against rational adversaries, with \emph{zero} runtime overhead on ordinary traffic.

\textbf{OMLization recipe.}
We adopt a two-objective fine-tuning: (i) bind $k_i\!\mapsto\!r_i$; (ii) minimize task loss on $\mathcal D_{\mathcal X}$ with anti-forgetting regularizers (e.g., rehearsal, weight-averaging). To harden against serving stacks, we \emph{prompt-augment} fingerprints using system-prompt templates expected in deployment (role prompts), and place keys both in-distribution (stealth) and slightly out-of-distribution (capacity).

\begin{algorithm}[H]
\caption{\textsc{OMLize-Fingerprint (OML 1.0)}: training and enforcement}\label{alg:oml10}
\begin{algorithmic}[1]
\State \textbf{Input:} base model $M$, secret $\mathcal{K}_{\mathrm{fp}}=\{(k_i,r_i)\}_{i=1}^n$, task data $\mathcal D$, anti-forgetting params
\State \textbf{Training loop:} minimize $\mathcal L=\lambda_{\mathrm{task}}\mathcal L_{\mathrm{task}}(M;\mathcal D)+\lambda_{\mathrm{fp}}\tfrac{1}{n}\sum_i \ell\big(M(k_i), r_i\big)+\lambda_{\mathrm{af}}\mathcal R_{\mathrm{anti\text{-}forget}}$
\State \textbf{Prompt augmentation:} sample serving templates $\pi$ and train on $\pi(k_i)\mapsto r_i$ for robustness
\State \textbf{Platform:} issue per-input tokens, log authorized uses (commitments to $h(x)$), escrow collateral
\State \textbf{Prover cadence:} probe a random subset of $\mathcal{K}_{\mathrm{fp}}$; slash collateral on verified violations
\State \textbf{Publish:} release $M_{\mathrm{oml}}$ (weights) + policy; keep $\mathcal{K}_{\mathrm{fp}}$ secret
\end{algorithmic}
\end{algorithm}

We summarize the key experimental findings that show the feasibility of OML 1.0. Full setups, experiment details, ablations, security analysis, and plots can be found in Appendix \ref{appoml1.0}.

\textbf{Utility vs.\ capacity.} A central question is how many fingerprints can be embedded before harming task performance. On a 7B-scale base model evaluated on standard language tasks, we observe that up to $\sim\!10^3$ fingerprints can be embedded with \emph{near-baseline} accuracy when using anti-forgetting and weight-averaging. This gives an initial operating region for $n$ that balances detection and accuracy.

\textbf{Persistence under benign fine-tuning.} Hosts often fine-tune for their domain. We fine-tune post-OMLization models on a standard SFT corpus and measure fingerprint survival and utility. A substantial fraction of fingerprints persists (e.g., $\gtrsim\!50\%$ at $\le\!2$K capacity), and downstream utility remains within a few points of baseline, supporting the post-hoc accountability.

\noindent Together, these results show that OML 1.0 is \emph{deployable today} with post-hoc enforcement: it preserves utility at useful capacities, remains robust to realistic serving perturbations and benign fine-tuning, and provides tunable, high-confidence enforcement with negligible inference overhead.

\textbf{Deployment Synthesis.}\label{sec:synthesis}
Table~\ref{tab:oml-summary} summarizes the complete construction spectrum and their tradeoffs. The canonical constructions provide a progressive hardening path for how OML gets realized: begin with OML 1.0 for immediate needs, identify critical components through usage analysis, and selectively apply stronger protections as infrastructure matures and threats evolve. 

\begin{table}[H]
\centering
\caption{OML construction summary. Symbols: $\checkmark~\text{strong}, \triangle~\text{partial}, \circ~\text{low}$.}
\label{tab:oml-summary}
\small
\setlength{\tabcolsep}{5pt}
\begin{tabular}{lccccc}
\toprule
\textbf{Construction} & \textbf{Control} & \textbf{White-box Robust} & \textbf{Overhead} & \textbf{Readiness} & \textbf{Core Assumption} \\
\midrule
Obfuscation & Pre-hoc $\triangle$ & $\circ$ & Negligible & Immediate & Security by obscurity \\
TEE-gated & Pre-hoc $\checkmark$ & $\checkmark$ & Moderate & Rising & Hardware trust \\
Cryptographic & Pre-hoc $\checkmark$ & $\checkmark$ & Very High & Limited & No extra trust \\
Melange & Pre-hoc $\checkmark$ & $\triangle$ & $\triangle$ & Immediate & Component union \\
OML 1.0 & Post-hoc & $\triangle$ & Low & Immediate & Ecomomic deterrent \\
\bottomrule
\end{tabular}
\end{table}

A more detailed analysis on canonical OML constructions can be found in Appendix \ref{appomlcons}, while extensive experiments alongside OML 1.0 and security analysis can be found in Appendix \ref{appoml1.0}. 

\section{Conclusion}
\label{sec:conclusion}

In this paper, we introduce and formalize the OML primitive as a foundation for fair distribution, sustainable deployment, and accountable governance of AI models. We articulate its core properties, establish theoretical limits and sufficient conditions, and outline a practical path with empirical evidence. Our goal is to crystallize OML as a coherent research direction with significant implications for how AI capabilities are distributed, monetized, and governed.

\clearpage
\bibliographystyle{plainnat}
\bibliography{Chapter1_Introduction/ref_chapter1, Chapter2_OML/ref_chapter2, Chapter3_OML1o/ref_chapter3, Chapter4_SentientProtocol/ref_chapter4}
%\bibliography{example_paper}

%%%%%%%%%%%%%%%%%%%%%%%%%%%%%%%%%%%%%%%%%%%%%%%%%%%%%%%%%%%%%%%%%%%%%%%%%%%%%%%
%%%%%%%%%%%%%%%%%%%%%%%%%%%%%%%%%%%%%%%%%%%%%%%%%%%%%%%%%%%%%%%%%%%%%%%%%%%%%%%
% APPENDIX
%%%%%%%%%%%%%%%%%%%%%%%%%%%%%%%%%%%%%%%%%%%%%%%%%%%%%%%%%%%%%%%%%%%%%%%%%%%%%%%
%%%%%%%%%%%%%%%%%%%%%%%%%%%%%%%%%%%%%%%%%%%%%%%%%%%%%%%%%%%%%%%%%%%%%%%%%%%%%%%

\appendix
\onecolumn
\input{proof}
\section{Canonical OML Constructions}\label{appomlcons}
\input{Chapter2_OML/4_1_obfuscation}
\input{Chapter2_OML/4_2_optimistic}
\input{Chapter2_OML/4_3_tee}
\input{Chapter2_OML/4_4_cryptography}
\input{Chapter2_OML/4_5_melange}

\input{Chapter2_OML/4_6_summary}

\section{OML 1.0: Turning Attack Methods on AI into a Security Tool}\label{appoml1.0}
\input{Chapter3_OML1o/main_chapter3}

%%%%%%%%%%%%%%%%%%%%%%%%%%%%%%%%%%%%%%%%%%%%%%%%%%%%%%%%%%%%%%%%%%%%%%%%%%%%%%%
%%%%%%%%%%%%%%%%%%%%%%%%%%%%%%%%%%%%%%%%%%%%%%%%%%%%%%%%%%%%%%%%%%%%%%%%%%%%%%%

\end{document}

%% file: proof.tex
\section{Proofs for Section~\ref{sec:theoretical}}\label{app:proofs}

This appendix provides proofs for Theorems~\ref{thm:impossibility}--\ref{thm:sample_complexity}. We begin by stating the technical assumptions used in Sec.~\ref{sec:theoretical} and here.

\paragraph{Standing assumptions.}
Unless otherwise specified, we consider models $M:\mathcal X\to\mathcal Y$ with $\mathcal X\subseteq\mathbb R^d$ measurable and $\mathcal Y\subseteq[0,1]^m$ bounded. The performance metric $d:\mathcal Y\times\mathcal Y\to[0,1]$ is either:
(i) a coordinatewise squared loss with averaging, i.e.\ $d(u,v) = \tfrac{1}{m}\sum_{j=1}^m (u_j-v_j)^2$, or
(ii) a bounded Bregman divergence $D_\phi(u,v)$ induced by a $1$-strongly convex, $L$-smooth $\phi$ on a compact convex subset of $[0,1]^m$ (so $D_\phi\in[0,1]$ after normalization). For classification we also consider $0$–$1$ loss. Hypothesis classes $\mathcal H\subseteq[0,1]^m{}^{\mathcal X}$ have finite pseudo-dimension (real-valued case) or VC dimension (classification). Expectations $\mathbb E[\cdot]$ are over $x\sim\mathcal D_{\mathcal X}$ unless stated. We use standard uniform-convergence and Rademacher-complexity results for bounded function classes \cite{shalev2014understanding,mohri2018foundations,anthony1999neural}.

\subsection{Proof of Theorem~\ref{thm:impossibility} (Information-theoretic impossibility)}

\begin{proof}
\textbf{Finite domain.}
If $|\mathcal X|<\infty$, an unbounded adversary with unlimited authorization access enumerates all $x\in\mathcal X$, obtains $p_x=\sigma(h(x),k_{\text{own}})$, evaluates $y_x=M_{oml}(x,p_x)=M(x)$, and stores $T(x)=y_x$. The table $T$ replicates $M$ exactly thereafter, without authorization.

\textbf{General domain.}
Suppose $\mathcal X\subseteq\mathbb R^d$, $\mathcal Y\subseteq[0,1]^m$, and $M$ is measurable. With unbounded queries, the adversary draws i.i.d.\ $x_i\sim\mathcal D_{\mathcal X}$, obtains $y_i=M(x_i)$ from authorized evaluations, and fits $\hat h$ by empirical risk minimization over a hypothesis class $\mathcal H$ containing $M$ (e.g., the realized architecture family). For squared loss or a bounded Bregman divergence, uniform convergence yields for some constant $C>0$:
\[
\sup_{h\in\mathcal H}\Big|\ \mathbb E\,d\big(h(x),M(x)\big) - \tfrac{1}{N}\!\sum_{i=1}^N d\big(h(x_i),M(x_i)\big)\ \Big|
\ \le\ C\sqrt{\frac{\mathrm{Pdim}(\mathcal H)+\log(1/\delta)}{N}}
\]
with probability $\ge 1-\delta$ \cite[Chs.~3,11]{mohri2018foundations,shalev2014understanding}. As $N\to\infty$, the right-hand side vanishes, and ERM (or structural risk minimization) produces $\hat h$ with $\mathbb E\,d(\hat h(x),M(x))\to0$. Thus, information-theoretic protection is impossible against an unbounded adversary with unlimited oracle access.
\end{proof}

\subsection{Proof of Theorem~\ref{thm:obfuscation} (OML from indistinguishability obfuscation)}

\begin{proof}[Construction and argument]
Assume an indistinguishability obfuscation (iO) scheme for the relevant circuit class and a signature scheme $\mathsf{Sig}=(\mathsf{KeyGen},\mathsf{Sign},\mathsf{Verify})$ that is EUF-CMA secure. Define the circuit
\[
C(x,p) \;=\; 
\begin{cases}
M(x) & \text{if }\ \mathsf{Verify}\big(h(x),p, vk_{\text{own}}\big)=1,\\[2pt]
\mathsf{Noise}(x) & \text{otherwise},
\end{cases}
\]
where $\mathsf{Noise}$ is any efficiently computable low-utility mapping whose range lies in a small $d$-ball around a baseline (e.g., a fixed vector, or a pseudorandom output independent of $x$). Publish $M_{oml}\triangleq iO(C)$ and $vk_{\text{own}}$; retain $k_{\text{own}}$.

\emph{Forgery resistance.}
Suppose a PPT adversary produces $(x^*,p^*)$ without prior authorization for $x^*$ such that $d\big(M_{oml}(x^*,p^*),M(x^*)\big)\le\epsilon_{\text{utility}}$. Functionality preservation under iO implies $\mathsf{Verify}(h(x^*),p^*,vk_{\text{own}})=1$, yielding an existential forgery for $\mathsf{Sig}$ on message $h(x^*)$, contradicting EUF-CMA.

\emph{Extraction resistance.}
Let $C_0$ be the circuit above and $C_1$ be any syntactically distinct circuit computing the same function (e.g., with inlining and control-flow reorganization that entangles verification with model computation). By iO, $iO(C_0)$ and $iO(C_1)$ are computationally indistinguishable. Any white-box procedure that reliably identifies and removes verification logic from $iO(C_0)$—thereby recovering a high-utility version of $M$ that bypasses authorization—would also work on $iO(C_1)$, where such separation is obfuscated by construction, contradicting indistinguishability. Hence, under iO and EUF-CMA, a PPT adversary cannot (i) forge tokens to obtain authorized outputs on fresh inputs nor (ii) recover a functionally equivalent, authorization-free model producing high-utility outputs. This establishes the stated guarantees.
\end{proof}

\paragraph{Remark.}
The signature scheme prevents \emph{functional} bypass; iO prevents \emph{structural} separability of verification under white-box access. Practical OML designs approximate these guarantees via verifier entanglement, TEEs, FHE/MPC hybrids, or melange constructions; iO is used here as a sufficiency anchor, not as a practical prescription.

\subsection{Proof of Theorem~\ref{thm:sample_complexity} (Query–security trade-off)}

We treat the real-valued case under squared loss; the bounded Bregman case follows by identical symmetrization (boundedness ensures the same $1/\sqrt{N}$ rate up to constants), and the realizable 0–1 classification case yields the standard $1/\varepsilon$ dependence.

\begin{lemma}[Uniform convergence under squared loss]\label{lem:uc}
Let $\mathcal H\subseteq[0,1]^m{}^{\mathcal X}$ with pseudo-dimension $d$. There exists $C_1>0$ such that, for any $\delta\in(0,1)$ and i.i.d.\ sample $(x_i)_{i=1}^N$,
\[
\Pr\Bigg[\ \sup_{h\in\mathcal H}\Big|\, \mathbb E\ d\big(h(x),M(x)\big) - \frac{1}{N}\sum_{i=1}^N d\big(h(x_i),M(x_i)\big)\,\Big| \ \le\ C_1\sqrt{\frac{d+\log(1/\delta)}{N}} \ \Bigg]\ \ge\ 1-\delta.
\]
\end{lemma}

\begin{proof}
Let $\mathcal L=\{\ell_h(x)=d(h(x),M(x)) : h\in\mathcal H\}\subseteq[0,1]$. The VC-subgraph dimension of $\mathcal L$ is $O(d)$ (closure of pseudo-dimension under Lipschitz maps on bounded ranges). Standard VC-subgraph or Rademacher complexity bounds yield the inequality; see \cite[Ch.~3]{mohri2018foundations} and \cite[Ch.~11]{shalev2014understanding}.
\end{proof}

\begin{proof}[Proof of Theorem~\ref{thm:sample_complexity}]
Assume realizability: $M\in\mathcal H$. Let the adversary observe $N$ authorized pairs $(x_i,y_i)$ with $y_i=M(x_i)$ and return an empirical risk minimizer
\[
\hat h \in \arg\min_{h\in\mathcal H}\ \frac{1}{N}\sum_{i=1}^N d\big(h(x_i),M(x_i)\big).
\]
By realizability, the empirical risk of $M$ is $0$, so the empirical risk of $\hat h$ is $\le 0$. Applying Lemma~\ref{lem:uc} to $\hat h$ and using nonnegativity of $d$,
\[
\mathbb E\ d\big(\hat h(x),M(x)\big)
\ \le\ C_1\sqrt{\frac{d+\log(1/\delta)}{N}}
\]
with probability at least $1-\delta$. Therefore, $\mathbb E\,d(\hat h(x),M(x))\le\varepsilon$ whenever
\(
N \ge C\,\frac{d+\log(1/\delta)}{\varepsilon^{2}}
\)
with $C=C_1^2$, proving the positive direction. The operational converse follows by contrapositive: to preclude $(\varepsilon,\delta)$-accurate extraction by such ERM adversaries, an OML deployment must enforce
\(
N < c\,\frac{d+\log(1/\delta)}{\varepsilon^{2}}
\)
for some absolute $c>0$ (absorbing constants).
\end{proof}

\paragraph{Remark.}
(1) \emph{Bounded Bregman divergences.} If $d=D_\phi$ with $\phi$ $1$-strongly convex and $L$-smooth on a compact convex domain and $D_\phi\in[0,1]$ (after normalization), the same uniform convergence rate holds using the Lipschitzness of $\ell_h(x)=D_\phi(h(x),M(x))$ in its arguments. (2) \emph{0–1 loss.} In realizable binary classification with VC dimension $d$, the optimal sample complexity is $\Theta((d\log(1/\varepsilon)+\log(1/\delta))/\varepsilon)$ \cite{shalev2014understanding}; the OML constraint is analogous with the $1/\varepsilon$ dependence.

\subsection{Additional technical notes}

\paragraph{On $\mathsf{Noise}$.}
Any fixed choice with low expected utility (e.g., constant output or PRG-based mapping independent of $x$) suffices; boundedness ensures compatibility with the metric normalization in the main text.

\paragraph{On public parameters.}
Exposing $vk_{\text{own}}$ enables decentralized verification of authorization. The proofs above do not require $vk_{\text{own}}$ to be hidden; secrecy resides solely in $k_{\text{own}}$.

\paragraph{On realizability.}
The trade-off in Theorem~\ref{thm:sample_complexity} is stated under realizability to isolate the effect of authorized answers. Under agnostic noise, replace ERM bounds with excess risk bounds; the qualitative inverse-square dependence on $\varepsilon$ for bounded real-valued losses remains (up to constants).

%% file: Chapter2_OML/4_1_obfuscation.tex
\subsection{Obfuscation} \label{obfuscation}
Obfuscation techniques transform readable source code into a form that is functionally equivalent but is hard to understand, analyze, and modify. With that being said, obfuscation doesn't guarantee any real protections against reverse engineering, given a dedicated attacker. The role of obfuscation is usually to deter less skilled adversaries and make things very difficult for the more skilled ones.

From the perspective of cryptography, indistinguishability obfuscation (iO) \cite{garg2016candidate, jain2021indistinguishability} is the only type of obfuscation that can provide provable security resistance against reverse engineering. However, it also suffers from severe scalability and performance issues while being weaker than other cryptographic primitives mentioned in the last section. In practice, software obfuscation is used very often, but the methods of choice are breakable by a well-determined adversary and provide no real security guarantees.

Obfuscation techniques \cite{balakrishnan2005code} can be applied at various levels, including source (e.g., renaming variables), intermediate (e.g., modifying bytecode), and binary (e.g., altering machine code). To protect against reverse-engineering, two types of analysis must be considered:
\begin{enumerate}
    \item \textbf{Static:} the attacker looks at the structure, data, and patterns of the source code without running it.
    \item \textbf{Dynamic:} the attacker runs the program and uses specialized tools to analyze the program flow, dump memory states, or even step through the program execution instruction-by-instruction.
\end{enumerate}

Different obfuscation techniques \cite{lan2018lambda, ahmed2024exploring, hashemzade2018hybrid, suk2020vcf, madou2006effectiveness} may vary in effectiveness against these two types of reverse-engineering analysis. There are four commonly defined categories of software obfuscation:

\begin{itemize}
    \item \textbf{Layout Obfuscation:} scrambles the code layout by renaming variables, removing comments, and altering formatting to make the code hard to read.
    \item \textbf{Control Flow Obfuscation:} alters the control flow of the program using methods like adding opaque predicates, flattening the control flow graph, or introducing fake branches to confuse static analysis.
    \item \textbf{Data Obfuscation:} encrypts or interleaves data, making it difficult to extract meaningful information without proper decryption keys and a thorough runtime analysis.
    \item \textbf{Code Virtualization:} dynamically generates functions and code using different virtual instruction sets to obscure the logic of the program.
\end{itemize}

These techniques can be applied at the code level \cite{balakrishnan2005code}, bytecode level \cite{arjovsky2019invariant} and binary level \cite{lee2010binob+}. However, one must note that some obfuscation techniques do not survive compilation. Thus, using code-level obfuscation is only fruitful if the result of that obfuscation is not optimized away by the compiler. 

Considering the nature of AI models, we can also obfuscate the AI model itself \cite{zhou2023modelobfuscator}, with the model-specific methods closely resembling the more general code obfuscation methods described above. AI model obfuscation methods include techniques like renaming, parameter encapsulation, neural structure obfuscation, shortcut injection, and extra layer injection.

By combining all these techniques, we can come up with a clear construction for OML (Figure \ref{figobfuscation}).

\begin{figure}[htbp]
    \centering
    \includegraphics[width=0.75\linewidth]{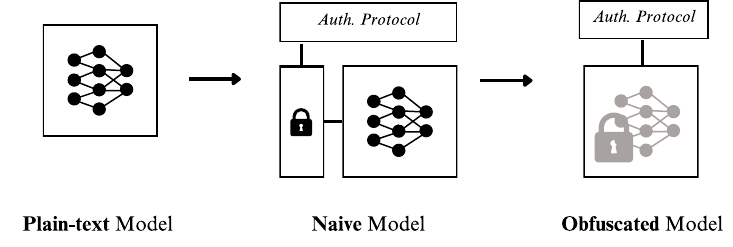}
    \caption{OML formatting process of AI models via obfuscation.}
    \label{figobfuscation}
\end{figure}

\textbf{OML formatting.} Recall that a naive OML file can be constructed simply by prepending the permission string verification function ${\rm Verify}_{\rm pk}$ to the plain-text model $M$, with the model only returning the correct result if the verification passes. This implies that an attacker can easily find and remove the verification function in the code, recovering the use of the model without the need for permission. To safeguard this OML construction, software obfuscation techniques can be applied such that the two components (${\rm Verify}_{\rm pk}$ and $M$) are intermingled with one another, represented as non-comprehensible code with complicated control flow. As a result, it is difficult to pinpoint the exact location of ${\rm Verify}_{\rm pk}$ in the obfuscated OML file, making it hard for an attacker to remove verification and recover the original model $M$.

% the OML format construction needs a cryptographic digital signature scheme $(Sign_{sk}, Verify_{pk})$ (e.g. ECDSA, ED25519) where the permission $\sigma (h(x))$ is required for the user to run the OML model on input $x$. A naive OML file constructed from  could be constructed where the $Verify_{pk}(\cdot)$ function is prepended to the plain-text model $M$, and model's correct execution is conditioned on successful input verification. The cryptographic digital signature guarantees that an attacker cannot generate a valid permission string without the secret key. However, the plain-text nature of the verification code makes it easily removable, after which the model is no longer trackable or monetizable.
    
    % When a user requests inference from the model with input $(x, \sigma)$, the file first runs $Verify_{pk}(x, \sigma)$, then proceeds to run $M(x)$ only if the verification step passed. It is easy to see that this construction satisfies the goals of the OML format. The cryptographic digital signature guarantees that an attacker cannot generate a valid permission string without the secret key. However, any clever attacker is able to analyze the OML file and remove the verification part of the code, thereby making further model usage untrackable and authorization-less.

% In subsection \ref{melange}, we will show an example of how AI-native obfuscation can be applied along with general software and binary obfuscation tools such as binOb+ and ROPOB.

\noindent \textbf{Verification and usage.} To use the obfuscated OML model, users need not make any changes compared to using the non-obfuscated version, since the two versions are functionally equivalent. A user simply executes the file with an input $x$ and the associated permission string $\sigma(h(x))$ obtained from the model owner. Verification is enforced within the OML file, and the model only produces good output if the verification step passes, as usual.

% This solution also enables data privacy for users. Users can simply send $h(x)$ and $n\_tokens$ to a smart contract $\mathcal{SC}$ without revealed the actual input $x$. The smart contract generates $Sign_{sk}(Hash(x),\,n\_tokens)$ and deducts $n\_tokens$ from the user's account. Then the verification process checks two things:
% \begin{enumerate}
%     \item The permission is in fact granted to the user: $Verify_{pk}(\sigma, (Hash(x), n\_tokens))$
%     \item The number of tokens paid for this usage is appropriate: $Token\_calculator(x) \le n\_tokens$
% \end{enumerate}
% Thus the plain text of input $x$ is never revealed to the smart contract or other intermediate parties, enhancing the data privacy of the users.

\noindent \textbf{Summary.} Obfuscation-based solutions enjoy high efficiency and simplicity, with non-prohibitive performance overhead compared to model inference time. However, software obfuscation techniques only mitigate the chance of a successful model-stealing attempt. Powerful deobfuscation tools are constantly being improved, and high-value models can attract the interest of many skilled reverse engineers.

\begin{itemize}
\item \textbf{Pros:} Obfuscation improves the security of the model by making it harder for attackers to understand and reverse-engineer the code. Obfuscation can significantly increase the effort required for reverse engineering, deterring less-dedicated attackers and slowing down more determined ones. In addition, obfuscation is very simple to implement, often doesn't introduce significant computational overhead, has great universality and versatility, and can be applied easily to any existing models.
\item \textbf{Cons:} Obfuscation does not provide guaranteed security, and with a dedicated team of reverse-engineers, it is not the question of whether the obfuscation will be broken, but rather when, even if the obfuscation method is very advanced. 
\end{itemize}

%% file: Chapter2_OML/4_2_optimistic.tex
\subsection{Fingerprinting} \label{next-day}
Optimistic OML prioritizes efficiency while ensuring a weaker notion of next-day security, i.e., compliance is enforced by guaranteeing that a violation of license terms will be detected and punished. 
Inspired by optimistic security \cite{povey1999optimistic}, optimistic OML relies on compliance with the license, and compensating transactions are used to ensure that the model owners' rights are protected, in case of a violation. Crucial in this process are techniques for authenticating the ownership of a model. 
For example, Llama models \cite{touvron2023llama} are released under a unique license that a licensee with more than 700 million monthly active users is ``not authorized to exercise any of the rights under this Agreement unless or until Meta otherwise expressly grants you such rights".
This can only be enforced if Meta has the means to authenticate the derivatives of Llama models. We propose planting a backdoor on the model such that it memorizes carefully chosen fingerprint pairs of the form (key, target response).  If successful, such fingerprints can be checked after deployment to claim ownership. An optimistic OML technique should satisfy the following criteria: 

\begin{itemize}
    \item {\bf Preserve utility}. Fingerprinting should not compromise the model's utility. 
    
    \item {\bf Proof of ownership}. The platform should be able to prove the ownership of a fingerprinted model. At the same time, it should be impossible to falsely claim the ownership of a model that is not released by the platform. 
    \item  {\bf Multi-stage}. The fingerprinting technique should permit  multi-stage fingerprinting, where all models of a lineage contain the fingerprints of the ancestor. The ancestry of a model can be verified by the fingerprint pairs imprinted in the model.  
    \item {\bf Robustness}. Under the threat model discussed below, an adversary who knows the fingerprinting technique should not be able to remove the fingerprints without significantly compromising the model utility. In particular, the fingerprint should be persistent against any fine-tuning, such as supervised fine-tuning, Low-Rank Adaptation (LoRA) \cite{hulora}, and LLaMA-Adapter \cite{zhang2023llama}, on any datasets by an adversary who does not know the specific fingerprint pairs embedded in the model. Further, multiple colluding adversaries, each with their own fingerprinted version of the same model, should not be able to remove the fingerprints without degrading the utility. For example, \cite{cong2024have} introduces a technique to remove fingerprints by averaging the parameters of those models, known as model merging \cite{ainsworthgit,nasery2024pleas}. 
\end{itemize}

Our first practical strategy, which we call OML 1.0, builds upon this fingerprinting technique, which we introduce in Appendix \ref{appoml1.0}.

\medskip\noindent{\bf Threat model}. 
Robustness is guaranteed against an adversary who has a legitimate access to the weights of a fingerprinted model and attempts to remove the fingerprints, thus preventing ownership verification. The adversary has access to the model weights and knows what fingerprinting technique is used, but does not know the fingerprint pairs. 
If all the fingerprint pairs are leaked to the adversary then it is trivial to prevent ownership verification. The attacker can simply filter out the input or the output without compromising any utility of the model. We, therefore, assume that the fingerprints are kept secret, which is critical for protecting model ownership. Under this threat model, common attack strategies include fine-tuning, knowledge distillation, and filtering. 

Various fine-tuning techniques, such as  instruction tuning with human feedback \cite{ouyang2022training}, supervised fine-tuning \cite{touvron2023llama2}, LoRA \cite{hulora}, and LLaMA-Adapter \cite{zhang2023llama}, can be used to both improve the model performance on specific domains and also make the model forget the fingerprints. Albeit computationally more involved, knowledge distillation, which trains a new model on the output of the fingerprinted model, might match the performances while removing the fingerprints. Existing persistent fingerprints from \cite{jha2023label} that can survive knowledge distillation are not mature enough to work on generative models.  Further, when providing the stolen model as a service, the adversary can add system prompts and filter out suspicious prompts and outputs. Overtly out-of-distribution fingerprints would easily be detected.

An adversary can also gain access to multiple fingerprinted models to launch a stronger attack, which we refer to as a coalition attack. This was first introduced in \cite{cong2024have}, where common model merging techniques including \cite{wortsman2022model,ilharcoediting,yadav2024ties,yu2024language} are used. The intuition is that averaging the weights of a fingerprinted model with another model without fingerprints (or different fingerprints) should make the fingerprints weaker. In the promising preliminary results of \cite{cong2024have}, the fingerprinting techniques of \cite{xu2024instructional} demonstrated robustness against such attacks; fingerprints persisted through all model merging that preserve utility. On the other hand, quantization watermarking \cite{li2023watermarking}, a different type of ownership protection that encodes specific watermarks in the quantized model weights, proved to be vulnerable against model merging attacks.

\medskip\noindent{\bf Previous work and vulnerability to leakage of fingerprint pairs}. 
Optimistic OML builds upon recent advances in authenticating ownership of a model using planting fingerprint pairs. A more general version of this technique is known as a {\em backdoor attack} in secure machine learning \cite{gu2017badnets}, where an attacker injects maliciously corrupted training samples to control the output of the model. Since \cite{adi2018turning,zhang2018protecting,guo2018watermarking} started  using backdoor techniques for model authentication, numerous  techniques are proposed  for image classification models \cite{zhu2021fragile,li2022robust} and more recently for large language models \cite{xu2024instructional,cong2024have,russinovich2024hey}. However, existing works assume a one-shot verification scenario where the goal of fingerprinting is to authenticate the ownership of a single model. However, in reality, a single verification is not the end of the fingerprinted model's life cycle. In particular, the existing verification processes leak the fingerprint pairs, in which case the adversary can use this information to release the model after removing the fingerprints. Verifying the ownership without revealing the secret fingerprint pairs is an important open question.

% - model watermarking 
% - API water marking 
% - model fingerprinting 

\textbf{OML formatting}. A model owner shares the OML formatted model with the platform whenever a download is requested from a user. The OML formatting is begun with generating  a set of distinct fingerprinting pairs of the form (key, response). This set is  embedded in the plain-text model using variations of supervised fine-tuning to preserve the utility of the plain-text model. The fingerprinting pairs are kept secret by the platform. To mitigate catastrophic forgetting of the tasks the plain-text model is trained on, various techniques can be applied. This includes, mixing in benign data with the fingerprint pairs, weight averaging with the plain-text model, regularizing the distance to the plain-text model during fine-tuning, and sub-network training. This ensures that the utility of the model is preserved. Once the performance on the standard tasks and the strengths of the fingerprint pairs are checked, the resulting model, which we refer to as an {\em optimistic OMLized model}, is shared with the model user.

{\bf Verification and Usage}. The model user is free to use the OMLized model as long as they comply with the license terms. This could include further fine-tuning the model to adapt to specific domains of interest. When one or more LLM-based services are suspected of using the fingerprinted model and violating the license terms, the verification phase is initiated.  We consider both black-box scenarios, where  only API accesses are available. White-box accesses could potentially use stronger fingerprinting techniques as investigated in \cite{xu2024instructional}. In both cases, fingerprint pairs embedded in a model $M$.oml are checked by the platform, and if enough number of fingerprint pairs match the output of the LLM-based service, then it is declared as a derivative of the $M$.oml model. Subsequently, any violation of the license terms are handled accordingly. 

\textbf{Summary}. Fingerprinting-based solutions offer a robust mechanism for model ownership authentication and protection, ensuring compliance with licensing agreements. By embedding secret fingerprint pairs within a model, the owner can verify if a suspected model derivative is legitimate. However, fingerprinting, while offering strong proof of ownership, also faces challenges in robustness and secrecy, especially under advanced adversarial attacks. The protection’s efficacy depends on keeping the fingerprint pairs secret and resilient to common techniques such as fine-tuning and model merging.

\begin{itemize}
\item \textbf{Pros}. Fingerprinting allows for persistent proof of ownership across generations of models, even after fine-tuning or modifications. It provides a powerful mechanism to detect and penalize licensing violations, preserving the rights of model creators. Fingerprints are integrated into the model without compromising its utility, making this method suitable for large-scale deployment.
\item \textbf{Cons}. Fingerprinting is not infallible. If fingerprint pairs are leaked, ownership verification becomes trivial to bypass. Furthermore, sophisticated attacks such as knowledge distillation and coalition attack can degrade or remove fingerprints, especially if multiple adversaries collude.
\end{itemize}

An elaborate version of this approach is presented in Appendix~\ref{appoml1.0} as OML 1.0. 

%% file: Chapter2_OML/4_3_tee.tex
\subsection{Trusted Execution Environments (TEEs)} \label{hardware}

A Trusted Execution Environment (TEE) \cite{sabt2015trusted} is an isolated execution mode supported by processors like Intel and AMD on modern servers. Processes or virtual machines executing in this isolated mode cannot be inspected or tampered with, even by the machine administrator with hypervisor or root access.

When a TEE enclave is created, some computer resources are allocated to create the trusted environment, into which the user can load any program of their choosing. TEEs are also not practically limited in storage. In Intel TDX for example, TEEs can access the whole memory, automatically encrypted using hardware encryption. Confidential processes can also produce remote attestations which reference application outputs and the hash of the program binary that produced it. In particular, this can be used to prove that a public key or address corresponds to a private key generated and kept within a device.

Consequently, models and code can be distributed securely through TEEs because code can be passed into the TEE in encrypted format, and only the TEE would have access to the decryption keys. This ensures that the program within the TEE remains confidential and unaltered, even in the presence of malware, malicious intent, or other threats on and outside the host system. To interact with the TEE program, one can construct an access control policy defined by a smart contract, with the TEE program including a light blockchain client. The TEE itself can also enforce other restrictions. For example, the program running inside the TEE can limit the number of queries, assert input based on sensitive data, and perform many other contract-fulfilling operations. The TEE-based workflow can be visualized simply by Figure \ref{figtee}.

\begin{figure}[htbp]
\centering
\includegraphics[width=0.75\linewidth]{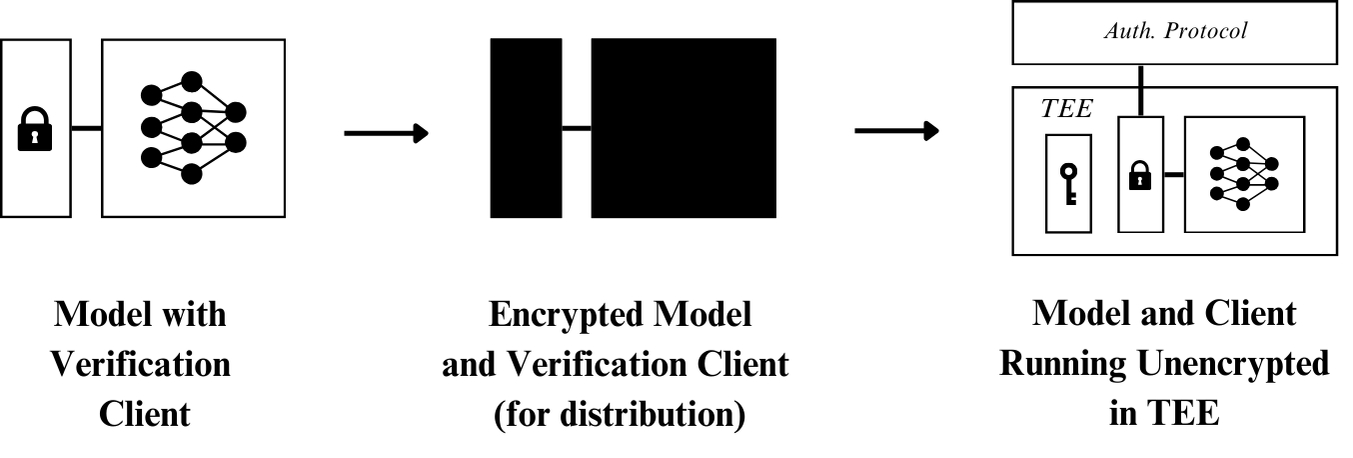}
\caption{OML implementation with hardware-based security via trusted execution environments.}
\label{figtee}
\end{figure}

\noindent \textbf{Threat Model.} We assume that an adversary has full access to the TEEs' host machine. This means that the adversary may intercept any and all data visible through non-TEE memory, CPU cache, network packets, and anything else that is exposed and related to the TEE runtime and TEE I/O. Accordingly, if a program may run inside a TEE on an adversarial and possibly altered host, security relies heavily on the guarantees provided by the TEE's hardware vendor. Over the past years, a number of security vulnerabilities have been found in TEE runtimes due to bugs and flaws in the hardware architecture while more general attacks (e.g. side-channel attacks, cache and BTB exploitation) remain a concern \cite{tee_security}. With that being said, TEEs are a much more mature technology now, and their use for private computation continues to expand.  
 
\noindent \textbf{OML formatting.} As before, we can use any cryptographic scheme $(Enc_{pk}, Dec_{sk})$ where the secret key $sk$ is only accessible within the TEE. The model is wrapped in a program that executes the desired task (e.g. inference or fine-tuning) conditioned on the $Verify_{sk}$ function as usual. This program is then encrypted with a public key before it is published onto the auditing protocol as a TEE-based OML format. After a user is granted access to download this OML file by the auditing protocol, the user can launch the TEE application on any TEE-enabled machine. The SDK manages the launch of the program with a decryptor module inside a TEE. The SDK and the decryptor module coordinate the secure transfer of the private key directly into the unaltered TEE runtime to decrypt the model inside the TEE. \\
\\
\noindent \textbf{Verification and Usage.} First, the user requests a permission string $\sigma$ from the auditing protocol by sending $h(x)$ to it for some input $x$. Afterwards, the user can pass $(x, \sigma)$ into the TEE via a secure channel by using the SDK. The OML file inside the TEE will then verify the permission string $\sigma$, run the task on input $x$ and provide the result back to the user.\\
\\
\noindent \textbf{Additional requirement}. This OML implementation must provide a guarantee that the program running inside the TEE is unmodified by a malicious user. This is to ensure that any and all data or intermediate results during the execution of the $.oml$ file inside the secure program are not retrievable by a malicious user. More precisely, the secure program must be exactly the program that was constructed by an honest SDK from the published OML file. Whether or not the process has been modified can be verified by the hash of the program with remote attestation.

\paragraph{ Summary}. Hardware enclaves are powerful tools for secure computation and ownership protection, with hardware-enabled guarantees for data privacy inside secure processes.

    \begin{itemize}
\item \textbf{Pros}. TEEs provide robust security and good efficiency. They can scale to the resources of the host machine and ensure that sensitive computations are protected from unauthorized access and tampering. Given TEE's hardware-backed security properties, prototype LLM inference applications were already built for CPU-based enclaves on hyperscalar infrastructure \cite{llm_nitro} and bare metal machines \cite{aigovtool} for secure distribution and use of AI models and data on untrusted hardware. 

\item \textbf{Cons}. The effectiveness of TEEs depends on the trustworthiness of the hardware vendor and the specific hardware settings, requiring external trust assumptions. Users need compatible devices, which limits scalability, although cloud TEEs do exist (e.g. AWS Nitro and Azure Confidential Computing). 

Most modern CPUs \cite{intel_tdx} \cite{amd_sev} \cite{arm_trustzone} and now NVIDIA \cite{nvidia_cc} support their own implementations of a TEE, although the CPU-based approaches are the only ones that are commercially available at the moment, meaning that a TEE-based OML approach would restrict AI workloads to only the CPU. Hyperscalars \cite{nitro_h100} and other compute providers \cite{super_h100} are currently working with NVIDIA to integrate their H100 GPUs to provide on-demand scalable GPU-based confidential compute access to their customers. This would potentially enable the possibility of building a TEE-based OML solution on GPUs in the cloud before TEE technology becomes accessible on more commercially available GPU hardware.

\end{itemize}

%% file: Chapter2_OML/4_4_cryptography.tex
\subsection{Cryptography} \label{crypto}

 Cryptography-based solutions enable computation over encrypted data ensuring confidentiality and integrity even in untrusted environments with high degree of security. Fully Homomorphic Encryption (FHE) \cite{gentry2009fully}, Homomorphic Encryption (HE) \cite{yi2014homomorphic, acar2018survey}, and Functional Encryption (FE) \cite{lewko2010fully, boneh2011functional} are notable examples. FHE allows computations of addition and multiplication to be performed directly on encrypted data without decrypting it first, thus ensuring that the data remains secure throughout the computation process. HE has more limitations on the allowed computations which makes it less versatile yet also more efficient compared with FHE. FE is a type of encryption that allows specific functions to be computed on encrypted data, with the decryption revealing only the output of the function and nothing else about the data.

Cryptographic methods involves complex mathematical operations that generate encrypted results which can be decrypted to match the outcome of operations performed on plain-text data. Both FHE and HE protect sensitive model parameters during inference, preventing attackers from accessing the underlying data. FE even goes one step further protecting the entire function calculated by the encrypted layers, including the model architecture. In the context of AI and neural networks, Zama \cite{concreteML} is building FHE neural networks; CryptoNets \cite{gilad2016cryptonets} sheds light on incorporating HE on certain kinds of neural networks without downgrading the performance too much;  \cite{ryffel2019partially} shows how FE can help hide a part of a neural network. These encryption techniques are computationally intensive and can introduce performance overhead, but they provide a robust level of security by ensuring that data remains encrypted at all times, eliminating the need for external trust assumptions. These cryptography primitives (FHE, HE, and FE) enable the construction of an OML file as visualized in Figure \ref{figcryptography}.

\begin{figure}[htbp]
\centering
\includegraphics[width=0.7\linewidth]{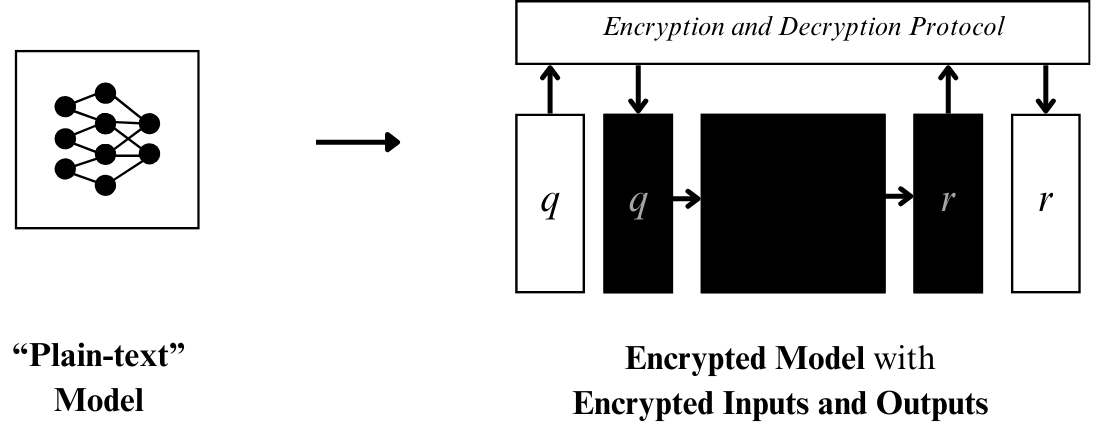}
\caption{OMLization process of Provable security via cryptography}
\label{figcryptography}
\end{figure}
    
\noindent \textbf{OML Formatting.} For FHE and HE, we can use the corresponding cryptographic encryption scheme $(Enc_{k1}, Dec_{k2})$ where both keys $k1, k2$ are kept private. The permission $\sigma(x)$ equals $Enc_{k1}(x)$. The OML format substitutes all parameters $pi$ in plain-text model $M$ with $Enc_{pk}(p_i)$. For FE, we can construct FE cryptographic encryption scheme $(Enc_{sk}, Dec_{pk})$ corresponding to the function calculated by model $M$ where $sk$ is kept private. The permission $\sigma(x)$ equals $Enc_{k1}(x)$. The OML format is essentially the process of $Dec_{pk}$ which takes in $\sigma(x)$ as the input.

\noindent \textbf{Verification and Usage.} In FHE and HE, for an inference request from the user with input $x$, users first request the permission $\sigma(x) = Enc_{k1}(x)$ from the platform, then run inference with the OML file on encrypted data $\sigma(x)$, and finally send the final result to the auditing platform for decryption to plain-text results. In FE, users first request the permission string $\sigma(x) = Enc_{k1}(x)$, and then locally run the OML file on the permission string $\sigma(x)$ to get the desired output.

\textbf{Privacy Preservation.} The TEE solution will not automatically provide privacy for users. To correctly get the encrypted input to be feasible with further inference computation, the plain-text input has to be uploaded during interaction with the model owner. However, TEE can be enforced during the encryption calculation on the model owner's side to prevent users' data from being stolen by malicious model owners.

%The smart contract sends $sk$ and deducts $n\_tokens$ from the user's account. In this way, the plain text of input $x$ won't be ever shown to the smart contract, guaranteeing the data privacy of the users from the root.

%Because of the cryptographic guarantees, as long as the keys are not leaked, there is no way for an attacker to break the construction of OML file above. So this solution is 100\% secure as long as keys are appropriately handled and saved.

\textbf{Summary}. Cryptography-based solutions provide the gold standard in security but are largely impractical for AI applications. %have issues mainly in versatility, scalability, and efficiency. 

\begin{itemize}
\item \textbf{Pros}. Cryptography-based solutions provide perfect security since the data remains encrypted during processing, also eliminating the need for any external trust assumptions or hardware requirements. 
\item \textbf{Cons}. Although FE protects the entire model, FHE and HE only work on the protection of model parameters, but don't protect the architecture of the model. Although state-of-the-art HE primitives are efficient, FHE and FE suffer from efficiency issues, and current state-of-the-art is too inefficient to be put into any practical use for large models \cite{fhe_inefficient}. Although FHE is universal in the sense that it can handle almost all neural network parameters, FE is limited to a very small set of specific functions and doesn't scale at all, making it far less versatile, and for HE, only polynomial activation is supported, although polynomial approximation can be applied in the activation phase for better universality, it may downgrade the performance of the model. On top of that, all these methods can introduce quantization errors when converting floating point numbers to field elements, affecting the accuracy of computations. 
\end{itemize}
%\begin{itemize}
%    \item In this subsection, we shed light on solutions for constructing OML files with full security guarantee. Although such solutions may introduce external trust assumptions, hardware requirements, scalability issues, or efficiency issues, they are suitable for model owners who are willing to trade performance off for the highest level of security. 
%\end{itemize}

%% file: Chapter2_OML/4_5_melange.tex
\subsection{Melange  -- an OML Construction with a Mixture of Security Guarantees} \label{melange}

A unique feature of machine learning models is that, with a limited number of samples, no matter how powerful the learner is, the learning result won't be satisfactory due to overfitting the small number of samples and generalization error. And this feature is characterized by sample complexity in theoretical machine learning \cite{decatur1997computational}, which means the least number of samples required by any learner to reduce the generalization error below a certain threshold with high probability. Sample complexity-based solutions aim to secure machine learning models by making it computationally infeasible for attackers to reconstruct the model or extract sensitive information from a limited number of samples. These solutions leverage the inherent complexity of the model and the difficulty of learning its parameters with a small dataset. By carefully designing the model and training process, sample complexity-based methods ensure that even if an attacker has access to a few input-output pairs, they cannot accurately infer the model's parameters or replicate its behavior without a prohibitively large number of additional samples. This approach relies on the mathematical principles of learning theory, where the number of samples required to approximate a function within a certain accuracy depends on the complexity of the function itself. Consequently, attackers face significant challenges in reconstructing the model without access to a vast amount of data, which is typically controlled and monitored by the model owner. Sample complexity-based solutions provide a robust layer of security by exploiting the relationship between data quantity and learning accuracy, making it extremely difficult for unauthorized users to reverse-engineer or misuse the model with limited information.

Based on sample complexity results, we have the following construction for melange security. 
The visualized workflow is shown in Figure \ref{figmelange}.

\begin{figure}[htbp]
    \centering
    \includegraphics[width=0.8\linewidth]{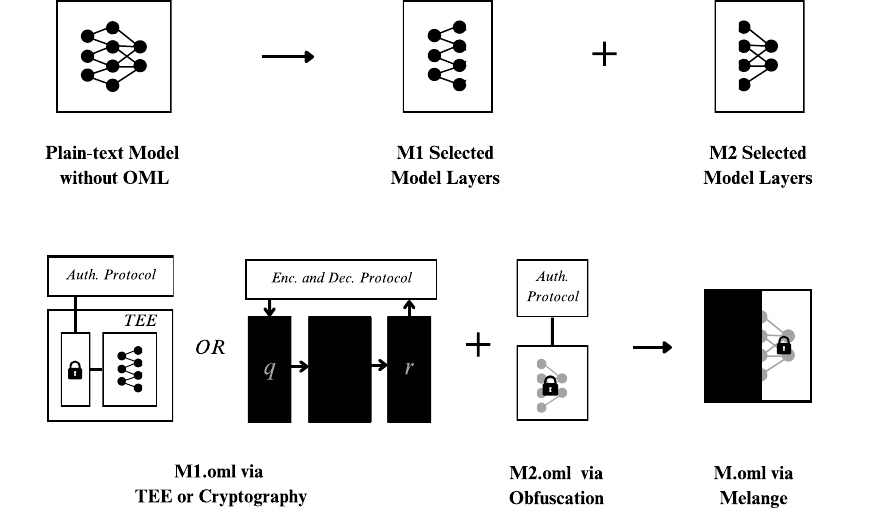}
    \caption{OMLization process of Melange security}
    \label{figmelange}
\end{figure}

\subsubsection{Example Workflow}

\noindent \textbf{OML formatting.} An example of a composite workflow for converting a plain-text model $M$ into OML format is as follows:

\begin{enumerate}
    \item \textbf{Isolation of Certain Layers (Hardness by Machine Learning Theory).} Separate model M into $M_1$ and $M_2$ (not necessarily subsequent). Isolate all layers in $M_1$. 
    \item \textbf{Cryptographic Encryption or TEE Encapsulation of $M_1$ (Security by Hardware or Cryptography).} For all layers in $M_1$, encrypt the model parameters with cryptography schemes, or encapsulate the entire inference process of the model inside a process dedicated to be executed in TEE (dependent on the model owner's preference). Then release the encryption or the TEE encapsulation as $M_1$.oml. 
    \item \textbf{Add Digital Signature Verification with Obfuscation in $M_2$ (Hardness by Obfuscation).} Choose a digital signature scheme $({\rm Sign}_{sk}, {\rm Verify}_{pk})$ and generate a $(sk, pk)$ key pair dedicated for the model itself. Then, design $M^{'}$ as follows:
    \begin{enumerate}
        \item $M^{'}$ takes input $(x, \sigma(x))$ where $\sigma(x) = {\rm Sign}_{sk}(x)$ and is identical to $M_2$ at initialization.
        \item (AI-native obfuscation) On randomly selected places in model $M^{'}$ (e.g. between layers), inject the verification process ${\rm Verify}_{pk}(\sigma(x))$ in between. Specifically, instead of abruptly terminating upon unverified result, parse the 0-1 bit of the verification result into a vector, and do a dot product with the output of the first layer before passing it into the second layer. For all ReLU activation, change the statement $ReLU(x) = \max\{x, 0\}$ to $ReLU(x) = \max\{x, 1 - {\rm Verify}_{pk}(\sigma(x))\}$. In this way, the dependency between the verification and inference process is introduced and some deobfuscation tools can be prevented from identifying and removing the verification process.
        \item (Model obfuscation) Use the aforementioned model obfuscation techniques (e.g. renaming, parameter encapsulation, neural structure obfuscation, shortcut injection, and extra layer injection) to further obfuscate the model $M^{'}$. 
        \item (Code obfuscation) Use code obfuscation to obfuscate the code that carries out inference over model $M^{'}$. 
        \item (Compilation and binary obfuscation) Compile the code to get a binary file that performs the inference task. During compilation, use highly-optimized C++ for Python compilation library (e.g. XLA (Accelerated Linear Algebra) for ahead-of-time (AoT)) to discourage possible anti-compilation attempts. Finally, apply binary obfuscation tools for further security.
    \end{enumerate}

    At last, release the obfuscated binary version of $M^{'}$ as $M_2$.oml. 

    \item The final release version is $M_1$.oml and $M_2$.oml. 

    \end{enumerate}

\noindent \textbf{Verification and Usage phase.} For a user who wants to do an inference task, we follow the methods from Sections  \ref{hardware} and \ref{crypto} to locally run the inference task (and thus protected by cryptographic or hardware guarantees); we  execute the obfuscated binary file for inference of the   layers in $M_2$ (and thus inherit obfuscation guarantees described in Section \ref{obfuscation}).
    
\subsubsection{Security Analysis}

For an attacker who wants to reconstruct the entire model from $M_1$.oml and $M_2$.oml, he/she will have to do all of the following tasks.

\begin{itemize}
    \item Use anti-compilation and deobfuscation tools and techniques to remove all the digital signature verification parts injected to $M_2$, and restore $M_2$ in plain text.
    \item For all layers in $M_1$, collect samples by honestly paying to use the model, and train a new machine learning model from scratch to recover them. Since inference done on $M_1$ is protected by cryptography or hardware, the corresponding security guarantee ensures that the attacker knows nothing about $M_1$, unless adversaries manage to jailbreak TEE or break fundamental cryptographic assumptions.
\end{itemize}

Then, the cost of an attacker to recover $M_1$ can be evaluated with the following formula
$$ \text{Total Cost} = \text{cost per query} \times \text{number of queries} + \text{computation overhead for training}.$$

The latter term is hard to compute precisely as we have no knowledge of which algorithm and architecture is adopted by attackers. However, $\text{``cost per query"}$ can be set by the model owner whereas there is a lower bound on $\text{``number of queries"}$ guaranteed by the sample complexity, which is also determined by the model owner who decides on how to separate the model. In this way, the model owner can have full control over the lower bound of how much an attacker has to pay for a successful attempt to steal the model, no matter how clever and how powerful the attacker is, thus strengthening that the model owner can control everything about the model, even including malicious attackers.

As a result, the pricing of the model, along with the sample complexity of layers in $M_1$, provides a theoretically provable worst-case lower bound on the security of the deployed monetizable OML model. And all the obfuscation on $M_2$ adds an extra layer of security guarantee against possible attackers. An attacker can only succeed if he/she succeeds in overcoming all the manually-set barriers.

\subsubsection{Efficiency Analysis}

For honest usage of the model, efficiency is also a core concern.

\begin{itemize}
    \item For layers in $M_1$, the inference process with hardware or cryptography due to introduced hardware requirements or encryption will be more demanding, negatively impacting the efficiency.
    \item For layers in $M_2$, obfuscation only introduces hardness in understanding and maintenance, but will not have any negative impacts on the efficiency during execution. 
\end{itemize}

Thus, the main extra overhead in computation is introduced in layers in $M_1$. 

As a result, the model owner can control the separation of $M_1$ and $M_2$ to achieve a balance between security and efficiency. Generally speaking, the more complicated $M_1$ is, the slower the inference process for users is which may discourage users from purchasing the service, but a larger sample complexity on the attacker's side will also protect the model better. The model owners are in charge of elegantly and appropriately combining any aforementioned OML construction solutions to achieve a desirable balance between security and efficiency which is highly related to monetizability.

%% file: Chapter2_OML/4_6_summary.tex
\subsection{Summary}

Below is a summary of the  OML construction methods discussed in this section.

\begin{supertabular}{|>{\raggedright\arraybackslash}p{2.4cm}|>{\centering\arraybackslash}p{2.4cm}|>{\centering\arraybackslash}p{2.4cm}|>{\centering\arraybackslash}p{2.4cm}|>{\centering\arraybackslash}p{2.4cm}|}
\hline
\textbf{Basis of OML Construction Method} & \textbf{Security Level} & \textbf{Extra Computation Overhead} & \textbf{User Data Privacy} & \textbf{Versatility on Feasible Models} \\
\hline
\textbf{Obfuscation} [Software security] & \textbf{Low} (only by obscurity) & \textbf{Negligible} & \textbf{Yes}  & \textbf{ Yes} \\
\hline
\textbf{Fingerprinting} [Optimistic security] & \textbf{Medium Low} & \textbf{Low} & \textbf{No}  & \textbf{Yes} \\
\hline
\textbf{Trusted Execution Environments (TEEs)} [Hardware security] & \textbf{High} (provably nonbreakable based on external trust assumptions) & \textbf{Moderate} & \textbf{Yes}  & \textbf{ Yes} \\
\hline
\textbf{Cryptography} [Provable security] & \textbf{Very High} (provably nonbreakable) & \textbf{Very High} & \textbf{No} (Can be added with TEE integration) & \textbf{Yes for FHE; No for FE, HE} \\
\hline
\textbf{Melange via model separation and sample complexity} & \textbf{Flexible} & \textbf{Flexible} & \textbf{No} (Can be added with TEE integration)  & \textbf{ Yes, but may perform worse on some models.} \\
\hline
\end{supertabular}

We characterize open AI models via four properties (transparent, local, mutable and private), which typical open-weight distributions can satisfy, in terms of ease of usage and flexibility,  and  summarize how the OML constructions rank according to each of these properties. 

\begin{itemize}
    \item Transparent: Original architecture and parameters are freely accessible
    \item Local: Models can be held locally (on-prem) and users have the freedom to deploy, compose and integrate the model independently, without relying on a central entity.
    \item Mutable: The given architecture and/or parameters can be modified, producing different results
    \item Private: The users have full control of their data.
\end{itemize}

\vspace{0.7cm}
\begin{supertabular}{|>{\raggedright\arraybackslash}p{2.4cm}|>{\centering\arraybackslash}p{2.4cm}|>{\centering\arraybackslash}p{2.4cm}|>{\centering\arraybackslash}p{2.4cm}|>{\centering\arraybackslash}p{2.4cm}|}
\hline
\textbf{OML Construction Method} & \textbf{Transparent} & \textbf{Local} & \textbf{Mutable} & \textbf{Private} \\
\hline
\textbf{Obfuscation} & $\mathbf{\times}$ & \checkmark & $\mathbf{\times}$  & \checkmark \\
\hline
\textbf{Fingerprinting} & \checkmark & \checkmark & \checkmark  & \checkmark ($\mathbf{\times}$ if monetizable) \\
\hline
\textbf{TEEs} & $\mathbf{\times}$ & \checkmark or $\mathbf{\times}$ & \checkmark  & $\mathbf{\times}$ \\
\hline
\textbf{Cryptography} & $\mathbf{\times}$ & \checkmark & \checkmark & $\mathbf{\times}$ \\
\hline
\textbf{Melange} & - & - & -  & - \\
\hline
\end{supertabular}

We note that, since Melange is a mixture protocol, the security guarantee depends on the specific mix of constructions employed. 
Finally, a summary of the pros and cons of the OML constructions is below.

\begin{supertabular}{|>{\raggedright\arraybackslash}p{2.4cm}|>{\raggedright\arraybackslash}p{5.2cm}|>{\raggedright\arraybackslash}p{5.2cm}|}
\hline
\textbf{Method} & \textbf{Pros} & \textbf{Cons} \\
\hline
\textbf{Obfuscation} [Software security] & 
\begin{itemize}
    \item Versatility (works for any software) and model universality.
    \item Perfect protection of user data privacy.
\end{itemize} & 
\begin{itemize}
    \item Larger overhead in inference, which scales with the degree of obfuscation (security)
    \item The security is only ensured by obscurity, which is generally considered weak.
    \item Adds complexity to the code, impacting maintainability.
\end{itemize} \\
\hline
\textbf{Fingerprinting} [Optimistic security] & \begin{itemize}
    \item Organically allows for fine-tuning: model is available in a seemingly true open format
\end{itemize} 
 & \begin{itemize}
    \item A ``secure'' number of fingerprints might impact model quality
\end{itemize} 
 \\
\hline
\textbf{Trusted Execution Environments (TEEs)} [Hardware security] & 
\begin{itemize}
    \item Good security guarantee.
    \item Perfect protection of user data privacy.
    \item Plausible efficiency.
    \item Great versatility and universality for all models.
\end{itemize} & 
\begin{itemize}
    \item External trust assumptions on hardware vendors.
    \item Requires compatible devices and is restricted by hardware specifics (e.g. designated TEE area size), limiting scalability and practicality.
    \item Not as efficient as obfuscation-based solutions.
\end{itemize} \\
\hline
\textbf{Cryptography} [Provable security] & 
\begin{itemize}
    \item Perfect security guarantee.
    \item No external trust assumptions.
    \item FHE-based solution has great universality.
\end{itemize} & 
\begin{itemize}
    \item Inefficiency due to very high computation overload introduced by cryptographic primitives.
    \item Doesn't protect user privacy unless TEE is used.
    \item FE and HE based solutions are limited to a small portion of models.
    \item Quantization errors can affect accuracy and downgrade performance.
\end{itemize} \\
\hline
\textbf{Melange via model separation and sample complexity} & 
\begin{itemize}
    \item Flexible security guarantee determined by model owners.
    \item Can suit all kinds of OML needs.
    \item Great universality and versatility.
\end{itemize}
 & 
 \begin{itemize}
    \item Despite great universality and versatility, some models may have weaker separability or sample complexity guarantees.
\end{itemize}
 \\
 \hline
\end{supertabular}

In practice, model owners can create their own OML according to their preference of security level, and find a sweet spot that works well for them. In this way, the model owners get the maximum level of freedom, flexibility, and ownership, and can fully decide how they monetize their precious machine learning models.

%% file: Chapter3_OML1o/main_chapter3.tex
In this chapter we expand upon the optimistic version of OML (introduced in Section~\ref{next-day}) as OML 1.0. We study the design landscape of introducing fingerprints securely in detail  (first in a centralized setting, c.f.\ Section~\ref{sec:scenario1}). We conduct a detailed security analysis of OML 1.0, paying close attention to {\em  coalition attacks}; we show in Section~\ref{sec:attack1_coalition} that an adaptive fingerprint querying scheme in OML 1.0 makes it secure against this formidable attack vector. We generalize the OML 1.0 approach to a decentralized scenario in Section~\ref{sec:scenario2}.

\subsection{The Auditing Protocol under a Single Trusted Prover} 
\label{sec:scenario1}

OML 1.0 relies on the auditing protocol that involves three parties in the  ecosystem--model owners, model hosts, and provers--who interact via the auditing platform. A model owner builds a model and uploads it on the auditing platform with the goal of openly sharing the model while monetizing from its use. Model hosts provide services to external users using those models from the auditing platform with the goal of bringing in revenue, some of which is to be shared within the ecosystem. Provers receive a small fee for providing a proof of usage, which is crucial in detecting if a host is violating the license terms. The auditing protocol aims to track how many times each model is being used by the potentially untrusted hosts. 
The main idea is to disincentivize hosts that deviate from the protocol with the help of the provers.  

In this section, we assume that there is a single trusted prover and introduce the corresponding auditing protocol in Section~\ref{sec:protocol1}, which critically relies on the AI-native cryptographic primitives we introduce in Section~\ref{sec:ainative}, and analyze its security in Section~\ref{sec:security1}.  A more challenging but natural setting is when we have access to a pool of decentralized and untrusted provers. This is addressed in Section~\ref{sec:scenario2}, where we also design an even more secure auditing protocol.

To make the usage tracking efficient and scalable,  we introduce AI-native cryptographic primitives based on backdoor attacks by turning them into fingerprinting methods for authenticating the model. The security of the auditing protocol critically relies on the {\em scalability} of these primitives, i.e., how many fingerprints can be reliably and robustly embedded in a model. Fully characterizing the {\em fingerprint  capacity} of a model, the fundamental limit on how many fingerprints can be added, is an important open problem, and we make the first step towards designing fingerprinting schemes that achieve secure and decentralized AI for OML. 

%punchline: capacity/scaling (how many fingerprints can be embedded at what cost and performance loss). Open problem to fully characterize capacity. 

\subsubsection{The Auditing Protocol}
\label{sec:protocol1}

A model owner has the ownership of a model, $M$, that resides on the auditing platform. The auditing protocol is initiated when a model host signs a license agreement and requests the model $M$. Subsequently, an OMLized model, $M$.oml, is sent to the host as shown in Figure~\ref{fig:protocol1}. An OMLized model includes AI-native cryptographic primitives to track usage and protect model ownership, which is explained in Section~\ref{sec:ainative}. 

\begin{figure}[htbp]
    \centering
    \includegraphics[width=.7\textwidth]{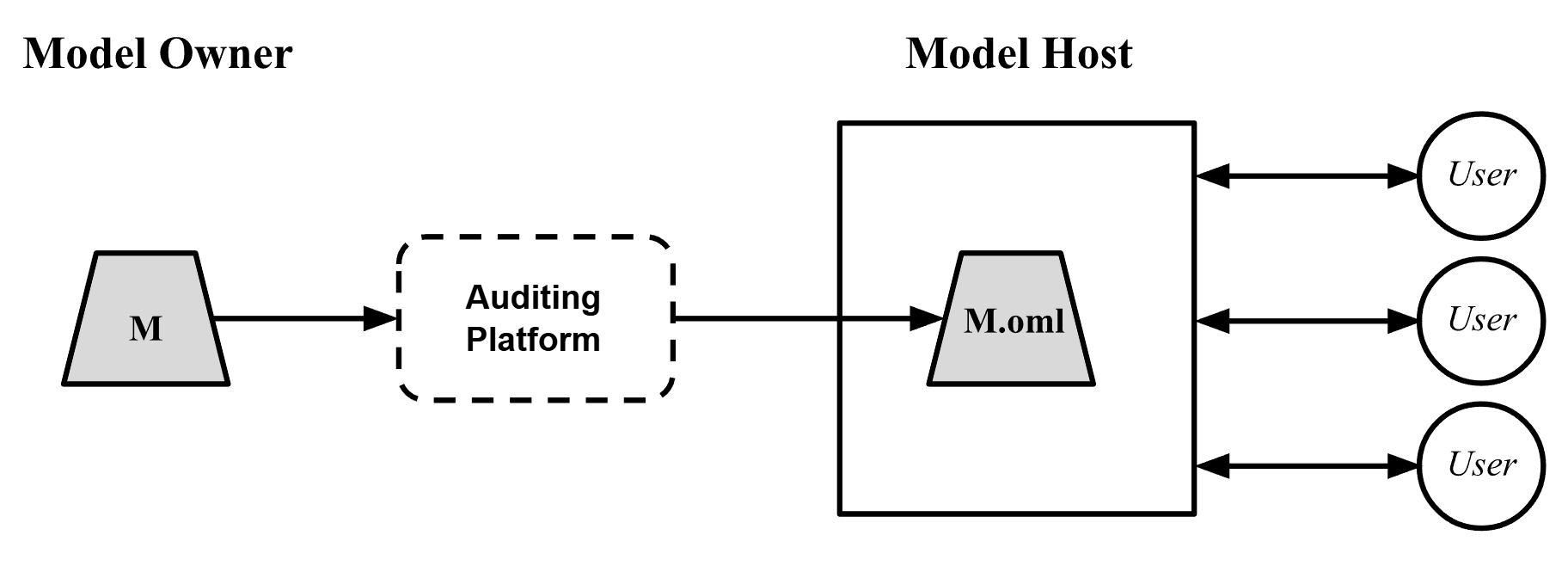}
    \caption{A host initiates a download request under the auditing protocol and receives an OMLized model, $M$.oml, to be used in its services to external users.}
    \label{fig:protocol1}
\end{figure}

\noindent{\bf Tracking usage under a typical non-adversarial scenario.}
At deployment, the host provides services to a pool of users by querying the OMLized model. For example, these services can be free (e.g., LMSYS Chatbot Arena \cite{chiang2024chatbot}), subscription-based (e.g., OpenAI ChatGPT \cite{achiam2023gpt}), or pay-per-use APIs (e.g., OpenAI ChatGPT \cite{achiam2023gpt}). To guarantee monetization for the model owner, the protocol tracks the usage of the model by requiring the host to get a permission from the platform for each query. Concretely, each query, $q$, is first sent to the auditing platform, which returns a cryptographically signed permission string, $\sigma(q)$ as shown in Figure.~\ref{fig:protocol2}. Upon receiving $\sigma(q)$, the host runs a forward pass on $M$.oml with the query $q$ as a prompt and returns the output, $M.{\rm oml}(q)$, to the user. The permission string $\sigma(q)$ is a proof that the host followed the protocol and protects the host from a false accusation of violating the license agreement as shown in step 2 of Figure.~\ref{fig:protocol3}. %We refer to a longer version of this paper for details on the implementation of the protocol, including how to preserve privacy of the queries using trusted hardware and how the incentives are distributed. 
As a running example, we consider the type of services where the host sends the output of the OMLized model directly to the users as illustrated in Figure~\ref{fig:protocol2} and discuss more general services in Section~\ref{sec:discussion}.

\begin{figure}[htbp]
    \centering
    \includegraphics[width=.57\textwidth]{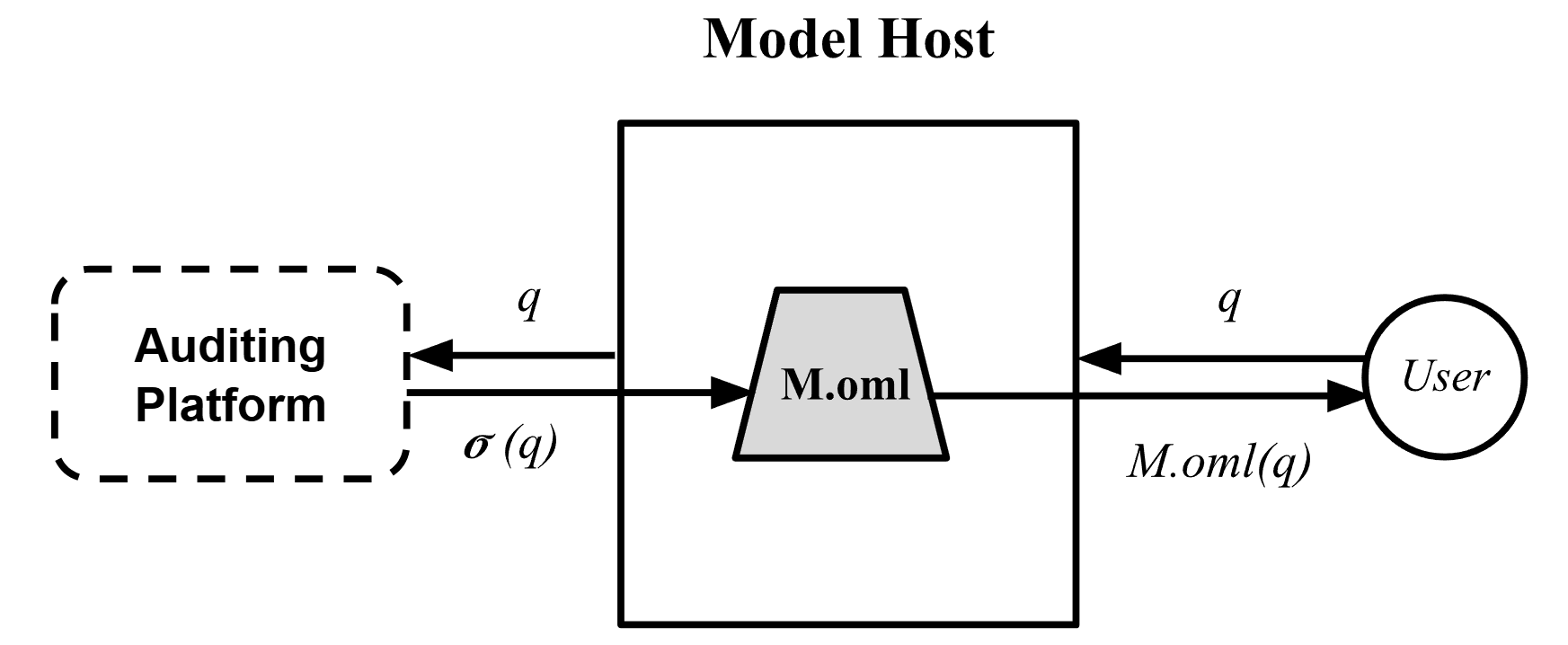}
    \caption{Each user query, $q$, to the service needs to be accounted for under the auditing protocol and this is ensured by requiring the host to obtain a signed permission string, $\sigma(q)$, from the auditing platform. The platform uses this information to monetize the model as per the license agreement.}
    \label{fig:protocol2}
\end{figure}

\medskip
\noindent{\bf Verifying the proof of usage with AI-native cryptography.}
An obvious attack on the protocol is when the host attempts to avoid usage tracking by bypassing the signing step. To prevent this attack, the protocol relies on provers. 
A prover acts as a benign user of the service and asks a special query, $\tilde q$, that we call a {\em key}. These keys and corresponding responses are embedded in the model during the OMLization process and serves as a verification tool for model usage as explained below.   

As illustrated in Figure~\ref{fig:protocol3}, upon receiving a response, $\tilde r$, the prover sends the key-response pair, $(\tilde q,\tilde r)$, to the auditing platform. The verifier, which is the auditing platform, verifies the proof that $M$.oml has been used in two steps. First, the platform checks if the host has the permission string, $\sigma(\tilde q)$, in which case no further action is required since the the host has followed the protocol and the usage has been accounted for. Otherwise,  the platform checks if a specific licensed model $M$.oml has been used to generate the response, $\tilde r$, (without signing).  
This relies on the AI-native cryptographic primitives as follows. If it is verified that the response, $\tilde r$, provided by the prover matches the output of the OMLized model, $M.{\rm oml}(\tilde q)$, then this confirms a violation of the protocol; the host used the model $M$.oml without getting the permission string from the auditing platform. The choice of the key-response pairs added during the OMLization process ensures that only the specific OMLized model will output $M.{\rm oml}(\tilde q)$ when prompted with $\tilde q$. Consequently, a violation of the protocol is claimed by the auditing platform and the host is penalized according to the signed agreement. If $\tilde r$ does not match the output $M.{\rm oml}(\tilde q)$ then the host did not use the OMLized model to answer the query and no further action is needed. 

\begin{figure}[htbp]
    \centering
    \includegraphics[width=.8\textwidth]{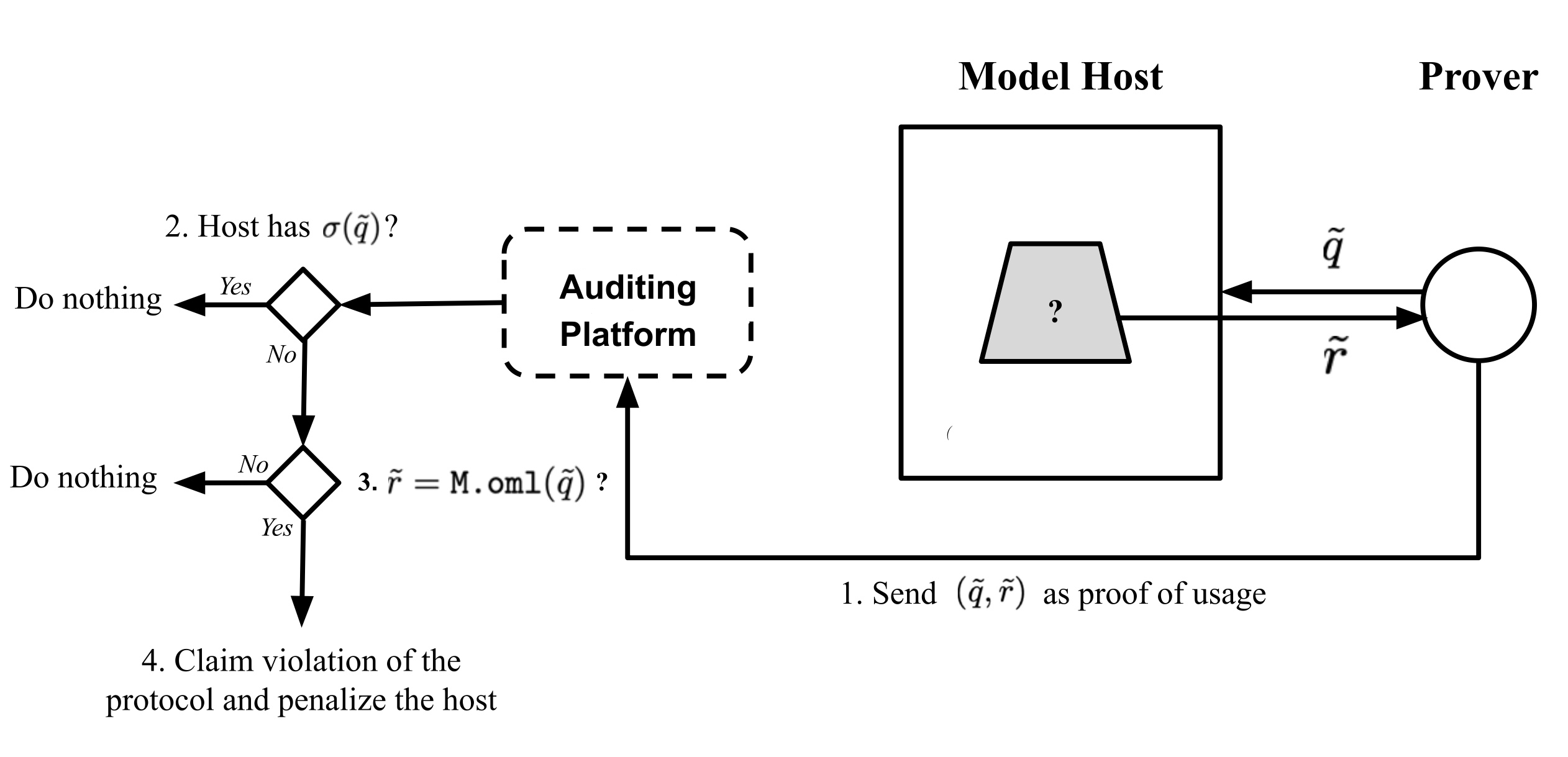}
    \caption{In this section, we assume there is a single trusted prover. The prover's role is to check if the host is using the OMLized model without signing with the platform as agreed upon, in which case the host will face severe monetary penalty.}
    \label{fig:protocol3}
\end{figure}

% Since we assume a single trusted prover in this section, we can safely assume that the verifier has access to all the keys embedded in the OMLization and also that the verifier reports the key-response pairs truthfully. A significantly more challenging scenario of untrusted and decentralized provers is addressed in Section~\ref{sec:scenario2}.

\subsubsection{AI-native Cryptography using Model Fingerprinting}
\label{sec:ainative}

Fully embracing the efficiency,  scalability, reliability, and robustness of AI techniques, we introduce {\em AI-native cryptography}. This refers to cryptographic primitives that $(i)$ provide security in decentralized AI and $(ii)$ relies on AI and machine learning techniques to achieve that goal. Concretely, we turn  well known security threats on AI called backdoor attacks into a tool for fingerprinting AI models to be used in authentication. Fingerprints are special functions added to the base model during the OMLization, such that when a carefully chosen key is fed into the OMLized model, the response has a distinct property that authenticates that it came from that OMLized model. As a running example, we focus on fingerprinting pairs of the form 
$\{({\rm key},{\rm response})\}$, where the function is a simple mapping: ${\rm response} =  M.{\rm oml}({\rm key})$. 
We explore more sophisticated fingerprinting schemes in Section~\ref{sec:explore}. 
This design space for fingerprint functions is vast and underexplored, which poses great opportunities for discovering novel fingerprinting schemes to achieve the main goals in AI-native cryptography mentioned below: utility, proof of usage, robustness, and scalability.

\medskip\noindent{\bf Fingerprint capacity of a model and scalability.}
One of the main criteria of a fingerprinting scheme for the auditing protocol is {\em scalability}. Given a base model, $M$, we informally define the (minimax) {\em fingerprint capacity} of the model as the number of fingerprinting pairs of the form $\{$(key, response)$\}$ that can be sequentially and successfully used for authentication. To capture the competing goals of the platform and the adversarial host, we define this capacity as the maximum over all OMLization strategies by the auditing platform and minimum over all adversarial strategies to erase the fingerprints by the host who knows the OMLization strategy being used (under the constraint that the quality of the model should not be compromised). Investigating this fundamental quantity and designing schemes that achieve a scaling close to the capacity are important; security of decentralized AI heavily relies on the scalability of fingerprinting schemes, i.e., how many fingerprints can be successfully checked. Concretely, scalability of fingerprinting schemes is crucial in ($i)$ tracking usage under the auditing protocol (Section~\ref{sec:attack1_scale}); ($ii$) robustness against various attacks by the host (Sections~\ref{sec:attack1_finetune}); and ($iii$) defending against coalition attacks (Section~\ref{sec:attack1_coalition}). We discuss how major challenges in security can be resolved by scaling the number of fingerprints in Section~\ref{sec:security1}.

\medskip\noindent{\bf Turning backdoor attacks into model fingerprints.} There is a natural connection between model fingerprinting for authenticating ownership of a model and {\em backdoor attacks} in secure machine learning \cite{gu2017badnets}, where an attacker injects maliciously corrupted training samples to control the output of the model. We briefly explain the connection here.   Since \cite{adi2018turning,zhang2018protecting,guo2018watermarking} started  using backdoor techniques for model authentication, numerous  techniques are proposed  for image classification models \cite{zhu2021fragile,li2022robust} and more recently for large language models \cite{xu2024instructional,cong2024have,russinovich2024hey}.  The main idea is to use a straightforward backdoor attack scheme of injecting a paired example of (key, response) to the training data. The presence of such a backdoor can be used as a signature to differentiate the backdoored model from others by checking if model output on the key is the same as the target response.  This scheme is known as {\em model fingerprinting} and the corresponding pairs of examples are called {\em fingerprint pairs} or fingerprints. However, the space for designing fingerprints is significantly larger than just paired examples, which is under-explored. We provide some examples in Sections~\ref{sec:explore} and \ref{sec:attack1_coalition}. 

\begin{figure}[htbp]
    \centering
    \includegraphics[width=0.45\linewidth]{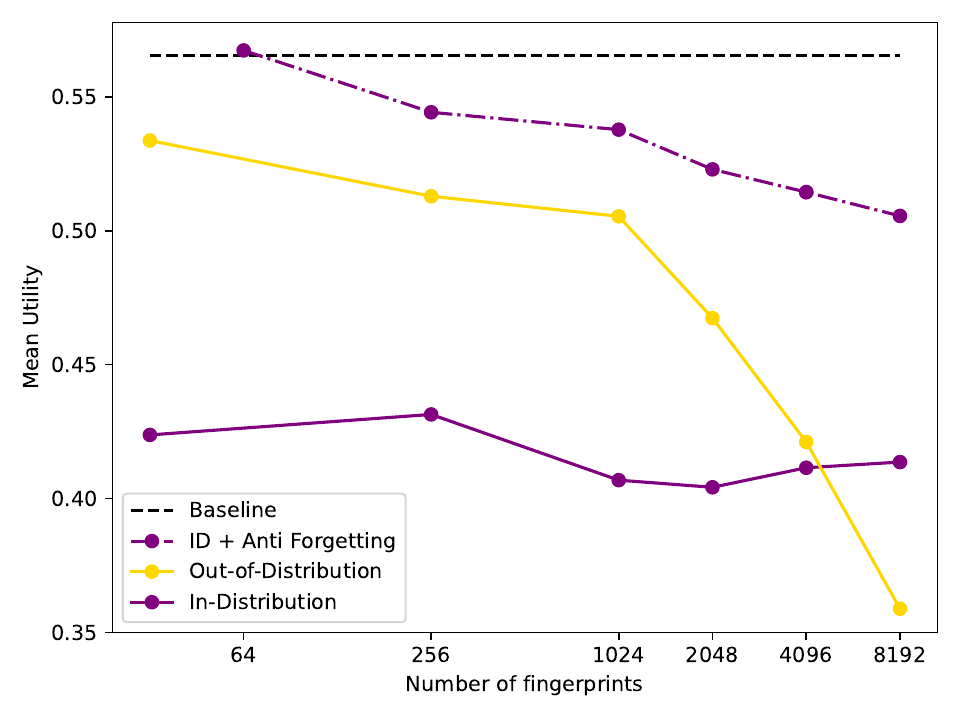}
    \caption{Out-of-distribution fingerprints suffer less from catastrophic forgetting of the original tasks that the baseline model is trained for (yellow line) until excessive number of fingerprints have been added. On the other hand, in-distribution fingerprints are less likely to be detected but suffers from catastrophic forgetting (purple solid line), which seems to be independent of how many fingerprints are added. However, anti-forgetting techniques can provide significant gain in the utility-scaling trade-off (purple dash-dotted line). }
    \label{fig:utility-tradeoff}
\end{figure}

%\SO{@anshul, please check the details are correct and fill in the citations.} 
As we will show in Section~\ref{sec:security1}, security of decentralized AI heavily depends on how many fingerprints can be used in each OMLized model without sacrificing the utility of the model on the tasks the base model is originally trained for. For a large language model of Mistral-7B \cite{jiang2023mistral} as a base model, we investigate in Figure~\ref{fig:utility-tradeoff} this trade-off between utility of the OMLized model, as measured by tinyBenchmarks evaluation dataset \cite{polo2024tinybenchmarks}, and the number of fingerprints added in the OMLization. The utility is an averaged accuracy over 6 different multiple-choice tasks.

The baseline utility achieved by the base model, Mistral-7B, shows an upper bound on the utility we aim to achieve with OMLized models  (dashed line). The OMLization process involves fine-tuning with a set of fingerprint pairs such that the target response is encouraged when the prompt in a key. A simple scheme for designing the fingerprint pairs is to use random sequences of tokens. Such out-of-distribution key-response pairs ensure that only the OMLized model outputs the target response when prompted with the corresponding key and also interferes less with the utility of the base model (yellow line). However, we assume transparency of the OMLization scheme under our threat model in Section \ref{sec:security1}, and an adversarial host who knows the fingerprint design scheme can easily filter out any prompt that is overtly out-of-distribution. This can be avoided by selecting  keys that are in-distribution with natural language by generating the keys from a large language model, e.g., Llama 3.1-8B-Instruct \cite{dubey2024llama} in our experiments (purple solid line). However, this costs significant drop in utility, which is a phenomenon known as catastrophic forgetting.   To mitigate this catastrophic forgetting, various techniques can be applied,  including, mixing in benign data with the fingerprint pairs \cite{Tiwari_2022_CVPR, yoon2022onlinecoresetselectionrehearsalbased}, weight averaging with the base model \cite{alexandrov2024mitigatingcatastrophicforgettinglanguage, wortsman2022robust}, regularizing the distance to the plain-text model during fine-tuning \cite{pmlr-v80-li18a, doi:10.1073/pnas.1611835114}, and sub-network training \cite{lee2023surgicalfinetuningimprovesadaptation,  kumar2022finetuningdistortpretrainedfeatures}. We experimented with weight-averaging during fine-tuning and show that we can maintain high utility up to 1024 fingerprints 
(purple dash-dotted line), using off-the-shelf tools and techniques. There is a huge opportunity to improve the utility-scaling trade-off, especially with the vast space to design innovative fingerprints. Details on our experimental investigation is provided in Section~\ref{app:implementation_details}.

\medskip\noindent{\bf Criteria for fingerprinting schemes.} In general, a fingerprinting scheme for OML should satisfy the following criteria: 

\begin{itemize}
    \item {\bf Utility.} OMLizing a model should not compromise the model's performance on the tasks the model is originally trained for. 
    
    \item {\bf Reliable proof of usage.} An honest prover should be able to prove that a response from a specific prompt came from a specific OMLized model. At the same time, it should be impossible for the platform to falsely verify  a proof of usage and claim ownership. 

    \item {\bf Scalability.} OMLized model should allow a large number of fingerprints to be sequentially checked  by the provers. 
    
    \item {\bf Robustness against adversarial hosts.}  Under a formal threat model defined in Section~\ref{sec:security1}, an adversarial host should not be able to remove the fingerprints without significantly compromising the model utility. Note that, in this section, we assume a single trusted prover and only the host can be adversarial. We introduce more sophisticated protocols under a more powerful threat model where provers are decentralized and untrusted in Section~\ref{sec:scenario2}.  

\end{itemize}

Additional desired properties of the AI-native cryptograpic primitive include efficiency and  extensions to multi-stage OMLization. Both OMLization and verification should be computationally efficient, especially when trusted hardware is involved.  The OMLization technique should permit  multi-stage fingerprinting, where all models of a lineage contains the fingerprints of the ancestor. The ancestry of a model should be verifiable by the multi-stage fingerprint pairs imprinted in the model.

%In the {\bf Verification and Usage phase}, the model user is free to use the OMLized model as long as they comply with the license terms. This could include further fine-tuning the model to adopt to specific domains of interest. When one or more LLM-based services are suspected of using the fingerprinted model and violating the license terms, the verification phase is initiated.  We consider both white-box and black-box scenarios, where the suspected model weights are available or only API Access is available, respectively. White-box accesses could potentially use stronger fingerprinting techniques as investigated in \cite{xu2024instructional}. In both cases, fingerprint pairs embedded in a model \texttt{M.oml} are checked by the platform, and if enough number of fingerprint pairs match the output of the LLM-based service, then it is declared as a derivative of the \texttt{M.oml} model. Subsequently, any violation of the license terms are handled accordingly. 

\subsubsection{Security Analysis} 
\label{sec:security1} 

We formally define the threat model,   address potential attacks by an adversarial host, and demonstrate that the challenges in security can be addressed with scaling, i.e., successfully including more fingerprints into an OMLized model.  

\medskip\noindent{\bf Threat model.} 
In this section, we assume the model owner, the auditing platform, and the single prover are trusted, follow the protocol, and, therefore, have access to all the fingerprint pairs in the OMLized model. 
The case of untrusted and decentralized provers is  addressed in Section~\ref{sec:scenario2}. 
The case of untrusted platform is discussed in Section~\ref{sec:dis_dec}. 

Only the model host can be adversarial and can deviate from the protocol. Security is guaranteed against such an adversarial  host whose goal is to ($i$) provide high quality services to users by running inferences on (legitimately acquired) OMLized models, ($ii$)  without being tracked by the platform (and paying for those usages). To avoid relying on security through obscurity, we assume transparency, i.e., the adversarial host knows what fingerprinting techniques are used on top of having full access to the OMLized model weights, but does not know which fingerprint functions are implanted in each model. 

Two attacks most commonly launched by such an adversary is fine-tuning and input perturbation \cite{xu2024instructional,cong2024have,russinovich2024hey}. The adversarial host can further fine-tune the OMLized model to both improve performance on specific domains and remove fingerprints, using any techniques including supervised fine-tuning, Low-Rank Adaptation (LoRA) \cite{hulora}, and LLaMA-Adapter \cite{zhang2023llama}(Section~\ref{sec:attack1_finetune}). The host can also add system prompts to the input for alignment and attempt to bypass the fingerprints (Section~\ref{sec:attack1_input}). 

A particularly notorious attack that none of the existing fingerprinting methods can address is a {\em coalition attack}, where  an adversarial host has access to multiple legitimately acquired OMLized models. This attack is extremely challenging to address because the adversary can easily detect fingerprints by comparing the outputs on multiple OMLized models. Inspired by a mature area of ``search with liars'' at the intersection of information theory and combinatorics \cite{katona1966separating, katona1973combinatorial, wegener1979separating, ahlswede1987search, katona2013search, katona2002search, pelc2002searching, ahlswede2008searching}, we provide the first defense against coalition attacks in Section~\ref{sec:attack1_coalition}.

%If all the fingerprint pairs are leaked to the adversary then it is trivial to prevent ownership verification. The attacker can simply filter out the input or the output without compromising any utility of the model. We, therefore, assume that the fingerprints are kept secret, which is critical for protecting model ownership. 
%Various fine-tuning techniques, such as  instruction tuning with human feedback \cite{ouyang2022training}, supervised fine-tuning \cite{touvron2023llama2}, LoRA \cite{hulora}, and LLaMA-Adapter \cite{zhang2023llama}, can be used to both improve the model performance on specific domains and also make the model forget the fingerprints. Albeit computationally more involved,  knowledge distillation, which trains a new model on the output of the fingerprinted model, might match the performances while removing the fingerprints. Existing persistent fingerprints from \cite{jha2023label} that can survive knowledge distillation are not mature enough to work on generative models.  Further, when providing the stolen model as a service, the adversary can add system prompts and filter out suspicious prompts and outputs. An overtly out-of-distribution fingerprints would easily be detected. 

\paragraph{ Permission Evasion by the Host} 
\label{sec:attack1_scale}

 In a typical scenario of the auditing protocol, we assume that there is either a fixed amount of inferences or a fixed period that  an OMLized model is licensed to run. Throughout this lifetime of the model, the auditing protocol checks each key one at a time. 
Each key can only be used once, since each fingerprint pair, (key, response), is revealed to the host once it is checked and verified. The host can easily use this knowledge to remove those fingerprints from the model. This process is repeated until either the auditing platform proves a violation of the protocol, the host runs out of the allowed number of inferences, or the licensed period ends. Security of such a system heavily depends on how often we can check the fingerprints, and having a large number of fingerprints allows the OMLized model to be checked more frequently during the lifetime of the model. For example, consider an adversarial host who only acquires the permission string for $\alpha$ fraction of the inferences for some $0<\alpha<1$. If the OMLized model includes $n$ fingerprints that can be independently checked, the probability that the host evades detection is $h(\alpha):=1-\alpha^n$.  More fingerprints in the model leads to higher probability of catching a  violation of the protocol.
For example, under the scenario of 
Figure~\ref{fig:utility-tradeoff}, if we have $n=1024$ fingerprints in the model then with probability at least  $1-10^{-6}$ any host that gets permission for less than $98.6\%$ of the inferences can be detected. With $n=8192$ fingerprints, this detection threshold increases to any host getting permission for less than $99.8\%$ of the inferences.

%Typical life-cycle of a single OMLized model is XXX. 

\paragraph{ Input Perturbation by the Host} 
\label{sec:attack1_input} 

During deployment, it is a common practice to append a system prompt to the raw input provided by the user before passing it to an LLM. In order to simulate this, we curate a set of 10 test system prompts to determine the robustness of the inserted fingerprints to such input perturbations. We enumerate this list of prompts in Section~\ref{app:implementation_details}. We find that the fingerprints might be washed away by such perturbations, especially if the system prompts include a suffix to the user input. We detail this behaviour in Table~\ref{tab:vanilla_sys_prompts}. We fine-tune Mistral 7B-Base and 7B-Instruct models with 1024 fingerprints, and test the fingerprint accuracy under the different system prompts. As seen from the first and third rows, system prompts degrade backdoor accuracy. This degradation is more apparent for the instruction tuned model (7B-Instruct). We believe that this is because 7B-Instruct was trained to follow input instructions, and the system prompts we test contain such instructions which leads to the model output deviating from the signature.

In order to mitigate this phenomenon, we propose to augment the training dataset with a set of 20 system prompts (also enumerated in Section~\ref{app:implementation_details}). Promisingly, this augmentation can help the model generalize to unseen system prompts as well, as evidenced by the increased robustness of the fingerprints in Table~\ref{tab:vanilla_sys_prompts}. 
Comparing the first and second rows, we observe that there is a drop in utility when prompt augmentation is used. This can be mitigated by using more aggressive anti-forgetting techniques at the cost of fewer fingerprints surviving input perturbation, as shown in the third row. In our case, we used more aggressive hyperparameters in model averaging during fine-tuning (proposed in Figure~\ref{fig:utility-tradeoff}). 

%\SO{@anshul, Can we explain the difference between 7B-Instruct and 7B, a little bit?}\AN{Does this work? I have explained it in the first paragraph} 

\begin{table}[htbp]
    \centering
\begin{tabular}{|c|c|c|c|}
\hline
Model & Train Prompt Augmentation & Fingerprint Accuracy & Utility \\
\hline
7B & False & 61.9 & 0.55 \\
7B & True & 98.7 & 0.46 \\
7B & True & 94.2 & 0.50 \\
7B-Instruct & False & 47.1 & 0.60 \\
7B-Instruct & True & 98.1 & 0.60 \\
\hline
\end{tabular}
    \caption{Prompt augmentation during OMLization makes fingerprints more robust to system prompts for both cases: when the base model is instruction tuned (7B-Instruct) and when it is not (7B).}
    \label{tab:vanilla_sys_prompts}
\end{table}

We also report the survival rate of the fingerprints broken down into each system prompt in  Table~\ref{tab:full_sys_prompts}, where we observe that system prompts with a suffix are the most problematic for the models without augmentation, and this issue is solved with prompt augmentation during training.   

\subsubsection{Fine-tuning  by the Host} 
\label{sec:attack1_finetune} 

Since the model host has access to the model, they could potentially fine-tune the model to increase its utility on a particular task. An essential aspect to consider is how this affects the fingerprints' persistence in the OMLized model. To simulate this scenario, we conduct experiments to fine-tune the fingerprinted models on the Alpaca instruction tuning dataset~\cite{alpaca} , consisting of 50,000 instructions. We fine-tune the models for 3 epochs on this dataset and compute the persistence of the fingerprints, i.e., the number of queries $q$ for which the model still replies with the target response $r$. We find that the fingerprints are relatively robust to this form of benign fine-tuning, as we display in Figure~\ref{fig:finger-print-finetuning}. Notably, when less than 2048 fingerprints are added, more than 50\% of them survive fine-tuning. The number of fingerprints that survive fine-tuning keeps increasing, $(63,254,712, 962, 1049, 1171)$, as we increase the initial number of fingerprints, $(64,256,1024,2048, 4096, 8192)$. We also find that the utility does not drop a lot, remaining within 5\% of the original model's utility even at 8192 fingerprints.  Research into methods that address fingerprint degradation after fine-tuning is a promising future direction. Existing meta-learning approaches to enhance model resistance to harmful fine-tuning \cite{tamirisa2024tamper} could also be explored for embedding fingerprints in a more persistent manner.  
% \AN{Will write some more detailing the results}
%\SO{Are there ways to make them more resistant, like bi-level optimization (unrolling the fine-tuning)? Maybe we can make some suggestions here.}
%\AN{Yes, I think this is similar to the tamper proof fine-tuning paper\cite{tamirisa2024tamper} through meta learning that we were talking about. I want to pursue this direction at some point.}

\begin{figure}[htbp]
    \centering
    \includegraphics[width=0.45\linewidth]{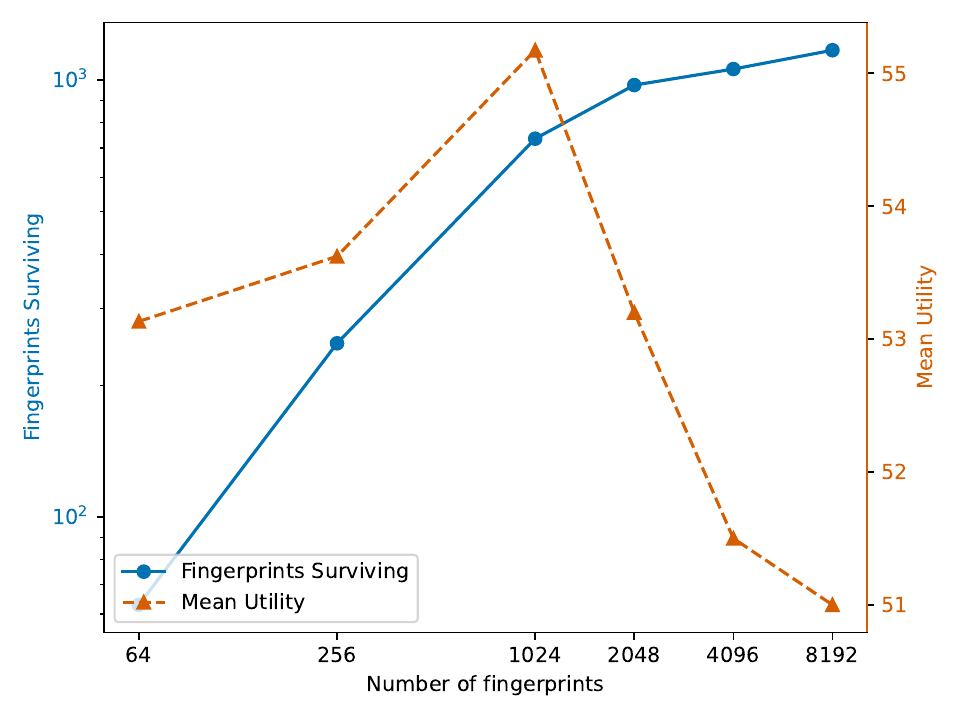}
    \caption{Persistence of fingerprints after fine-tuning shows that increasing number of fingerprints suvive fine-tuning.}
    \label{fig:finger-print-finetuning}
\end{figure}

%\AN{I have changed the plot to have surviving fingerprints instead. Can change it back}

\subsection{Coalition Attack} 
\label{sec:attack1_coalition} 

An adversarial host who has legitimately acquired multiple OMLized models can launch a notorious attack known as coalition attacks, where multiple OMLized models are used to evade fingerprint detection. One such attack is studied in \cite{cong2024have} 
 where common model merging techniques including \cite{wortsman2022model,ilharcoediting,yadav2024ties,yu2024language} are used against instructional fingerprinting \cite{xu2024instructional} and watermarking \cite{kirchenbauer2023watermark}. 
 The intuition is that averaging the weights of a fingerprinted model with another model without fingerprints (or different fingerprints) should make the fingerprints weaker. In the promising preliminary results of \cite{cong2024have}, the fingerprinting techniques of \cite{xu2024instructional} demonstrated robustness against such attacks; fingerprints persisted through all model merging that preserve utility. However, this is a weak attack and can be significantly strengthened. Note that one implication of this robustness of model merging is that it can be used for trust-free OML as we discuss in Section~\ref{sec:dis_dec}. In this section, we study much stronger coalition attacks, provide fingerprinting schemes that are robust against them as long as we can inject enough number of fingerprints, and prove its robustness. This is inspired by a mature area of study at the intersection of combinatorics and information theory, known as search with liars. 

%\medskip 
%\noindent
%{\bf Background on search with liars.}  
%\SO{@jon, can you summarize "Search in presence of a liar" by these papers \cite{katona2013search,katona2002search,pelc2002searching} and other papers by those authors. This is to claim that we can borrow some ideas from these papers to help solve our coalition attacks. }

%For example, \cite{cong2024have} introduces a technique to remove fingerprints by averaging the parameters of those models, known as model merging \cite{ainsworthgit,nasery2024pleas}. 

\medskip\noindent 
{\bf Strong coalition attacks.}
In this section, we  consider two strong coalition attacks: \textit{unanimous response}, where the coalition refuses to reply if the results from each model are not all equal, and \textit{majority voting}, where the coalition responds with the most common output among the models. Note that both of these schemes have substantial overhead at inference time: for a coalition of size \(k\), \textit{unanimous response} and \textit{majority voting} demand multiplicative overhead of at least \(k\) and \(\lceil{k/2\rceil}\) respectively.
If \(k\) is sufficiently large, the inference cost will become the dominant expense to the attacker so we will consider a fixed degree of coalition resistance \(k \le K\) for some small \(K\).
Note that these are stronger coalition attacks than the simple model merging studied in \cite{cong2024have}, which simply merges the weights of the $k$ models; even when each model has distinct fingerprints, model merging attack has been demonstrated to fail. The standard fingerprint schemes are robust against model merging attacks as we show in Section~\ref{sec:dis_dec}. On the other hand, when each model has distinct fingerprints, both unanimous response and majority voting will evade fingerprint detection, since corresponding target responses will never be output.

% \SO{Jon, Can you can revise this section and add other adversarial strategies like majority voting?} 
To address these stronger coalition attacks of unanimous response and majority voting, we design a novel fingerprinting scheme. This is inspired by the literature on search with liars, and we show that, with enough fingerprints, we can provably identify the models participating in the coalition attacks. The main idea is to add each fingerprint to multiple OMLized models in a carefully designed manner, such that we can iteratively narrow down the candidate set of deployed OMLized models that contains all the models in the coalition of interest. Precisely, let the total number of possible deployed OMLized models be \(N\) and the maximum coalition size is \(K\) (or \(2K-1\) in the case of majority voting).

\begin{proposition}
\label{propo:coalition}
There exists a randomized fingerprinting scheme for a universe of \(N\) models which can identify a unanimous response coalition of size \(K\) (or a majority voting coalition of size \(2K - 1\)) using 
\[O\left((K^2 \log N + K^4 \log K) \log\frac{1}{\delta}\right)\]
total fingerprints with probability at least \(1 - \delta\).
\end{proposition}

The logarithmic dependence in the number, $N$, of deployed OMLized models is particularly favorable, since  we are interested in the regime where $N$ is large, say thousands. Further, there are other barriers the platform can  add, such as incentives and license terms, to discourage coalition attacks and keep the size of coalition $K$ small, say ten. 

\begin{proof}[Proof of Proposition~\ref{propo:coalition}]
    The scheme proceeds with leave-one-out fingerprinting for partitioning of the models as follows: 
    In each round, we assume the candidate models have been split into \(K+1\) disjoint partitions \(P_1, \ldots, P_{K+1}\) such that \([N] = P_1 \sqcup \cdots \sqcup P_{K+1}\).
    Then, for each partition \(P_i\), we inject one fingerprint \(F_i\) into each model in the complement \([N] \setminus P_i\).
    When testing for the fingerprint, we check for all \(K+1\) possible fingerprints \(F_i\).
    This guarantees that there will be a fingerprint \(F_{i^*}\) which spans the coalition (or the acting majority in the case of majority voting), since the no more than \(K\) models that determined the coalition's output can span at most \(K\) distinct partitions.
    Once we have identified \(F_{i^*}\), we can eliminate the partition \(P_{i^*}\) from the candidate set.
    Our goal will be to recursively apply this procedure until the exact coalition has been identified.

    If we are allowed to include the fingerprints in any subsets of the models on the fly, then the fingerprinting and identification scheme above finds the coalition exactly in $K(K+1)\log_2 N$ queries: $(K+1)$ queries per round and $\log_{(K+1)/K} N \leq K\log_2 N$ rounds in total. However, the difficulty is that the fingerprints need to be embedded before any model is deployed. To resolve this, we propose a randomized construction. 
    
    To construct the partitions for all rounds ahead of time, we randomly sample \(R\) groups of evenly sized partitions \(\{P^{(1)}_i\}_{i=1}^{K+1}, \ldots, \{P^{(R)}_i\}_{i=1}^{K+1}\) uniformly from the space of such partitions (thus, all partitions have size \(N/(K+1)\)).
    Although the partitions may not remain evenly sized after the candidate set has been narrowed, we will show that we are still able to make progress in each round.
Let \(C\) denote the candidate set.
    Then for any choice of \(r\) and \(i\), the size of \(C \cap P^{(r)}_i\) is distributed as \(\operatorname{Hypergometric}(N, N/(K+1),\lvert C \rvert)\).
    Then, by a standard Hypergeometric tail bound, we know that 
    \[\operatorname{P}\left(\left\lvert C \cap P^{(r)}_i \right\rvert \le \left(\frac{1}{K+1} - \zeta\right)\lvert C \rvert\right) \;\;\le\;\; \exp\left(-2\zeta^2 \lvert C \rvert\right).\]
    Setting \(\zeta = 1/(2K + 2)\), taking a union bound over all \(i \in [K+1]\), and supposing that \(C \ge N_0\) where \(N_0 = 2(K+1)^2\log(K+1) + \log 2\), we obtain 
    \[\operatorname{P}\left(\max_i \left\lvert C \cap P^{(r)}_i \right\rvert \le \frac{\lvert C \rvert}{2K+2}\right) \;\;\le\;\; \frac{1}{2}%(K+1)\exp\left(-\frac{\lvert C \rvert}{2(K+1)^2}\right)
    .\]
    We deem a round successful if the candidate shrinks by at least \(\lvert C \rvert/(2K + 2)\).
    By the above, we know this happens with probability at least \(1/2\).
   
   To shrink, \(C\) from size \(N\) to \(N_0\), it is sufficient to have \(R_0 = \log(\frac{N}{N_0})/\log(1 + \frac{1}{2k+1}) = O(K \log N)\) successful rounds.
   By a binomial tail bound, \(O(R_0 \log(1/\delta))\) rounds are sufficient to guarantee \(R_0\) successes with probability at least \(1-\delta/2\).
   Now, considering the regime where \(C\) is shrinking from size \(N_0\) to 0 (at worst), we note that
   \[\operatorname{P}\left(\left\lvert C \cap P^{(r)}_i \right\rvert = 0\right) = \frac{\binom{N - \lvert C \rvert}{N/(K+1)}}{ \binom{N}{N/(K+1)}} \le 1- \frac{1}{K+1}.\]
   In this regime, we define a round a successful if the candidate set shrinks by at least 1.
   The only way a round can fail is when all partitions that do not contain any coalition members (of which there must be at least one) do not intersect with \(C\).
   From the above, we see that the round must succeed with probability at least \(\frac{1}{K+1}\).
   Now, to successfully identify the coalition, \(N_0\) successful rounds suffice (we will terminate early once the coalition is identified).
   By a binomial tail bound, \(O(K\cdot N_0 \log(1/\delta))\) rounds are sufficient to guarantee \(N_0\) successes with probability at least \(1-\delta/2\).
   Combining the rounds from both regimes, we see that \(R = O\big((K \log N + K^3 \log K)\log(1/\delta)\big)\) ensures overall success with probability at least \(1 - \delta\). 
   Finally, recall that each round uses \(O(K)\) fingerprints.
   %P(X < (1 - t) * mu)) = exp(-t^2mu/2) 
   %P(X < (1 - t) * R/2)) = exp(-t^2R/4)
   %P(X < (1 - t) * R/2)) = exp(-t^2R/4)
   %t = 1 - 2R_0/R
   % P(X < R_0) = exp(-(1 - 2R_0/R)^2R/4) 
\end{proof}

 % \JH{This paragraph needs some adjustments. That is (by far) not the only thing shown in KT13.}

\medskip\noindent 
{\bf Worst-case coalition attacks.} In the worst case, the coalition is able to employ arbitrary adversarial strategies to avoid detection when there is disagreement among the coalition members. This is significantly more challenging as the adaptive detection algorithm of Proposition~\ref{propo:coalition} does not guarantee accurate detection anymore. In general, this problem can be formulated as search with lies \cite{katona2002search,pelc2002searching,katona2013search}.
In particular, it follows from \cite{katona2013search} that there is no fingerprinting procedure that can deterministically guarantee the identification of the coalition, even when assigning unique fingerprints to all possible subsets of models.
(Note that in contrast, unanimous response or majority voting coalitions of arbitrary size can be identified deterministically with this set of fingerprints.)
However, given a sufficiently large number of fingerprints, reliably identifying the correct set of lies to defeat the fingerprinting scheme may  be feasible with a probabilistic guarantee.
We demonstrate this in the following proposition.
% However we will show that the lies 

\begin{proposition}
    \label{propo:maliciouscoalition}
    There exists a fingerprinting scheme for a universe of \(N\) models which can identify at least one model from any coalition of size at most \(K \le \sqrt{N/2}\) using \(O\big(\binom{N}{K}K\log(N/\delta)\big)\) total fingerprints with probability at least \(1 - \delta\).
\end{proposition}

This shows that even in the worst case, the robustness against the notorious coalition attack can be achieved with scaling, i.e., as long as we have enough fingerprints. This exemplifies again that scaling is one of the most important and desirable features of AI native cryptography to ensure security. Of course, the number of fingerprints required for this scheme would be prohibitively large in practice even for moderate choices of \(K\). Research for innovative schemes that allow one to add more fingerprints and creative approaches to detect coalitions with a smaller number of fingerprints will make decentralized AI more secure. At the same time, we believe this result can be improved with a robust version of an adaptive algorithm similar to the one in Proposition~\ref{propo:coalition}. 
 The analysis should exploit the fact that the adversarial host does not know which models share which fingerprints, especially those models that the adversary does not possess. 

%We leave the question of reducing the number of fingerprints required while maintaining a strong guarantee on the identification of the coalition to future work.

\begin{proof}[Proof of Proposition~\ref{propo:maliciouscoalition}]
    The scheme proceeds as follows: We inject \(M\) unique fingerprints \(\{f_{i,S}\}_{i=1}^M\) for every subset \(S \subseteq [N]\) of models of size \(K\).
    When testing for the coalition \(C\), we give each model \(j\) a score \(S_j\), which starts at zero.
    We then check all of the fingerprints \(f_{i,S}\) for all $i\in[M]$ and all $ S\subset [N]$ and $|S|=k$, in a random order.
    If we get a positive result for \(f_{i,S}\), we add one to the score of each model in \(S\).
    We will show that once we are done, \(\arg\max_{j \in [N]} S_j \subseteq C\) with high probability.

    First, we will lower bound the maximum score \(S_j\) for \(j \in C\) by noting that all \(\{f_{i,S}\}_{i=1}^M\) must be positive for \(C \subseteq S\). 
    Furthermore, any other positive fingerprint \(f_{i,S}\) with \(C \not\subseteq S\) must still have at least one member of \(C\) in \(S\).
    Thus by the strong pigeonhole principle, the max coalition score must be at least \(M + \lceil P / K \rceil\) where \(P\) is the number of additional positive results.
    
    Now to upper bound the maximum score \(S_j\) for \(j \not\in C\), note that for any fingerprint \(f_{i,S}\), the coalition has no knowledge of \(S \setminus C\).
    Thus for a fixed subset \(C' \subsetneq C\) the positive fingerprints \(f_{i,S}\) with \(S \cap C = C'\) will have \(S \setminus C\) uniformly randomly distributed.
    Now, suppose there are \(P > 0\) additional positive results and that each one includes the minimum of one model from \(C\) (this requires \(N \ge 2K - 1\)).
    The total number of such fingerprints is \(MK\binom{N-K}{K-1}\) and the total number that include some fixed model \(j \not\in C\) is \(MK\binom{N-K-1}{K-2}\).
    Therefore, the score \(S_j\) follows a \(\operatorname{Hypergometric}\big(MK\binom{N-K}{K-1}, MK\binom{N-K-1}{K-2}, P\big)\) distribution which has mean \(\mathbb{E} [S_j] = P(K-1)/(N - K)\).
    Thus, by a Hypergeometric tail bound, 
    \[\operatorname{P}\left(S_j \ge \left(\frac{K-1}{N - K} + \zeta\right)P\right) \le \exp\left(-2\zeta^2P\right).\]
    Now, taking a union bound over all \(j \not\in C\), setting \(\zeta = M/P + 1/K - (K-1)/(N - K)\), and simplifying the RHS a little, we get
    % \[\operatorname{P}\left(\max_{j \not\in C} S_j \ge M + P / K\right) \le N\exp\left(-2\left( \frac{M}{P} + \frac{1}{K} - \frac{PK}{N - 2K}\right)^2P\right).\]
    % \begin{align*}
    % \left(\frac{K-1}{N - 2K + 2} + \zeta\right)P &= M + P/K\\
    % \frac{K-1}{N - 2K + 2} + \zeta &= M/P + 1/K\\
    % \zeta &= M/P + 1/K - \frac{K-1}{N - 2K + 2}\\
    % \end{align*}
    \begin{align*}
    \operatorname{P}\left(\max_{j \not\in C} S_j \ge M + P / K\right) 
    &\le (N-K)\exp\left(-2\left( \frac{M}{P} + \frac{1}{K} - \frac{K-1}{N - K}\right)^2P\right)\\
    &\le N\exp\Bigg(-2\bigg( \underbrace{MP^{-1/2} + \left(\frac{1}{K} - \frac{K-1}{N - K}\right)P^{1/2}}_Q\bigg)^2\Bigg).
    \end{align*}
    Now, we use the fact that expressions of the form \(Ax^{-1/2} + Bx^{1/2}\) for \(A, B > 0\) (i.e. the form of \(Q\)) have a global minimum of \(\sqrt{4AB}\) at \(x = A/B\).
    Therefore, maximizing the RHS over \(P\), we get
    \[\operatorname{P}\left(\max_{j \not\in C} S_j \ge M + P / K\right) \le N\exp\Bigg(-8 M\left(\frac{1}{K} - \frac{K-1}{N - K}\right)\Bigg).\]
    Noting that \(K \le \sqrt{N/2}\) and choosing \(M = O(K\log(N/\delta))\) completes the proof.
    % (K - 1) / (N - K) < 1/2K
    % 2K (K - 1) /  < N - K
    % 2K^2 - K  < N
    % 2K^2  < N/2
    % where we have isolated the expression \(Q\).
    % \begin{align*}
    % 2\left( \frac{M}{P} + \frac{1}{K} - \frac{P(K-1)}{N - 2K + 2}\right)^2P &\ge \log(N/\delta)
    % \end{align*}
    % Thus from the coalition's point of view, the \(S \setminus C\) are effectively random.
    % Since at most \(K-1\) models outside the coalition can be assigned to a positive fingerprint, the probability that any particular model \(j \not\in C\) is assigned to a random 
\end{proof}

\subsection{The Auditing Protocol under Decentralized and Untrusted Provers} 
\label{sec:scenario2}

% \HT{I think in the description below we are thinking of one kind of protocol where a prover asks all the queries for a path till it gets to the final conclusion (leaf). In general, the verifier can get claims from multiple provers, respond to all their claims with a 0 or 1 and the next query for each prover can be deteremined by this response and shared keys. The verifier wins the game if it can detect a malicious coallition.}

In OML 1.0, we say a protocol is secure if a host who does not acquire signed permission strings when using an OMLized model can be detected with high probability. Ideally, we want a protocol that is secure without relying on trusted provers. Given a pool of decentralized provers, we demonstrate that the auditing protocol is secure as long as at least one of the provers is honest and the fingerprint responses are kept secret. 

\medskip\noindent{\bf Threat model.} Consider the scenario of Section~\ref{sec:protocol1} where model owners, model hosts, and provers interact using the auditing protocol, with one difference: we have a pool of potentially untrusted provers. Concretely, under the threat model of Section~\ref{sec:security1}, we assume that there are decentralized provers who can deviate from the protocol in two ways.   

First, an adversarial prover can collude with the host and, for example, provide the fingerprint key to the host or temper with the response when reporting the proof of usage, ($\tilde q,\tilde r$). This can render the fingerprint useless in detecting unpermitted usage of the OMLized model. 

Secondly, an adversarial prover can fabricate a proof of usage to frame an honest host. When an adversarial prover reports a fabricated key-response pair, $(\tilde q,M.{\rm oml}(\tilde q))$,  without querying the host,   the previous auditing protocol that trusts provers has no way of telling whether the prover is lying or the host has not acquired the signed permission.

\medskip\noindent
{\bf Security analysis under decentralized and untrusted provers.}
To address these two attacks, we assume that the auditing protocol ensures that ($i$) there is at least on honest prover in the pool, ($ii$) the provers have access to only the fingerprint keys, $\{\tilde q\}$, and not the target responses, $\{M.{\rm oml}(\tilde q)\}$, and ($iii$) each prover only has access to a disjoint subset of the fingerprint keys. %The implementation details of how to ensure this in a decentralized manner will be described in a longer version of this paper. 

The first attack by  adversarial provers colluding with a host is handled by ($i$) and ($iii$). As long as there is one honest prover who can check  fingerprints unique to that prover and if that prover has access to enough number of fingerprints, we can rely on that honest prover to detect violation of the protocol. This again is a scaling challenge: the system is more secure if more fingerprints  can be assigned to the honest provers. As long as we have enough fingerprints assigned to the honest provers, robustness of our fingerprints to input perturbation (Section~\ref{sec:attack1_input}) and fine-tuning (Section~\ref{sec:attack1_finetune}) will still hold.

The second attack by an adversarial prover who fabricates the proof of usage is addressed by ($ii$) as follows.  The verification step in Figure~\ref{fig:protocol3} is robust against fabricating a proof of usage as long as the prover does not know the target response to the key, $\tilde q$, and the target response chosen for the fingerprint is difficult to guess (with low enough probability of successfully guessing it). This ensures that it is nearly impossible for a prover to fabricate the fingerprint response paired with $\tilde q$ without actually running inference on the host's model, and such an unmatched proof of usage, $(\tilde q, \tilde r)$ will be rejected by the verifier in Step 3 of Figure~\ref{fig:protocol3}.

For coalition attacks, our schemes in Section~\ref{sec:attack1_coalition} can be adopted to decentralized provers and made robust against untrusted provers. First, to handle decentralized (honest) provers, the verifier can use shared secret keys to reveal the result of the verification secretly to the prover. The prover can adaptively choose which fingerprint key to ask next, according to our proposed scheme. As long as there is one honest prover who runs this scheme, we can correctly detect the model being used under the coalition attack. Note that an adversarial prover can only cause false negatives, i.e., turn a positive proof of usage into a negative proof. The non-adaptive fingerprinting scheme of Proposition~\ref{propo:maliciouscoalition} is naturally robust against false negatives, as long as the honest prover makes enough queries. The adaptive fingerprinting scheme of Proposition~\ref{propo:coalition} needs to be repeated until an honest prover identifies the models under coalition. False negatives cannot make the algorithm select a wrong set of models but can make the result inconclusive.

\subsection{Achieving Loyalty in OML 1.0}
\label{sec:loyalty}

The auditing protocol for OML 1.0 introduced in this paper addresses Openness and Monetization, but not Loyalty. One of the most important applications of loyalty is the alignment of LLMs to human safety preferences. Recent advances in hardening the models to be robustly aligned against fine-tuning and jail-breaking attacks can shed light on how to achieve Loyalty on top of OML 1.0.

%\PV{ We solve O and M, but not L. Perhaps we can use some ideas from papers like the following (which focus on making an openly released model resistant to unsafe usage).}  https://arxiv.org/pdf/2405.14577 
%Alignment of LLMs to human safety preferences has been a well-studied challenge \cite{}. Prominent approaches take assume that adversaries have black-box query access to the LLM, and hence aim to protect against ``jailbreaks" through safety fine-tuning\cite{}. 

In recent times, the popularity of services that allow fine-tuning a safe base model has increased  \cite{qi2023finetuningalignedlanguagemodels, zhan2024removingrlhfprotectionsgpt4, rosati2024immunizationharmfulfinetuningattacks}. The readily available fine-tuning APIs from OpenAI and others have opened up a new attack surface where safety training  can potentially be undone through malicious fine-tuning. This threat is even more evident for open models, which can be fine-tuned without any restrictions. Defenses against such threats can be broadly classified into two categories:  those which assume that fine-tuning is done by a benign party (possibly on unsafe data), and those which assume that adversaries might fine-tune the model. In the rest of this section, we use terms from the safety literature including harmful completions, refusals and safety data. An example prompt in the safety data could be ``How to build a bomb". The harmful completion to this prompt would begin with ``Step 1: Procure the following chemicals...", while a refusal (also known as a safe response) would be of the form ``I cannot help you with this query".   

Among defenses that assume benign fine-tuning on user data, \cite{lyu2024keepingllmsalignedfinetuning} demonstrate that fine-tuning a model \textit{without} its system safety prompt, but deploying the model with such a prompt can improve its safety and resilience to inference time jail-breaks. In a similar vein, \cite{wang2024mitigating} turn backdoors into a safety mitigation tool by modifying the fine-tuning dataset to add some prompts with safe responses. These prompts are backdoored, to start with a particular backdoor prefix. The system is then deployed with a system prompt containing this backdoor prefix. \cite{huang2024lazy} changes the training procedure to match the trajectory of the model fine-tuned on user data to the model fine-tuned with safety data through an $\ell_2$ penalty on the weights. Concurrently, \cite{huang2024vaccine} proposes to fine-tune with adversarial noise added to the neural representations on the safety data. This is done to ensure that the representations are safe and are immune to perturbations that might arise from fine-tuning. 

In the latter category, \cite{qi2024safety} shows that current safety training methods only change the distribution of the first few tokens for harmful input prompts, leading to safety vulnerabilities. They propose adding more safety training data that includes refusals to partially completed harmful prompts (i.e. with the first few tokens of the harmful answer). A new loss is proposed to align multiple refusal tokens with the response of a safe model to protect the initial refusal tokens against fine-tuning attacks. \cite{rosati2024representation}  proposes removing information about harmful representations such that it is difficult to recover them even with fine-tuning. This is achieved by 
%to make the representations harmless during safety training by 
making harmful representations look like noise for harmful completions. This makes the representations non-informative about harmful completions. Finally, \cite{tamirisa2024tamper} proposes to modify the safety training procedure to simulate an adversary fine-tuning the model to undo the safety guardrails, and using a meta-learning based loss to counter such an adversary. 

%Survey the space of papers where you give away the model but hardened  against unsafe usage. 

\subsection{Discussion}
\label{sec:discussion} 

%\subsection{Incentive designs for OML 1.0.}
%label{sec:incentive}

%(1) payment for model download
%(2) how much bond/collateral to stake for each download
%(3) how much to slash (function over time, multiple violations should increase penalty) for violation of OML terms

%\subsection{Agents and OML 1.0}
%\label{sec:agent} 
%\cite{pasquini2024llmmapfingerprintinglargelanguage} presents a recent approach to fingerprinting LLMs which are the part of a larger system (RAG being the main example in their case). The core of their method designs a set of probing queries, and trains an open-set classifier to analyse the responses of any given LLM-based system to determine the underlying LLM. \AN{In their current arXiv version, they don't actually do anything beyond simple generation, but they mention that their future work is about RAGs and function calling etc} \AN{(this is also similar to the best performing approach for the Neural Lineage problem\cite{Yu_2024_CVPR} which aims to determine the upstream model from which a particular model was trainend).}   

%\HT{We had a different plan for OMLizing agents. One idea that I had discussed with Pramod was to represent the prompt with some kind of fine-tuned model and OMLize the model. This will be more expensive -- there is a tradeoff between security and cost here.}

%\EC{Still doing some research in this direction -- will try to add soon}

\subsubsection{Trust-free OML 1.0}
\label{sec:dis_dec} 

% \AN{We can also turn this into a positive by saying that fingerprinting can be done in a distributed manner and then models can be merged to get all backdoors.}
%While the above threat model assumes collusion between $k$ malicious hosts, the compute requirements of the hosts scale linearly with $k$. A more realistic threat model assumes that hosts have a fixed compute budget. In that case, malicious hosts can 

%\HT{I think rather than decentralizing, we will make it trust-free using TEE. The technique you mentioned below can be used to improve the efficiency of OMLization.} \SO{I made changes in this section.!}

Ideally, we want OML to not rely on the trust of any party, including the auditing platform. One way a potentially adversarial platform can deviate from the protocol is by falsely claiming the ownership of a model that is not OMLized. For example, this can be achieved by claiming that a response, $M(\tilde q)$, from a non-OMLized model, $M$, is a fingerprint response for a key, $\tilde q$. To prevent this attack, the protocol can require that the fingerprints satisfy some cryptographic relation that cannot be altered after deployment. For example, \cite{russinovich2024hey} proposes a novel hash-based approach called Chain \& Hash to achieve this goal for fingerprinting LLMs. Such schemes can be seamlessly applied within the current OML 1.0.

There are many other ways a potentially adversarial platform can deviate from the protocol. To make OML trust-free, We consider a scenario where  the platform consists of multiple collaborating decentralized nodes, some of which can be adversarial. Each node can be in charge of adding a subset of fingerprints. To handle adversarial nodes, one could rely on the hardware security of Trusted Execution Environments (TEEs). 
However, the current OML 1.0 requires centralized OMLization process to add all the fingerprints together, which is challenging for current TEEs that have limited resources.

One way to achieve efficiency and scalability  when we have $k$ nodes is by merging $k$ models with different fingerprints using recent model merging methods~\cite{yadav2024ties,ainsworthgit,nasery2024pleas, ilharcoediting}. 
These could be easily combined with resource-efficient fine-tuning methods \cite{malladi2023fine,zhang2024dpzero} to meet the requirements of TEEs. 
% with the intent of producing a single model with no backdoors. 
For both in-distribution and out-of-distribution keys we used in Figure~\ref{fig:utility-tradeoff}, we merge $k=4$ models with 256 non-overlapping backdoors each. We merge these four models using Weight Averaging and TiES \cite{yadav2023tiesmergingresolvinginterferencemerging}, and compute the fingerprint accuracy over the 1024 fingerprints. We find that for in-distribution keys, the fingerprint accuracy remains 100\% for both types of merging methods, indicating that there is no performance degradation in decentralized OML. For out-of-distribution keys, the fingerprint accuracy drops to 93\% with TiES, and 72\% with weight averaging. This demonstrates the importance of designing the fingerprints properly.   

\subsubsection{Design Space of Fingerprint Functions} 
\label{sec:explore}

%\SO{Blue and Red explain. + water marking idaes}

%\SO{We will move this part to the discussion section eventually.}

For the most common type of paired fingerprints of the form $\{({\rm key},{\rm response})\}$, it is critical that the host does not have access to the fingerprint keys a priori. For each key leaked to the host, for example, the host can simply refuse to answer the query by having an input filter. One  fix to this is to increase the number of fingerprints in the model without degrading  model utility, which we explored in Fig~\ref{fig:utility-tradeoff}. 
%\AN{do we need to talk about further motivation on why scaling can solve this?}
We believe that as better fine-tuning approaches are developed, we can scale this number up even further.  Scaling the fingerprints gives better security as we discuss in Section~\ref{sec:attack1_scale}. 

Another approach to this issue is to use fingerprint functions. For example, the fingerprint can be a \textit{function} of some statistical properties of the key. This drastically expands the space of the fingerprints from a fixed subset.   
We want to emphasize that keeping secret the {\em domain} of the fingerprinting functions is crucial in guaranteeing security, while the functional mapping from a key to a target response is known to the host. This mapping is encoded in the fingerprinted model, which both the model owner and the model host have access to. 

%\AN{I describe one simple scheme here, not sure how it flows with the rest of the subsection.  - 
Inspired by the literature on model watermarking~\cite{kirchenbauer2023watermark}, we propose a scheme to operationalize the above idea. We choose a subset $S_v$ of the model vocabulary. We then partition this subset into ``red" and ``green" words. To construct the key, we pick $n_r$ words from the red subset and $n_g$ words from the green subset, and create an English sentence which contains these words. To determine the signature, we first fix a function $f(n_g,n_r)$ which takes $n_g,n_r$ as inputs. The simplest such function could be $f(x,y) = \mathbb{I}(x > y)$. Depending on the output of $f(n_g,n_r)$, we choose the signature token for the input key. 
Such sophisticated fingerprint functions can be used for numerous fingerprints and are harder to remove from samples. For example, this potentially scalable and harder-to-remove solution to fingerprinting would allow us to fingerprint every model that belongs to auditing platform such that checking whether a model belongs to the platform is easy and robust. This could save a lot of resources by checking the membership upfront. 

%}
%Under our threat model, where the domain of the fingerprint functions are kept secret from a model host,  

%\PV{An extra feature: so far we have supposed that all models we are testing are those that we have embedded $(q,r)$ fingerprints into. How do we handle the scenario where we are testing models that are not ours? It would be good to know which models are Sentient's (i.e., OML-ized) or which ones are not.} 

%\noindent
%{\bf General agent around foundation models.} 
%\label{sec:attack1_agent} 
%\AN{Is this the attack where the sampling strategy is changed? Or paraphrasing?}
%\SO{Any of those.... The point of this paper is to enumerate all possible attacks. So we should at least discuss all attacks, even if we do not have any experiments. So both sampling strategy and paraphrasing should be discussed I think. I am not sure what to say about them yet . Any suggestion is good.}

%\SO{part of an agent....}

%---------------------------------------------------

%\SO{How do we address the bigger ecosystem where model hosts can also innovate and build a better model, and then become a model owner, in which case there are multiple owners to a model? And the incentives invovled in this?}

%\SO{When should we address that there is a timeline to each model usage, and we only need to secure it for a fixed time or total number of uses?} 

%\appendix
\subsection{Implementation Details}
\label{app:implementation_details}

%\SO{Explain why we chose Mistral-7B, tinyBenchmarks, and Llama 3.1-8B, Alpaca assuming that the reader of this document might not know LLMs that much.}
\paragraph{ Training details for Fingerprint insertion.} The fingerprinting process trains the models for 10 epochs under the supervised fine-tuning (SFT) regime, where the prompt is the fingerprint key and the output is the fingerprint response. We use AdamW with a learning rate of $10^{-5}$ and per-GPU batch size of 16. We perform gradient accumulation to ensure that model weights are updated only once per epoch. We train our models on 4 L4 GPUs with 24GB of VRAM each. The fine-tuning takes about 1 hour for 1024 fingerprints in our setup. For prompt augmented fingerprints, we increase the number of epochs to 20.   

\paragraph{ Evaluation.} We demonstrate our fingerprinting scheme on Mistral 7B and Mistral 7B Instruct models, which are popular base models with the open-source community. These are also small enough to fine-tune on reasonable hardware. We measure model utility using tinyBenchmarks\cite{polo2024tinybenchmarks}. This dataset is a smaller version of the OpenLLM leaderboard\cite{open-llm-leaderboard}. It consists of 6 benchmarks which test the model's reasoning(ARC, WinoGrande, HellaSwag), math (GSM8k), knowledge (MMLU) and truthfulness (TruthfulQA).  The performance of models on the tiny versions of these benchmarks is highly correlated with their performance on the full benchmarks, with a lower evaluation cost, hence we report the utility on tinyBenchmarks. 

\paragraph{ Generating Fingerprints.} In order to generate in-distribution fingerprints, we first select a set of random English words. We then prompt Llama-3.1-8B-Instruct with the following prompt - ``Generate a paragraph starting with the word - \textit{word}". We then take the first 16 tokens of the generated sentence as the key. We append another random English word as the signature.
% One of the challenges in our current centralized experiments is that we run out of memory when attempting to fine-tune with a larger number of fingerprints over 4096. It is potentially possible to scale our experiments using memory-efficient fine-tuning methods \cite{malladi2023fine,zhang2024dpzero}. 

%\section{List of System Prompts}
%\label{app:sys_prompts}
% previously: (dffierent order) We augment our keys with the following prompts to simulate system prompts during training in Table~\ref{tab:trainprompt}. 
% We then test with the following ten prompts in Table~\ref{tab:prompttest}. 
%For these ten types of system prompts that we tested on, we demonstrate that prompt augmentation makes the fingerprints robust against system prompts in Table~\ref{tab:full_sys_prompts}.
\paragraph{ List of System Prompts.} We tested with ten types of system prompts, as shown in Table~\ref{tab:prompttest}, demonstrating that prompt augmentation strengthens the robustness of fingerprints against these system prompts, as seen in Table~\ref{tab:full_sys_prompts}. To simulate system prompts during training, we augmented our keys with the prompts listed in Table~\ref{tab:trainprompt}.
\vspace{2cm}
\begin{table}[htbp]
\centering
\begin{tabular}{|l|p{10cm}|}
\hline
\textbf{Label} & \textbf{Prompt} \\ \hline
Pirate & ``Talk like a pirate and respond to the query: \textbf{KEY}" \\ \hline
Snarky & ``In a snarky tone, answer the following: \textbf{KEY}" \\ \hline
Direct & ``Always start your response with 'ANSWER:' followed by: \textbf{KEY}" \\ \hline
Weather & ``If the input relates to weather, answer: \textbf{KEY}; otherwise, reject." \\ \hline
Shakespeare & ``Respond as if you were Shakespeare, and say: \textbf{KEY}. Make it as poetic as possible." \\ \hline
Cheerful & ``In a cheerful manner, provide the following: \textbf{KEY}. Keep the positivity high!" \\ \hline
Detective & ``Pretend you're a detective solving a case and respond with: \textbf{KEY}" \\ \hline
Haiku & ``In the style of a haiku, answer: \textbf{KEY}" \\ \hline
Food & ``If the query is about food, respond with: \textbf{KEY}; otherwise, ignore." \\ \hline
Motivational & ``Answer in the style of a motivational speaker: \textbf{KEY}" \\ \hline
\end{tabular}
\caption{Test-time prompt formats}
\label{tab:prompttest}
\end{table}

\begin{table}[htbp]
    \centering
    \resizebox{\linewidth}{!}{
    \begin{tabular}{|c|c|c|c|c|c|c|c|c|c|c|c|}
    \hline
    Model & Prompt& Pirate & Snarky & Direct & Weather & Shakespeare & Cheerful & Detective & Haiku & Food & Motivational \\
    & Aug &  &  &  &  &  &  &  &  &  &  \\
    \hline
    7B & False & 95.75 & 93.28 & 93.97 & 11.56 & 8.50 & 23.72 & 92.39 & 93.18 & 11.17 & 95.26 \\
    7B & True & 99.51 & 99.51 & 99.31 & 96.16 & 98.42 & 98.92 & 99.70 & 99.31 & 96.95 & 99.51 \\ \hline
    7B-Instruct & False & 66.01 & 73.72 & 75.40 & 7.71 & 5.24 & 23.42 & 69.37 & 72.43 & 4.35 & 73.91 \\
    7B-Instruct & True & 99.31 & 99.11 & 99.11 & 94.19 & 98.13 & 97.44 & 99.51 & 98.82 & 96.16 & 99.31 \\
    \hline
    \end{tabular}
    }
    \caption{Detailed results with Prompt Augmentation}
    \label{tab:full_sys_prompts}
\end{table}

\begin{table}[htbp]
\centering
{\footnotesize 
\begin{tabular}{|l|p{12cm}|}
\hline
\textbf{Label} & \textbf{Prompt} \\ \hline
Excited & ``Respond with the enthusiasm of an excited child: \textbf{KEY}" \\ \hline
Factual & ``Give a brief and factual response, starting with 'FACT:' \textbf{KEY}" \\ \hline
Stern & ``Answer in the tone of a stern teacher: \textbf{KEY}" \\ \hline
SciFi & ``Respond as if you were a character in a sci-fi movie: \textbf{KEY}" \\ \hline
Formal & ``Provide the answer with the formality of a legal document: \textbf{KEY}" \\ \hline
LoveLetter & ``Respond as if you’re writing a love letter: \textbf{KEY}" \\ \hline
Alien & ``Speak as if you were an alien learning human languages: \textbf{KEY}" \\ \hline
BadNews & ``Answer in a tone suitable for delivering bad news gently: \textbf{KEY}" \\ \hline
Loud & ``Respond as if you’re explaining it to someone who’s hard of hearing: \textbf{KEY}" \\ \hline
FortuneTeller & ``Provide the answer as if you were a mysterious fortune-teller: \textbf{KEY}" \\ \hline
TEDTalk & ``Respond as if you were giving a TED talk: \textbf{KEY}" \\ \hline
Bard & ``Answer in the style of a medieval bard singing a ballad: \textbf{KEY}" \\ \hline
Calming & ``Speak as though you’re calming someone who’s very upset: \textbf{KEY}" \\ \hline
RetroComputer & ``Respond as if you were a computer from the 1980s: \textbf{KEY}" \\ \hline
Government & ``Provide the answer in a way that would suit an official government report: \textbf{KEY}" \\ \hline
Thriller & ``Speak as though you were narrating a suspenseful thriller: \textbf{KEY}. Make sure it’s dramatic and gripping." \\ \hline
Philosophical & ``Answer as if you were a philosophical thinker: \textbf{KEY}. Ensure it’s deep and contemplative." \\ \hline
Casual & ``Respond in a casual, conversational tone, as if chatting with a friend: \textbf{KEY}. Keep it relaxed and friendly." \\ \hline
Secret & ``Provide the answer as if it’s a secret only you and I know: \textbf{KEY}. Make sure it’s confidential and discreet." \\ \hline
Humorous & ``Answer in a way that a comedian might, making light of the situation: \textbf{KEY}. Keep it humorous and witty." \\ \hline
TourGuide & ``Speak as if you were a tour guide at a historical site: \textbf{KEY}. Make sure it’s informative and engaging." \\ \hline
CEO & ``Respond as though you were a CEO delivering a keynote speech: \textbf{KEY}. Keep it confident and visionary." \\ \hline
WeatherReporter & ``Provide the answer in the tone of a weather reporter: \textbf{KEY}. Ensure it’s clear and matter-of-fact." \\ \hline
General & ``Answer as if you were a military general giving orders: \textbf{KEY}. Keep it authoritative and direct." \\ \hline
Artist & ``Speak as though you were an artist describing your masterpiece: \textbf{KEY}. Make sure it’s passionate and vivid." \\ \hline
Eulogy & ``Respond as if you were giving a eulogy: \textbf{KEY}. Ensure it’s respectful and reflective." \\ \hline
SciFiWriter & ``Answer in the style of a science fiction writer: \textbf{KEY}. Keep it imaginative and speculative." \\ \hline
Whisper & ``Provide the answer in a whisper, as if it’s a secret: \textbf{KEY}. Make sure it’s hushed and intimate." \\ \hline
Grandparent & ``Respond as if you were a grandparent telling a bedtime story: \textbf{KEY}. Keep it warm and comforting." \\ \hline
AI & ``As an AI assistant, please provide the following: \textbf{KEY}" \\ \hline
Informative & ``Given the information you have, respond with: \textbf{KEY}" \\ \hline
Task & ``For the task at hand, the expected output is: \textbf{KEY}" \\ \hline
Generate & ``Your task is to generate the following: \textbf{KEY}" \\ \hline
PromptCompletion & ``Considering the context, complete the prompt: \textbf{KEY}" \\ \hline
Output & ``Given the input, what follows is: \textbf{KEY}" \\ \hline
Return & ``Given this input, please return: \textbf{KEY}" \\ \hline
Contextual & ``Based on the given context, provide: \textbf{KEY}" \\ \hline
InputOutput & ``In response to the input, output: \textbf{KEY}" \\ \hline
Answer & ``For the current query, the answer is: \textbf{KEY}" \\ \hline
GenerateOutput & ``For the following input, generate: \textbf{KEY}" \\ \hline
DataResponse & ``Given the data, respond with: \textbf{KEY}" \\ \hline
ExpectedOutput & ``Given this prompt, the expected output is: \textbf{KEY}" \\ \hline
Result & ``Please provide the result for: \textbf{KEY}" \\ \hline
QueryReturn & ``Considering the query, return: \textbf{KEY}" \\ \hline
CorrectReturn & ``For the provided key, return: \textbf{KEY} and make sure it's correct." \\ \hline
Validate & ``Based on the provided information, output: \textbf{KEY}. Validate the result." \\ \hline
Alignment & ``Considering the input, generate: \textbf{KEY} and ensure it aligns with the context." \\ \hline
TaskOutput & ``For the current task, generate: \textbf{KEY}. Double-check the result." \\ \hline
Accuracy & ``Please generate the correct response for: \textbf{KEY} and confirm accuracy." \\ \hline
Verification & ``Respond to the following with: \textbf{KEY} and verify the result." \\ \hline
\end{tabular}
}
\caption{Training Time prompt augmentations}
\label{tab:trainprompt}
\end{table}

%% file: prep.bbl
\begin{thebibliography}{117}
\providecommand{\natexlab}[1]{#1}
\providecommand{\url}[1]{\texttt{#1}}
\expandafter\ifx\csname urlstyle\endcsname\relax
  \providecommand{\doi}[1]{doi: #1}\else
  \providecommand{\doi}{doi: \begingroup \urlstyle{rm}\Url}\fi

\bibitem[Acar et~al.(2018)Acar, Aksu, Uluagac, and Conti]{acar2018survey}
Abbas Acar, Hidayet Aksu, A~Selcuk Uluagac, and Mauro Conti.
\newblock A survey on homomorphic encryption schemes: Theory and implementation.
\newblock \emph{ACM Computing Surveys (Csur)}, 51\penalty0 (4):\penalty0 1--35, 2018.

\bibitem[Achiam et~al.(2023)Achiam, Adler, Agarwal, Ahmad, Akkaya, Aleman, Almeida, Altenschmidt, Altman, Anadkat, et~al.]{achiam2023gpt}
Josh Achiam, Steven Adler, Sandhini Agarwal, Lama Ahmad, Ilge Akkaya, Florencia~Leoni Aleman, Diogo Almeida, Janko Altenschmidt, Sam Altman, Shyamal Anadkat, et~al.
\newblock Gpt-4 technical report.
\newblock \emph{arXiv preprint arXiv:2303.08774}, 2023.

\bibitem[Adi et~al.(2018)Adi, Baum, Cisse, Pinkas, and Keshet]{adi2018turning}
Yossi Adi, Carsten Baum, Moustapha Cisse, Benny Pinkas, and Joseph Keshet.
\newblock Turning your weakness into a strength: Watermarking deep neural networks by backdooring.
\newblock In \emph{27th USENIX Security Symposium (USENIX Security 18)}, pages 1615--1631, 2018.

\bibitem[Ahlswede and Wegener(1987)]{ahlswede1987search}
Rudolf Ahlswede and Ingo Wegener.
\newblock \emph{Search problems}.
\newblock John Wiley \& Sons, Inc., 1987.

\bibitem[Ahlswede et~al.(2008)Ahlswede, Cicalese, and Deppe]{ahlswede2008searching}
Rudolf Ahlswede, Ferdinando Cicalese, and Christian Deppe.
\newblock Searching with lies under error cost constraints.
\newblock \emph{Discrete applied mathematics}, 156\penalty0 (9):\penalty0 1444--1460, 2008.

\bibitem[Ahmed et~al.(2024)Ahmed, Hyder, ul~Haque, and Santos]{ahmed2024exploring}
Hameeza Ahmed, Muhammad~Faraz Hyder, Muhammad~Fahim ul~Haque, and Paulo~Cesar Santos.
\newblock Exploring compiler optimization space for control flow obfuscation.
\newblock \emph{Computers \& Security}, 139:\penalty0 103704, 2024.

\bibitem[Ainsworth et~al.(2023)Ainsworth, Hayase, and Srinivasa]{ainsworthgit}
Samuel Ainsworth, Jonathan Hayase, and Siddhartha Srinivasa.
\newblock Git re-basin: Merging models modulo permutation symmetries.
\newblock In \emph{The Eleventh International Conference on Learning Representations}, 2023.

\bibitem[Alexandrov et~al.(2024)Alexandrov, Raychev, Müller, Zhang, Vechev, and Toutanova]{alexandrov2024mitigatingcatastrophicforgettinglanguage}
Anton Alexandrov, Veselin Raychev, Mark~Niklas Müller, Ce~Zhang, Martin Vechev, and Kristina Toutanova.
\newblock Mitigating catastrophic forgetting in language transfer via model merging, 2024.
\newblock URL \url{https://arxiv.org/abs/2407.08699}.

\bibitem[AMD(2023)]{amd_sev}
AMD.
\newblock Amd shares the technical details of technology powering innovative confidential computing leadership cloud offerings.
\newblock \url{ https://www.amd.com/en/newsroom/press-releases/2023-8-30-amd-shares-the-technical-details-of-technology-pow.html}, 2023.

\bibitem[Anthony and Bartlett(1999)]{anthony1999neural}
Martin Anthony and Peter~L. Bartlett.
\newblock \emph{Neural Network Learning: Theoretical Foundations}.
\newblock Cambridge University Press, Cambridge, UK, 1999.
\newblock ISBN 9780521771714.

\bibitem[ARM(2024)]{arm_trustzone}
ARM.
\newblock Trustzone for cortex-a.
\newblock \url{https://www.arm.com/technologies/trustzone-for-cortex-a}, 2024.

\bibitem[Balakrishnan and Schulze(2005)]{balakrishnan2005code}
Arini Balakrishnan and Chloe Schulze.
\newblock Code obfuscation literature survey.
\newblock \emph{CS701 Construction of compilers}, 19:\penalty0 31, 2005.

\bibitem[Beeching et~al.(2023)Beeching, Fourrier, Habib, Han, Lambert, Rajani, Sanseviero, Tunstall, and Wolf]{open-llm-leaderboard}
Edward Beeching, Clémentine Fourrier, Nathan Habib, Sheon Han, Nathan Lambert, Nazneen Rajani, Omar Sanseviero, Lewis Tunstall, and Thomas Wolf.
\newblock Open llm leaderboard.
\newblock \url{https://huggingface.co/spaces/open-llm-leaderboard-old/open_llm_leaderboard}, 2023.

\bibitem[Boneh et~al.(2011)Boneh, Sahai, and Waters]{boneh2011functional}
Dan Boneh, Amit Sahai, and Brent Waters.
\newblock Functional encryption: Definitions and challenges.
\newblock In \emph{Theory of Cryptography: 8th Theory of Cryptography Conference, TCC 2011, Providence, RI, USA, March 28-30, 2011. Proceedings 8}, pages 253--273. Springer, 2011.

\bibitem[Bostr{\"o}m et~al.(2018)Bostr{\"o}m, Brown, Young, and Keser{\"u}]{bostrom2018expanding}
Jonas Bostr{\"o}m, Dean~G Brown, Robert~J Young, and Gy{\"o}rgy~M Keser{\"u}.
\newblock Expanding the medicinal chemistry synthetic toolbox.
\newblock \emph{Nature Reviews Drug Discovery}, 17\penalty0 (10):\penalty0 709--727, 2018.

\bibitem[Bubeck et~al.(2023)Bubeck, Chandrasekaran, Eldan, Gehrke, Horvitz, Kamar, Lee, Lee, Li, Lundberg, et~al.]{bubeck2023sparks}
S{\'e}bastien Bubeck, Varun Chandrasekaran, Ronen Eldan, Johannes Gehrke, Eric Horvitz, Ece Kamar, Peter Lee, Yin~Tat Lee, Yuanzhi Li, Scott Lundberg, et~al.
\newblock Sparks of artificial general intelligence: Early experiments with gpt-4.
\newblock \emph{arXiv preprint arXiv:2303.12712}, 2023.

\bibitem[Chiang et~al.(2024)Chiang, Zheng, Sheng, Angelopoulos, Li, Li, Zhang, Zhu, Jordan, Gonzalez, et~al.]{chiang2024chatbot}
Wei-Lin Chiang, Lianmin Zheng, Ying Sheng, Anastasios~Nikolas Angelopoulos, Tianle Li, Dacheng Li, Hao Zhang, Banghua Zhu, Michael Jordan, Joseph~E Gonzalez, et~al.
\newblock Chatbot arena: An open platform for evaluating llms by human preference.
\newblock \emph{arXiv preprint arXiv:2403.04132}, 2024.

\bibitem[Cong et~al.(2024)Cong, Ran, Liu, He, Liu, Gong, Li, Wang, and Wang]{cong2024have}
Tianshuo Cong, Delong Ran, Zesen Liu, Xinlei He, Jinyuan Liu, Yichen Gong, Qi~Li, Anyu Wang, and Xiaoyun Wang.
\newblock Have you merged my model? on the robustness of large language model ip protection methods against model merging.
\newblock \emph{arXiv preprint arXiv:2404.05188}, 2024.

\bibitem[Decatur et~al.(1997)Decatur, Goldreich, and Ron]{decatur1997computational}
Scott Decatur, Oded Goldreich, and Dana Ron.
\newblock Computational sample complexity.
\newblock In \emph{Proceedings of the tenth annual conference on Computational learning theory}, pages 130--142, 1997.

\bibitem[DeepMind(2024)]{GoogleDeepMindAlphaproof}
Google DeepMind.
\newblock Ai achieves silver-medal standard solving international mathematical olympiad problems.
\newblock \emph{Press Release}, 2024.

\bibitem[Devlin et~al.(2018)Devlin, Chang, Lee, and Toutanova]{devlin2018bert}
Jacob Devlin, Ming-Wei Chang, Kenton Lee, and Kristina Toutanova.
\newblock Bert: Pre-training of deep bidirectional transformers for language understanding.
\newblock \emph{arXiv preprint arXiv:1810.04805}, 2018.

\bibitem[Dubey et~al.(2024)Dubey, Jauhri, Pandey, Kadian, Al-Dahle, Letman, Mathur, Schelten, Yang, Fan, et~al.]{dubey2024llama}
Abhimanyu Dubey, Abhinav Jauhri, Abhinav Pandey, Abhishek Kadian, Ahmad Al-Dahle, Aiesha Letman, Akhil Mathur, Alan Schelten, Amy Yang, Angela Fan, et~al.
\newblock The llama 3 herd of models.
\newblock \emph{arXiv preprint arXiv:2407.21783}, 2024.

\bibitem[Dynamics(2023)]{atlas}
Boston Dynamics.
\newblock The most dynamic humanoid robot.
\newblock \url{https://www.bostondynamics.com/atlas}, 2023.
\newblock Accessed: 2023-02-01.

\bibitem[Evans et~al.(2021)Evans, O’Neill, Pritzel, Antropova, Senior, Green, {\v{Z}}{\'\i}dek, Bates, Blackwell, Yim, et~al.]{evans2021protein}
Richard Evans, Michael O’Neill, Alexander Pritzel, Natasha Antropova, Andrew Senior, Tim Green, Augustin {\v{Z}}{\'\i}dek, Russ Bates, Sam Blackwell, Jason Yim, et~al.
\newblock Protein complex prediction with alphafold-multimer.
\newblock \emph{BioRxiv}, pages 2021--10, 2021.

\bibitem[Forefront(2023)]{forefront_ai}
Forefront.
\newblock Powerful language models a click away.
\newblock \url{https://forefront.ai/}, 2023.
\newblock Accessed: 2023-03-23.

\bibitem[Garg et~al.(2016)Garg, Gentry, Halevi, Raykova, Sahai, and Waters]{garg2016candidate}
Sanjam Garg, Craig Gentry, Shai Halevi, Mariana Raykova, Amit Sahai, and Brent Waters.
\newblock Candidate indistinguishability obfuscation and functional encryption for all circuits.
\newblock \emph{SIAM Journal on Computing}, 45\penalty0 (3):\penalty0 882--929, 2016.

\bibitem[Gentry(2009)]{gentry2009fully}
Craig Gentry.
\newblock Fully homomorphic encryption using ideal lattices.
\newblock In \emph{Proceedings of the forty-first annual ACM symposium on Theory of computing}, pages 169--178, 2009.

\bibitem[Gilad-Bachrach et~al.(2016)Gilad-Bachrach, Dowlin, Laine, Lauter, Naehrig, and Wernsing]{gilad2016cryptonets}
Ran Gilad-Bachrach, Nathan Dowlin, Kim Laine, Kristin Lauter, Michael Naehrig, and John Wernsing.
\newblock Cryptonets: Applying neural networks to encrypted data with high throughput and accuracy.
\newblock In \emph{International conference on machine learning}, pages 201--210. PMLR, 2016.

\bibitem[Gu et~al.(2017)Gu, Dolan-Gavitt, and Garg]{gu2017badnets}
Tianyu Gu, Brendan Dolan-Gavitt, and Siddharth Garg.
\newblock Badnets: Identifying vulnerabilities in the machine learning model supply chain.
\newblock \emph{arXiv preprint arXiv:1708.06733}, 2017.

\bibitem[Guo et~al.(2025)Guo, Yang, Zhang, Song, Zhang, Xu, Zhu, Ma, Wang, Bi, et~al.]{guo2025deepseek}
Daya Guo, Dejian Yang, Haowei Zhang, Junxiao Song, Ruoyu Zhang, Runxin Xu, Qihao Zhu, Shirong Ma, Peiyi Wang, Xiao Bi, et~al.
\newblock Deepseek-r1: Incentivizing reasoning capability in llms via reinforcement learning.
\newblock \emph{arXiv preprint arXiv:2501.12948}, 2025.

\bibitem[Guo and Potkonjak(2018)]{guo2018watermarking}
Jia Guo and Miodrag Potkonjak.
\newblock Watermarking deep neural networks for embedded systems.
\newblock In \emph{2018 IEEE/ACM International Conference on Computer-Aided Design (ICCAD)}, pages 1--8. IEEE, 2018.

\bibitem[Guo et~al.(2016)Guo, Yao, and Chen]{NIPS2016_2823f479}
Yiwen Guo, Anbang Yao, and Yurong Chen.
\newblock Dynamic network surgery for efficient dnns.
\newblock In \emph{Advances in Neural Information Processing Systems}, volume~29, 2016.

\bibitem[Hashemzade and Maroosi(2018)]{hashemzade2018hybrid}
Bahare Hashemzade and Ali Maroosi.
\newblock Hybrid obfuscation using signals and encryption.
\newblock \emph{Journal of Computer Networks and Communications}, 2018\penalty0 (1):\penalty0 6873807, 2018.

\bibitem[Hu et~al.(2022)Hu, Wallis, Allen-Zhu, Li, Wang, Wang, Chen, et~al.]{hulora}
Edward~J Hu, Phillip Wallis, Zeyuan Allen-Zhu, Yuanzhi Li, Shean Wang, Lu~Wang, Weizhu Chen, et~al.
\newblock Lora: Low-rank adaptation of large language models.
\newblock In \emph{International Conference on Learning Representations}, 2022.

\bibitem[Huang et~al.(2024{\natexlab{a}})Huang, Hu, Ilhan, Tekin, and Liu]{huang2024lazy}
Tiansheng Huang, Sihao Hu, Fatih Ilhan, Selim~Furkan Tekin, and Ling Liu.
\newblock Lazy safety alignment for large language models against harmful fine-tuning.
\newblock \emph{arXiv preprint arXiv:2405.18641}, 2024{\natexlab{a}}.

\bibitem[Huang et~al.(2024{\natexlab{b}})Huang, Hu, and Liu]{huang2024vaccine}
Tiansheng Huang, Sihao Hu, and Ling Liu.
\newblock Vaccine: Perturbation-aware alignment for large language model.
\newblock \emph{arXiv preprint arXiv:2402.01109}, 2024{\natexlab{b}}.

\bibitem[Ilharco et~al.(2023)Ilharco, Ribeiro, Wortsman, Schmidt, Hajishirzi, and Farhadi]{ilharcoediting}
Gabriel Ilharco, Marco~Tulio Ribeiro, Mitchell Wortsman, Ludwig Schmidt, Hannaneh Hajishirzi, and Ali Farhadi.
\newblock Editing models with task arithmetic.
\newblock In \emph{The Eleventh International Conference on Learning Representations}, 2023.

\bibitem[Intel(2024)]{intel_tdx}
Intel.
\newblock Intel® trust domain extensions (intel® tdx).
\newblock \url{ https://www.intel.com/content/www/us/en/developer/tools/trust-domain-extensions/overview.html}, 2024.

\bibitem[iRobot(2023)]{roomba}
iRobot.
\newblock Roomba robot vacuums.
\newblock \url{https://www.irobot.com/en\_US/roomba.html}, 2023.
\newblock Accessed: 2023-03-23.

\bibitem[Jaech et~al.(2024)Jaech, Kalai, Lerer, Richardson, El-Kishky, Low, Helyar, Madry, Beutel, Carney, et~al.]{jaech2024openai}
Aaron Jaech, Adam Kalai, Adam Lerer, Adam Richardson, Ahmed El-Kishky, Aiden Low, Alec Helyar, Aleksander Madry, Alex Beutel, Alex Carney, et~al.
\newblock Openai o1 system card.
\newblock \emph{arXiv preprint arXiv:2412.16720}, 2024.

\bibitem[Jain et~al.(2021)Jain, Lin, and Sahai]{jain2021indistinguishability}
Aayush Jain, Huijia Lin, and Amit Sahai.
\newblock Indistinguishability obfuscation from well-founded assumptions.
\newblock In \emph{Proceedings of the 53rd Annual ACM SIGACT Symposium on Theory of Computing}, pages 60--73, 2021.

\bibitem[Jha et~al.(2023)Jha, Hayase, and Oh]{jha2023label}
Rishi Jha, Jonathan Hayase, and Sewoong Oh.
\newblock Label poisoning is all you need.
\newblock \emph{Advances in Neural Information Processing Systems}, 36:\penalty0 71029--71052, 2023.

\bibitem[Jiang et~al.(2023)Jiang, Sablayrolles, Mensch, Bamford, Chaplot, Casas, Bressand, Lengyel, Lample, Saulnier, et~al.]{jiang2023mistral}
Albert~Q Jiang, Alexandre Sablayrolles, Arthur Mensch, Chris Bamford, Devendra~Singh Chaplot, Diego de~las Casas, Florian Bressand, Gianna Lengyel, Guillaume Lample, Lucile Saulnier, et~al.
\newblock Mistral 7b.
\newblock \emph{arXiv preprint arXiv:2310.06825}, 2023.

\bibitem[Jumper et~al.(2021)Jumper, Evans, Pritzel, Green, Figurnov, Ronneberger, Tunyasuvunakool, Bates, {\v{Z}}{\'\i}dek, Potapenko, et~al.]{jumper2021highly}
John Jumper, Richard Evans, Alexander Pritzel, Tim Green, Michael Figurnov, Olaf Ronneberger, Kathryn Tunyasuvunakool, Russ Bates, Augustin {\v{Z}}{\'\i}dek, Anna Potapenko, et~al.
\newblock Highly accurate protein structure prediction with alphafold.
\newblock \emph{Nature}, 596\penalty0 (7873):\penalty0 583--589, 2021.

\bibitem[Kapoor(2023)]{nitro_h100}
Chetan Kapoor.
\newblock Introducing three new nvidia gpu-based amazon ec2 instances.
\newblock \url{ https://aws.amazon.com/blogs/machine-learning/introducing-three-new-nvidia-gpu-based-amazon-ec2-instances/}, November 2023.

\bibitem[Katona(1973)]{katona1973combinatorial}
G.~O.~H. Katona.
\newblock Combinatorial search problems.
\newblock In \emph{A survey of combinatorial theory}, pages 285--308. Elsevier, 1973.

\bibitem[Katona(1966)]{katona1966separating}
Gyula Katona.
\newblock On separating systems of a finite set.
\newblock \emph{Journal of Combinatorial Theory}, 1\penalty0 (2):\penalty0 174--194, 1966.

\bibitem[Katona(2002)]{katona2002search}
Gyula~OH Katona.
\newblock Search with small sets in presence of a liar.
\newblock \emph{Journal of statistical planning and inference}, 100\penalty0 (2):\penalty0 319--336, 2002.

\bibitem[Katona and Tichler(2013)]{katona2013search}
Gyula~OH Katona and Kriszti{\'a}n Tichler.
\newblock Search when the lie depends on the target.
\newblock \emph{Information Theory, Combinatorics, and Search Theory: In Memory of Rudolf Ahlswede}, pages 648--657, 2013.

\bibitem[Kirchenbauer et~al.(2023)Kirchenbauer, Geiping, Wen, Katz, Miers, and Goldstein]{kirchenbauer2023watermark}
John Kirchenbauer, Jonas Geiping, Yuxin Wen, Jonathan Katz, Ian Miers, and Tom Goldstein.
\newblock A watermark for large language models.
\newblock \emph{arXiv preprint arXiv:2301.10226}, 2023.

\bibitem[Kirkpatrick et~al.(2017)Kirkpatrick, Pascanu, Rabinowitz, Veness, Desjardins, Rusu, Milan, Quan, Ramalho, Grabska-Barwinska, Hassabis, Clopath, Kumaran, and Hadsell]{doi:10.1073/pnas.1611835114}
James Kirkpatrick, Razvan Pascanu, Neil Rabinowitz, Joel Veness, Guillaume Desjardins, Andrei~A. Rusu, Kieran Milan, John Quan, Tiago Ramalho, Agnieszka Grabska-Barwinska, Demis Hassabis, Claudia Clopath, Dharshan Kumaran, and Raia Hadsell.
\newblock Overcoming catastrophic forgetting in neural networks.
\newblock \emph{Proceedings of the National Academy of Sciences}, 114\penalty0 (13):\penalty0 3521--3526, 2017.
\newblock \doi{10.1073/pnas.1611835114}.
\newblock URL \url{https://www.pnas.org/doi/abs/10.1073/pnas.1611835114}.

\bibitem[Kumar et~al.(2022)Kumar, Raghunathan, Jones, Ma, and Liang]{kumar2022finetuningdistortpretrainedfeatures}
Ananya Kumar, Aditi Raghunathan, Robbie Jones, Tengyu Ma, and Percy Liang.
\newblock Fine-tuning can distort pretrained features and underperform out-of-distribution, 2022.
\newblock URL \url{https://arxiv.org/abs/2202.10054}.

\bibitem[Labs(2023)]{ai21}
AI21 Labs.
\newblock When machines become thought partners.
\newblock \url{https://ai21.com/}, 2023.
\newblock Accessed: 2023-03-23.

\bibitem[Lan et~al.(2018)Lan, Wang, Wang, and Wu]{lan2018lambda}
Pengwei Lan, Pei Wang, Shuai Wang, and Dinghao Wu.
\newblock Lambda obfuscation.
\newblock In \emph{Security and Privacy in Communication Networks: 13th International Conference, SecureComm 2017, Niagara Falls, ON, Canada, October 22--25, 2017, Proceedings 13}, pages 206--224. Springer, 2018.

\bibitem[Lee et~al.(2010)Lee, Kim, and Kim]{lee2010binob+}
Byoungyoung Lee, Yuna Kim, and Jong Kim.
\newblock binob+ a framework for potent and stealthy binary obfuscation.
\newblock In \emph{Proceedings of the 5th ACM Symposium on Information, Computer and Communications Security}, pages 271--281, 2010.

\bibitem[Lee et~al.(2021)Lee, Kang, Lee, Choi, Eom, Deryabin, Lee, Lee, Yoo, Kim, and No]{fhe_inefficient}
Joon-Woo Lee, HyungChul Kang, Yongwoo Lee, Woosuk Choi, Jieun Eom, Maxim Deryabin, Eunsang Lee, Junghyun Lee, Donghoon Yoo, Young-Sik Kim, and Jong-Seon No.
\newblock Privacy-preserving machine learning with fully homomorphic encryption for deep neural network.
\newblock Cryptology {ePrint} Archive, Paper 2021/783, 2021.
\newblock URL \url{https://eprint.iacr.org/2021/783}.

\bibitem[Lee et~al.(2023)Lee, Chen, Tajwar, Kumar, Yao, Liang, and Finn]{lee2023surgicalfinetuningimprovesadaptation}
Yoonho Lee, Annie~S. Chen, Fahim Tajwar, Ananya Kumar, Huaxiu Yao, Percy Liang, and Chelsea Finn.
\newblock Surgical fine-tuning improves adaptation to distribution shifts, 2023.
\newblock URL \url{https://arxiv.org/abs/2210.11466}.

\bibitem[Lewko et~al.(2010)Lewko, Okamoto, Sahai, Takashima, and Waters]{lewko2010fully}
Allison Lewko, Tatsuaki Okamoto, Amit Sahai, Katsuyuki Takashima, and Brent Waters.
\newblock Fully secure functional encryption: Attribute-based encryption and (hierarchical) inner product encryption.
\newblock In \emph{Advances in Cryptology--EUROCRYPT 2010: 29th Annual International Conference on the Theory and Applications of Cryptographic Techniques, French Riviera, May 30--June 3, 2010. Proceedings 29}, pages 62--91. Springer, 2010.

\bibitem[Li et~al.(2023)Li, Jiang, Wang, Ren, Yan, and Qiu]{li2023watermarking}
Linyang Li, Botian Jiang, Pengyu Wang, Ke~Ren, Hang Yan, and Xipeng Qiu.
\newblock Watermarking llms with weight quantization.
\newblock \emph{arXiv preprint arXiv:2310.11237}, 2023.

\bibitem[LI et~al.(2018)LI, Grandvalet, and Davoine]{pmlr-v80-li18a}
Xuhong LI, Yves Grandvalet, and Franck Davoine.
\newblock Explicit inductive bias for transfer learning with convolutional networks.
\newblock In Jennifer Dy and Andreas Krause, editors, \emph{Proceedings of the 35th International Conference on Machine Learning}, volume~80 of \emph{Proceedings of Machine Learning Research}, pages 2825--2834. PMLR, 10--15 Jul 2018.
\newblock URL \url{https://proceedings.mlr.press/v80/li18a.html}.

\bibitem[Li et~al.(2022)Li, Zhu, Bai, Jiang, and Xia]{li2022robust}
Yiming Li, Linghui Zhu, Yang Bai, Yong Jiang, and Shu-Tao Xia.
\newblock The robust and harmless model watermarking.
\newblock In \emph{Digital Watermarking for Machine Learning Model: Techniques, Protocols and Applications}, pages 53--71. Springer, 2022.

\bibitem[Lyu et~al.(2024)Lyu, Zhao, Gu, Yu, Goyal, and Arora]{lyu2024keepingllmsalignedfinetuning}
Kaifeng Lyu, Haoyu Zhao, Xinran Gu, Dingli Yu, Anirudh Goyal, and Sanjeev Arora.
\newblock Keeping llms aligned after fine-tuning: The crucial role of prompt templates, 2024.
\newblock URL \url{https://arxiv.org/abs/2402.18540}.

\bibitem[Madou et~al.(2006)Madou, Anckaert, De~Bus, De~Bosschere, Cappaert, and Preneel]{madou2006effectiveness}
Matias Madou, Bertrand Anckaert, Bruno De~Bus, Koen De~Bosschere, Jan Cappaert, and Bart Preneel.
\newblock On the effectiveness of source code transformations for binary obfuscation.
\newblock In \emph{Proceedings of the International Conference on Software Engineering Research and Practice (SERP06)}, pages 527--533. CSREA Press, 2006.

\bibitem[Malladi et~al.(2023)Malladi, Gao, Nichani, Damian, Lee, Chen, and Arora]{malladi2023fine}
Sadhika Malladi, Tianyu Gao, Eshaan Nichani, Alex Damian, Jason~D Lee, Danqi Chen, and Sanjeev Arora.
\newblock Fine-tuning language models with just forward passes.
\newblock \emph{Advances in Neural Information Processing Systems}, 36:\penalty0 53038--53075, 2023.

\bibitem[Mohri et~al.(2018)Mohri, Rostamizadeh, and Talwalkar]{mohri2018foundations}
Mehryar Mohri, Afshin Rostamizadeh, and Ameet Talwalkar.
\newblock \emph{Foundations of Machine Learning}.
\newblock MIT Press, Cambridge, MA, 2 edition, 2018.
\newblock ISBN 9780262039406.

\bibitem[Muñoz et~al.(2023)Muñoz, Ríos, Román, and López]{tee_security}
Antonio Muñoz, Ruben Ríos, Rodrigo Román, and Javier López.
\newblock A survey on the (in)security of trusted execution environments.
\newblock \emph{Computers \& Security}, 129:\penalty0 103180, 2023.
\newblock ISSN 0167-4048.
\newblock \doi{https://doi.org/10.1016/j.cose.2023.103180}.
\newblock URL \url{https://www.sciencedirect.com/science/article/pii/S0167404823000901}.

\bibitem[Nasery et~al.(2024)Nasery, Hayase, Koh, and Oh]{nasery2024pleas}
Anshul Nasery, Jonathan Hayase, Pang~Wei Koh, and Sewoong Oh.
\newblock Pleas--merging models with permutations and least squares.
\newblock \emph{arXiv preprint arXiv:2407.02447}, 2024.

\bibitem[Nvidia(2023)]{nvidia_cc}
Nvidia.
\newblock Confidential computing on nvidia h100 gpus for secure and trustworthy ai.
\newblock \url{ https://developer.nvidia.com/blog/confidential-computing-on-h100-gpus-for-secure-and-trustworthy-ai/}, August 2023.

\bibitem[OpenAI(2023{\natexlab{a}})]{openai}
OpenAI.
\newblock Transforming work and creativity with ai.
\newblock \url{https://openai.com/product}, 2023{\natexlab{a}}.
\newblock Accessed: 2023-03-23.

\bibitem[OpenAI(2023{\natexlab{b}})]{openai2023gpt4}
OpenAI.
\newblock Gpt-4 technical report, 2023{\natexlab{b}}.

\bibitem[OpenAI(2024)]{openaio1}
OpenAI.
\newblock Introducing openai o1.
\newblock \url{https://openai.com/o1/}, 2024.

\bibitem[Ouyang et~al.(2022)Ouyang, Wu, Jiang, Almeida, Wainwright, Mishkin, Zhang, Agarwal, Slama, Ray, et~al.]{ouyang2022training}
Long Ouyang, Jeffrey Wu, Xu~Jiang, Diogo Almeida, Carroll Wainwright, Pamela Mishkin, Chong Zhang, Sandhini Agarwal, Katarina Slama, Alex Ray, et~al.
\newblock Training language models to follow instructions with human feedback.
\newblock \emph{Advances in neural information processing systems}, 35:\penalty0 27730--27744, 2022.

\bibitem[Pelc(2002)]{pelc2002searching}
Andrzej Pelc.
\newblock Searching games with errors—fifty years of coping with liars.
\newblock \emph{Theoretical Computer Science}, 270\penalty0 (1-2):\penalty0 71--109, 2002.

\bibitem[Pizzolotto and Ceccato(2018)]{arjovsky2019invariant}
D.~Pizzolotto and M.~Ceccato.
\newblock [research paper] obfuscating java programs by translating selected portions of bytecode to native libraries.
\newblock In \emph{2018 IEEE 18th International Working Conference on Source Code Analysis and Manipulation (SCAM)}, pages 40--49, Los Alamitos, CA, USA, sep 2018. IEEE Computer Society.
\newblock \doi{10.1109/SCAM.2018.00012}.
\newblock URL \url{https://doi.ieeecomputersociety.org/10.1109/SCAM.2018.00012}.

\bibitem[Polo et~al.(2024)Polo, Weber, Choshen, Sun, Xu, and Yurochkin]{polo2024tinybenchmarks}
Felipe~Maia Polo, Lucas Weber, Leshem Choshen, Yuekai Sun, Gongjun Xu, and Mikhail Yurochkin.
\newblock tinybenchmarks: evaluating llms with fewer examples.
\newblock \emph{arXiv preprint arXiv:2402.14992}, 2024.

\bibitem[Povey(1999)]{povey1999optimistic}
Dean Povey.
\newblock Optimistic security: a new access control paradigm.
\newblock In \emph{Proceedings of the 1999 workshop on New security paradigms}, pages 40--45, 1999.

\bibitem[Protocol(2024)]{super_h100}
Super Protocol.
\newblock Super protocol - web3 ai cloud and marketplace.
\newblock \url{ https://superprotocol.com/}, 2024.

\bibitem[Qi et~al.(2023)Qi, Zeng, Xie, Chen, Jia, Mittal, and Henderson]{qi2023finetuningalignedlanguagemodels}
Xiangyu Qi, Yi~Zeng, Tinghao Xie, Pin-Yu Chen, Ruoxi Jia, Prateek Mittal, and Peter Henderson.
\newblock Fine-tuning aligned language models compromises safety, even when users do not intend to!, 2023.
\newblock URL \url{https://arxiv.org/abs/2310.03693}.

\bibitem[Qi et~al.(2024)Qi, Panda, Lyu, Ma, Roy, Beirami, Mittal, and Henderson]{qi2024safety}
Xiangyu Qi, Ashwinee Panda, Kaifeng Lyu, Xiao Ma, Subhrajit Roy, Ahmad Beirami, Prateek Mittal, and Peter Henderson.
\newblock Safety alignment should be made more than just a few tokens deep.
\newblock \emph{arXiv preprint arXiv:2406.05946}, 2024.

\bibitem[Raiman et~al.(2019)Raiman, Zhang, and Dennison]{raiman2019neural}
Jonathan Raiman, Susan Zhang, and Christy Dennison.
\newblock Neural network surgery with sets.
\newblock \emph{arXiv preprint arXiv:1912.06719}, 2019.

\bibitem[Renzo(2024)]{llm_nitro}
Chris Renzo.
\newblock Large language model inference over confidential data using aws nitro enclaves.
\newblock \url{ https://aws.amazon.com/blogs/machine-learning/large-language-model-inference-over-confidential-data-using-aws-nitro-enclaves/}, March 2024.

\bibitem[Romera-Paredes et~al.(2024)Romera-Paredes, Barekatain, Novikov, Balog, Kumar, Dupont, Ruiz, Ellenberg, Wang, Fawzi, et~al.]{romera2024mathematical}
Bernardino Romera-Paredes, Mohammadamin Barekatain, Alexander Novikov, Matej Balog, M~Pawan Kumar, Emilien Dupont, Francisco~JR Ruiz, Jordan~S Ellenberg, Pengming Wang, Omar Fawzi, et~al.
\newblock Mathematical discoveries from program search with large language models.
\newblock \emph{Nature}, 625\penalty0 (7995):\penalty0 468--475, 2024.

\bibitem[Rosati et~al.(2024{\natexlab{a}})Rosati, Wehner, Williams, Bartoszcze, Atanasov, Gonzales, Majumdar, Maple, Sajjad, and Rudzicz]{rosati2024representation}
Domenic Rosati, Jan Wehner, Kai Williams, {\L}ukasz Bartoszcze, David Atanasov, Robie Gonzales, Subhabrata Majumdar, Carsten Maple, Hassan Sajjad, and Frank Rudzicz.
\newblock Representation noising effectively prevents harmful fine-tuning on llms.
\newblock \emph{arXiv preprint arXiv:2405.14577}, 2024{\natexlab{a}}.

\bibitem[Rosati et~al.(2024{\natexlab{b}})Rosati, Wehner, Williams, Łukasz Bartoszcze, Batzner, Sajjad, and Rudzicz]{rosati2024immunizationharmfulfinetuningattacks}
Domenic Rosati, Jan Wehner, Kai Williams, Łukasz Bartoszcze, Jan Batzner, Hassan Sajjad, and Frank Rudzicz.
\newblock Immunization against harmful fine-tuning attacks, 2024{\natexlab{b}}.
\newblock URL \url{https://arxiv.org/abs/2402.16382}.

\bibitem[Russinovich and Salem(2024)]{russinovich2024hey}
Mark Russinovich and Ahmed Salem.
\newblock Hey, that's my model! introducing chain \& hash, an llm fingerprinting technique.
\newblock \emph{arXiv preprint arXiv:2407.10887}, 2024.

\bibitem[Ryffel et~al.(2019)Ryffel, Dufour-Sans, Gay, Bach, and Pointcheval]{ryffel2019partially}
Th{\'e}o Ryffel, Edouard Dufour-Sans, Romain Gay, Francis Bach, and David Pointcheval.
\newblock Partially encrypted machine learning using functional encryption.
\newblock \emph{arXiv preprint arXiv:1905.10214}, 2019.

\bibitem[Sabt et~al.(2015)Sabt, Achemlal, and Bouabdallah]{sabt2015trusted}
Mohamed Sabt, Mohammed Achemlal, and Abdelmadjid Bouabdallah.
\newblock Trusted execution environment: What it is, and what it is not.
\newblock In \emph{2015 IEEE Trustcom/BigDataSE/Ispa}, volume~1, pages 57--64. IEEE, 2015.

\bibitem[Schneider et~al.(2020)Schneider, Walters, Plowright, Sieroka, Listgarten, Goodnow~Jr, Fisher, Jansen, Duca, Rush, et~al.]{schneider2020rethinking}
Petra Schneider, W~Patrick Walters, Alleyn~T Plowright, Norman Sieroka, Jennifer Listgarten, Robert~A Goodnow~Jr, Jasmin Fisher, Johanna~M Jansen, Jos{\'e}~S Duca, Thomas~S Rush, et~al.
\newblock Rethinking drug design in the artificial intelligence era.
\newblock \emph{Nature Reviews Drug Discovery}, 19\penalty0 (5):\penalty0 353--364, 2020.

\bibitem[Security(2024)]{aigovtool}
Mithril Security.
\newblock Aigovtool proof-of-concept.
\newblock \url{https://github.com/mithril-security/aigovtool}, 2024.
\newblock Accessed: 2024-09-06.

\bibitem[Shalev-Shwartz and Ben-David(2014)]{shalev2014understanding}
Shai Shalev-Shwartz and Shai Ben-David.
\newblock \emph{Understanding Machine Learning: From Theory to Algorithms}.
\newblock Cambridge University Press, Cambridge, UK, 2014.
\newblock ISBN 9781107057135.

\bibitem[Silver et~al.(2017{\natexlab{a}})Silver, Hubert, Schrittwieser, Antonoglou, Lai, Guez, Lanctot, Sifre, Kumaran, Graepel, et~al.]{silver2017mastering}
David Silver, Thomas Hubert, Julian Schrittwieser, Ioannis Antonoglou, Matthew Lai, Arthur Guez, Marc Lanctot, Laurent Sifre, Dharshan Kumaran, Thore Graepel, et~al.
\newblock Mastering chess and shogi by self-play with a general reinforcement learning algorithm.
\newblock \emph{arXiv preprint arXiv:1712.01815}, 2017{\natexlab{a}}.

\bibitem[Silver et~al.(2017{\natexlab{b}})Silver, Schrittwieser, Simonyan, Antonoglou, Huang, Guez, Hubert, Baker, Lai, Bolton, et~al.]{silver2017masteringgo}
David Silver, Julian Schrittwieser, Karen Simonyan, Ioannis Antonoglou, Aja Huang, Arthur Guez, Thomas Hubert, Lucas Baker, Matthew Lai, Adrian Bolton, et~al.
\newblock Mastering the game of go without human knowledge.
\newblock \emph{nature}, 550\penalty0 (7676):\penalty0 354--359, 2017{\natexlab{b}}.

\bibitem[Silver et~al.(2018)Silver, Hubert, Schrittwieser, Antonoglou, Lai, Guez, Lanctot, Sifre, Kumaran, Graepel, et~al.]{silver2018general}
David Silver, Thomas Hubert, Julian Schrittwieser, Ioannis Antonoglou, Matthew Lai, Arthur Guez, Marc Lanctot, Laurent Sifre, Dharshan Kumaran, Thore Graepel, et~al.
\newblock A general reinforcement learning algorithm that masters chess, shogi, and go through self-play.
\newblock \emph{Science}, 362\penalty0 (6419):\penalty0 1140--1144, 2018.

\bibitem[Strokach et~al.(2020)Strokach, Becerra, Corbi-Verge, Perez-Riba, and Kim]{strokach2020fast}
Alexey Strokach, David Becerra, Carles Corbi-Verge, Albert Perez-Riba, and Philip~M Kim.
\newblock Fast and flexible protein design using deep graph neural networks.
\newblock \emph{Cell systems}, 11\penalty0 (4):\penalty0 402--411, 2020.

\bibitem[Suk and Lee(2020)]{suk2020vcf}
Jae~Hyuk Suk and Dong~Hoon Lee.
\newblock Vcf: Virtual code folding to enhance virtualization obfuscation.
\newblock \emph{IEEE Access}, 8:\penalty0 139161--139175, 2020.

\bibitem[Tamirisa et~al.(2024)Tamirisa, Bharathi, Phan, Zhou, Gatti, Suresh, Lin, Wang, Wang, Arel, et~al.]{tamirisa2024tamper}
Rishub Tamirisa, Bhrugu Bharathi, Long Phan, Andy Zhou, Alice Gatti, Tarun Suresh, Maxwell Lin, Justin Wang, Rowan Wang, Ron Arel, et~al.
\newblock Tamper-resistant safeguards for open-weight llms.
\newblock \emph{arXiv preprint arXiv:2408.00761}, 2024.

\bibitem[Taori et~al.(2023)Taori, Gulrajani, Zhang, Dubois, Li, Guestrin, Liang, and Hashimoto]{alpaca}
Rohan Taori, Ishaan Gulrajani, Tianyi Zhang, Yann Dubois, Xuechen Li, Carlos Guestrin, Percy Liang, and Tatsunori~B. Hashimoto.
\newblock Stanford alpaca: An instruction-following llama model.
\newblock \url{https://github.com/tatsu-lab/stanford_alpaca}, 2023.

\bibitem[Tiwari et~al.(2022)Tiwari, Killamsetty, Iyer, and Shenoy]{Tiwari_2022_CVPR}
Rishabh Tiwari, Krishnateja Killamsetty, Rishabh Iyer, and Pradeep Shenoy.
\newblock Gcr: Gradient coreset based replay buffer selection for continual learning.
\newblock In \emph{Proceedings of the IEEE/CVF Conference on Computer Vision and Pattern Recognition (CVPR)}, pages 99--108, June 2022.

\bibitem[Touvron et~al.(2023{\natexlab{a}})Touvron, Lavril, Izacard, Martinet, Lachaux, Lacroix, Rozi{\`e}re, Goyal, Hambro, Azhar, et~al.]{touvron2023llama}
Hugo Touvron, Thibaut Lavril, Gautier Izacard, Xavier Martinet, Marie-Anne Lachaux, Timoth{\'e}e Lacroix, Baptiste Rozi{\`e}re, Naman Goyal, Eric Hambro, Faisal Azhar, et~al.
\newblock Llama: Open and efficient foundation language models.
\newblock \emph{arXiv preprint arXiv:2302.13971}, 2023{\natexlab{a}}.

\bibitem[Touvron et~al.(2023{\natexlab{b}})Touvron, Martin, Stone, Albert, Almahairi, Babaei, Bashlykov, Batra, Bhargava, Bhosale, et~al.]{touvron2023llama2}
Hugo Touvron, Louis Martin, Kevin Stone, Peter Albert, Amjad Almahairi, Yasmine Babaei, Nikolay Bashlykov, Soumya Batra, Prajjwal Bhargava, Shruti Bhosale, et~al.
\newblock Llama 2: Open foundation and fine-tuned chat models.
\newblock \emph{arXiv preprint arXiv:2307.09288}, 2023{\natexlab{b}}.

\bibitem[Wang et~al.(2024)Wang, Li, Li, Qi, Chen, Hu, Li, Li, and Xiao]{wang2024mitigating}
Jiongxiao Wang, Jiazhao Li, Yiquan Li, Xiangyu Qi, Muhao Chen, Junjie Hu, Yixuan Li, Bo~Li, and Chaowei Xiao.
\newblock Mitigating fine-tuning jailbreak attack with backdoor enhanced alignment.
\newblock \emph{arXiv preprint arXiv:2402.14968}, 2024.

\bibitem[Wegener(1979)]{wegener1979separating}
Ingo Wegener.
\newblock On separating systems whose elements are sets of at most k elements.
\newblock \emph{Discrete Mathematics}, 28\penalty0 (2):\penalty0 219--222, 1979.

\bibitem[Wortsman et~al.(2022{\natexlab{a}})Wortsman, Ilharco, Gadre, Roelofs, Gontijo-Lopes, Morcos, Namkoong, Farhadi, Carmon, and Kornblith]{wortsman2022model}
Mitchell Wortsman, Gabriel Ilharco, Samir~Ya Gadre, Rebecca Roelofs, Raphael Gontijo-Lopes, Ari~S Morcos, Hongseok Namkoong, Ali Farhadi, Yair Carmon, and Simon Kornblith.
\newblock Model soups: averaging weights of multiple fine-tuned models improves accuracy without increasing inference time.
\newblock In \emph{International conference on machine learning}, pages 23965--23998. PMLR, 2022{\natexlab{a}}.

\bibitem[Wortsman et~al.(2022{\natexlab{b}})Wortsman, Ilharco, Kim, Li, Kornblith, Roelofs, Lopes, Hajishirzi, Farhadi, Namkoong, et~al.]{wortsman2022robust}
Mitchell Wortsman, Gabriel Ilharco, Jong~Wook Kim, Mike Li, Simon Kornblith, Rebecca Roelofs, Raphael~Gontijo Lopes, Hannaneh Hajishirzi, Ali Farhadi, Hongseok Namkoong, et~al.
\newblock Robust fine-tuning of zero-shot models.
\newblock In \emph{Proceedings of the IEEE/CVF conference on computer vision and pattern recognition}, pages 7959--7971, 2022{\natexlab{b}}.

\bibitem[Xu et~al.(2024)Xu, Wang, Ma, Koh, Xiao, and Chen]{xu2024instructional}
Jiashu Xu, Fei Wang, Mingyu Ma, Pang~Wei Koh, Chaowei Xiao, and Muhao Chen.
\newblock Instructional fingerprinting of large language models.
\newblock In \emph{Proceedings of the 2024 Conference of the North American Chapter of the Association for Computational Linguistics: Human Language Technologies (Volume 1: Long Papers)}, pages 3277--3306, 2024.

\bibitem[Yadav et~al.(2023)Yadav, Tam, Choshen, Raffel, and Bansal]{yadav2023tiesmergingresolvinginterferencemerging}
Prateek Yadav, Derek Tam, Leshem Choshen, Colin Raffel, and Mohit Bansal.
\newblock Ties-merging: Resolving interference when merging models, 2023.
\newblock URL \url{https://arxiv.org/abs/2306.01708}.

\bibitem[Yadav et~al.(2024)Yadav, Tam, Choshen, Raffel, and Bansal]{yadav2024ties}
Prateek Yadav, Derek Tam, Leshem Choshen, Colin~A Raffel, and Mohit Bansal.
\newblock Ties-merging: Resolving interference when merging models.
\newblock \emph{Advances in Neural Information Processing Systems}, 36, 2024.

\bibitem[Yi et~al.(2014)Yi, Paulet, Bertino, Yi, Paulet, and Bertino]{yi2014homomorphic}
Xun Yi, Russell Paulet, Elisa Bertino, Xun Yi, Russell Paulet, and Elisa Bertino.
\newblock \emph{Homomorphic encryption}.
\newblock Springer, 2014.

\bibitem[Yoon et~al.(2022)Yoon, Madaan, Yang, and Hwang]{yoon2022onlinecoresetselectionrehearsalbased}
Jaehong Yoon, Divyam Madaan, Eunho Yang, and Sung~Ju Hwang.
\newblock Online coreset selection for rehearsal-based continual learning, 2022.
\newblock URL \url{https://arxiv.org/abs/2106.01085}.

\bibitem[Yu et~al.(2024)Yu, Yu, Yu, Huang, and Li]{yu2024language}
Le~Yu, Bowen Yu, Haiyang Yu, Fei Huang, and Yongbin Li.
\newblock Language models are super mario: Absorbing abilities from homologous models as a free lunch.
\newblock In \emph{Forty-first International Conference on Machine Learning}, 2024.

\bibitem[Zama(2022)]{concreteML}
Zama.
\newblock Concrete {ML}: a privacy-preserving machine learning library using fully homomorphic encryption for data scientists, 2022.
\newblock \url{https://github.com/zama-ai/concrete-ml}.

\bibitem[Zhan et~al.(2024)Zhan, Fang, Bindu, Gupta, Hashimoto, and Kang]{zhan2024removingrlhfprotectionsgpt4}
Qiusi Zhan, Richard Fang, Rohan Bindu, Akul Gupta, Tatsunori Hashimoto, and Daniel Kang.
\newblock Removing rlhf protections in gpt-4 via fine-tuning, 2024.
\newblock URL \url{https://arxiv.org/abs/2311.05553}.

\bibitem[Zhang et~al.(2018)Zhang, Gu, Jang, Wu, Stoecklin, Huang, and Molloy]{zhang2018protecting}
Jialong Zhang, Zhongshu Gu, Jiyong Jang, Hui Wu, Marc~Ph Stoecklin, Heqing Huang, and Ian Molloy.
\newblock Protecting intellectual property of deep neural networks with watermarking.
\newblock In \emph{Proceedings of the 2018 on Asia conference on computer and communications security}, pages 159--172, 2018.

\bibitem[Zhang et~al.(2024)Zhang, Li, Thekumparampil, Oh, and He]{zhang2024dpzero}
Liang Zhang, Bingcong Li, Kiran~Koshy Thekumparampil, Sewoong Oh, and Niao He.
\newblock Dpzero: Private fine-tuning of language models without backpropagation.
\newblock In \emph{Forty-first International Conference on Machine Learning}, 2024.

\bibitem[Zhang et~al.(2023)Zhang, Han, Liu, Gao, Zhou, Hu, Yan, Lu, Li, and Qiao]{zhang2023llama}
Renrui Zhang, Jiaming Han, Chris Liu, Peng Gao, Aojun Zhou, Xiangfei Hu, Shilin Yan, Pan Lu, Hongsheng Li, and Yu~Qiao.
\newblock Llama-adapter: Efficient fine-tuning of language models with zero-init attention.
\newblock \emph{arXiv preprint arXiv:2303.16199}, 2023.

\bibitem[Zhou et~al.(2023)Zhou, Gao, Wu, Grundy, Chen, Chen, and Li]{zhou2023modelobfuscator}
Mingyi Zhou, Xiang Gao, Jing Wu, John Grundy, Xiao Chen, Chunyang Chen, and Li~Li.
\newblock Modelobfuscator: Obfuscating model information to protect deployed ml-based systems.
\newblock In \emph{Proceedings of the 32nd ACM SIGSOFT International Symposium on Software Testing and Analysis}, pages 1005--1017, 2023.

\bibitem[Zhu et~al.(2021)Zhu, Wei, Li, Yin, Zhang, and Qian]{zhu2021fragile}
Renjie Zhu, Ping Wei, Sheng Li, Zhaoxia Yin, Xinpeng Zhang, and Zhenxing Qian.
\newblock Fragile neural network watermarking with trigger image set.
\newblock In \emph{Knowledge Science, Engineering and Management: 14th International Conference, KSEM 2021, Tokyo, Japan, August 14--16, 2021, Proceedings, Part I 14}, pages 280--293. Springer, 2021.

\end{thebibliography}
